\documentclass{article}

\usepackage[dvipsnames,table,xcdraw]{xcolor}
\usepackage[numbers,sort]{natbib}
\usepackage{iclr2025_conference,times}

\usepackage{microtype}
\usepackage{graphicx}
\usepackage{subfigure}
\usepackage{booktabs} %
\usepackage{url}

\usepackage{amsmath}
\usepackage{amssymb}
\usepackage{mathtools}
\usepackage{amsthm}
\usepackage{nccmath}
\usepackage{algorithm}
\usepackage{pdflscape}
\usepackage{rotating}
\newcommand{\rev}[1]{{\color{black}#1}}
\definecolor{cvprblue}{rgb}{0.21,0.49,0.74}
\usepackage[pagebackref,breaklinks]{hyperref}

\usepackage[capitalize,noabbrev]{cleveref}

\usepackage{enumitem}
\usepackage{thmtools, thm-restate}

\usepackage{wrapfig,lipsum,booktabs}
\usepackage{algpseudocode}
\usepackage{caption}
\usepackage{adjustbox}
\usepackage{textcomp}
\usepackage{pythonhighlight}

\usepackage[accsupp]{axessibility}
\usepackage[fixed]{fontawesome5}

\theoremstyle{plain}
\newtheorem{theorem}{Theorem}[section]
\newtheorem{proposition}[theorem]{Proposition}
\newtheorem{lemma}[theorem]{Lemma}

\theoremstyle{definition}
\newtheorem{definition}[theorem]{Definition}

\theoremstyle{remark}

\DeclareMathOperator*{\argmin}{arg\,min}

\newcommand*\diff{\mathop{}\!\mathrm{d}}

\newcommand*\Id{\mathrm{Id}}
\def\*#1{\mathbf{#1}}
\def\+#1{\mathcal{#1}}
\def\emp*#1{{#1}_n}

\newcommand*\supp{\mathrm{Spt}}
\newcommand*{\eqdef}{\vcentcolon =}

\definecolor{blush}{rgb}{0.87, 0.36, 0.51}

\newcommand{\teal}[1]{\textcolor{teal}{#1}}

\newcommand{\ERED}[1]{\boldsymbol{\textcolor{red}{#1}}}
\newcommand{\EBGE}[1]{\boldsymbol{\textcolor{blue}{#1}}}
\newcommand{\EORG}[1]{\boldsymbol{\textcolor{orange}{#1}}}
\definecolor{ForestGreen}{RGB}{34, 139, 34}
\newcommand{\EGGE}[1]{\boldsymbol{\textcolor{ForestGreen}{#1}}}

\newcommand{\xhdr}[1]{\textbf{#1}\:}

\definecolor{eqdst}{RGB}{31,119,180}  %
\definecolor{eqgmg}{RGB}{255,127,14}  %

\usepackage{xspace}
\definecolor{darkolivegreen}{rgb}{0.33, 0.42, 0.18}

\newcommand{\dst}{\textcolor{Mahogany}{DST}\xspace}
\newcommand{\gmg}{\textcolor{OliveGreen}{GMG}\xspace}

\usepackage{tikz}
\usetikzlibrary{shadows}

\iclrfinalcopy %
\begin{document}

\title{Disentangled Representation Learning with the Gromov-Monge Gap}

\date{}
\makeatletter
\makeatother

\author{Theo Uscidda$^{1,2}$\thanks{Equal contribution} \quad Luca Eyring$^{2,3,4,}$\footnotemark[1] \quad Karsten Roth$^{2,4,5}$ \\ \textbf{Fabian Theis$^{2,3,4}$ \quad  
  Zeynep Akata$^{2,3,4,}$\thanks{Equal advising\textsuperscript{†}} \quad 
  Marco Cuturi$^{1,6,\dagger}$} \\
\\
$^1$CREST-ENSAE \quad $^2$Helmholtz Munich \quad $^3$TU Munich\\
$^4$Munich Center of Machine Learning \quad  $^5$Tubingen AI Center \quad $^6$Apple\\
\small{\texttt{theo.uscidda@ensae.fr \quad luca.eyring@tum.de}}
\vspace{-3mm}
}

\maketitle
\begin{abstract}

Learning disentangled representations from unlabelled data is a fundamental challenge in machine learning. Solving it may unlock other problems, such as generalization, interpretability, or fairness. Although remarkably challenging to solve in theory, disentanglement is often achieved in practice through prior matching. Furthermore, recent works have shown that prior matching approaches can be enhanced by leveraging geometrical considerations, e.g., by learning representations that preserve geometric features of the data, such as distances or angles between points. However, matching the prior while preserving geometric features is challenging, as a mapping that \textit{fully} preserves these features while aligning the data distribution with the prior does not exist in general. To address these challenges, we introduce a novel approach to disentangled representation learning based on quadratic optimal transport. We formulate the problem using Gromov-Monge maps that transport one distribution onto another with minimal distortion of predefined geometric features, preserving them \textit{as much as can be achieved}. To compute such maps, we propose the Gromov-Monge-Gap (GMG), a regularizer quantifying whether a map moves a reference distribution with minimal geometry distortion. We demonstrate the effectiveness of our approach
for disentanglement across four standard benchmarks, outperforming other methods leveraging geometric considerations. Code is available at: \url{https://github.com/ExplainableML/GMG}

\end{abstract}

\everypar{\looseness=-1}
\section{Introduction}
Learning low-dimensional representations of high-dimensional data is a fundamental challenge in unsupervised deep learning~\citep{bengio2014representation}. Emphasis is put on learning representations that allow for efficient and robust adaptation across a wide range of tasks~\citep{higgins2018definition,pmlr-v97-locatello19a}. The fundamental property of \textit{disentanglement} has shown significant promise to improve generalization~\citep{pmlr-v119-locatello20a,roth2023disentanglement,hsu2023disentanglement,barinpacela2024identifiability}, interpretability and fairness~\citep{locatello2019fair,trauble2021corr}. Most works regard disentanglement as a one-to-one map between learned representations and ground-truth latent factors, effectively seeking to recover these factors from data alone in an unsupervised fashion.
While unsupervised disentanglement is theoretically impossible~\citep{pmlr-v97-locatello19a}, the inductive biases of autoencoder architectures ensure effective disentanglement in practice~\citep{rolinek2019vae_pca,zietlow2021demystifying}. Most approaches operate using variational autoencoder (VAE) frameworks~\citep{Kingma2014}, using objectives that match latent VAE posteriors to factorized priors~\citep{higgins2017betavae,kim2018factorvae,kumar2018dipvae,burgess2018annealedvae,chen2018betatcvae}.

More recently, studies such as~\citet{gropp2020isometric,chen2020learning,lee2022regularized,horan2021when,nakagawa2023gromovwasserstein,huh2023isometric,hahm2024isometricrepresentationlearningdisentangled} have provided a new perspective, showing that geometric constraints on representation spaces may also enable disentanglement.
Typically, latent representations are encouraged to preserve key geometric features of the data distribution, such as (scaled) distances or angles between samples. \citet{horan2021when} even demonstrate that unsupervised disentanglement is \emph{always} possible provided that the latent space is locally isometric to the data, further supporting the geometric desiderata. However, combining prior matching with these geometric aspects is challenging. In general, a mapping that perfectly aligns the data distribution with the prior while \emph{fully} preserving the geometric features of interest may not exist. This leads to an \textit{inherent trade-off}: Practitioners must carefully fine-tune regularization terms, either by altering prior matching to prioritize geometry preservation, or vice-versa.

In this work, we demonstrate how to \textit{effectively combine geometric desiderata with prior matching} within the VAE framework, using optimal transport (OT) theory~\citep{santambrogio2015optimal,Peyre2019computational}. By treating mappings from the data manifold to the latent space (encoders) or vice versa (decoders) as transport maps \( T : \mathcal{X} \to \mathcal{Y} \), we can leverage the Gromov-\{Monge, Wasserstein\} paradigm~\citep{sturm2020space,memoli2011gromov}, which aligns two distributions by finding a mapping that minimizes the distortion between intra-domain cost functions defined on their supports. Specifically, we consider cost functions \( c_\mathcal{X}(\mathbf{x}, \mathbf{x}') \) on \( \mathcal{X} \) and \( c_\mathcal{Y}(\mathbf{y}, \mathbf{y}') \) on \( \mathcal{Y} \) that encode geometric features such as scaled distances or angles. Consequently, the resulting mapping transforms one distribution onto the other while preserving these geometric features \textit{as much as possible}.

\textbf{Our Contribution}: A novel OT-based approach to disentanglement through geometric considerations.

\begin{itemize}[leftmargin=.3cm,itemsep=.0cm,topsep=0cm]
\item[(i)] We address the challenge of learning disentangled representations using geometric constraints by leveraging Gromov-Monge mappings between the data and prior distributions. Since \emph{fully} preserving geometric features---such as (scaled) distances or angles between points---during the alignment of these two distributions is generally impossible, we aim to find an alignment that, instead, minimizes the distortion of these features, thereby preserving them \emph{as much as possible}.

\item[(ii)] Inspired by~\citep{uscidda2023monge}, we introduce the \emph{Gromov-Monge Gap} (GMG), a regularizer that measures how closely a map \( T \) approximates a Gromov-Monge map for costs \( c_\mathcal{X}, c_\mathcal{Y} \). GMG measures whether \( T \) transports distributions with minimal distortion w.r.t. \( c_\mathcal{X}, c_\mathcal{Y} \). We propose an efficient procedure to compute GMG and describe how to integrate it within the VAE framework.

\item[(iii)] We show that when \( c_\mathcal{X} \) and \( c_\mathcal{Y} \) encode scaled distances or angles, the GMG and its finite-sample counterpart are weakly convex functions. In both cases, we precisely characterize the weak convexity constants and analyze their practical implications for practitioners.

\item[(iv)] Across four standard disentangled representation learning benchmarks, we show that incorporating geometry-preserving desiderata via the GMG significantly enhances disentanglement across various methods, from the standard \( \beta \)-VAE to the combination of \( \beta \)-TCVAE with HFS~\citep{roth2023disentanglement}.%
\end{itemize}

\section{Background: \rev{On Disentanglement, Quadratic-OT and Distortion}}
\label{sec:background}

\subsection{Disentangled Representation Learning}
\label{sec:disentangled-representational-learning}

\paragraph{The Disentanglement Formalism.}
Disentanglement has varying operational definitions. In this work, we follow the common understanding~\citep{higgins2017betavae,pmlr-v97-locatello19a,pmlr-v119-locatello20a,trauble2021corr,roth2023disentanglement} where data $\*x$ is generated by a process $p(\*x|\*z)$ operating on ground-truth latent factors $\*z\sim p(\*z)$, modeling underlying source of variations (s.a.\ object shape, color, background\dots). Given a dataset $\mathcal{D} = \{\*x_i\}_{i=1}^N$, $\*x_i\sim p_\textrm{data}$, unsupervised disentangled representation learning aims to find a mapping $e_\phi$ s.t.\ $e_\phi(\*x_i) \approx \mathbb{E}[\*z|\*x_i]$, up to element-wise transformations. Notably, this is to be achieved without prior information on $p(\*z)$ and $p(\*x|\*z)$.

\paragraph{Unsupervised Disentanglement through Prior Matching.}
Most unsupervised disentanglement methods operate on variational autoencoders (VAEs, \citet{Kingma2014}), which define a generative model of the form $p_\theta(\*x, \*z) = p(\*z)p_\theta(\*x|\*z)$. Here, $p_\theta(\*x|\*z)$ is a product of exponential family distributions with parameters computed by a decoder $d_\theta(\*z)$. The latent prior $p(\*z)$ is usually chosen as a standard Gaussian $\mathcal{N}(\*z|\*0_d, \*I_d)$, and the probabilistic encoder $q_\phi(\*z|\*x)$ is implemented through neural networks $e_\phi(\*x),\sigma_\phi(\*x)$ that predicts the latent parameters so that $q_\phi(\*z|\*x) = \mathcal{N}(\*z|e_\phi(\*x),\sigma^2_\phi(\*x))$. The $\beta$-VAE~\citep{higgins2017betavae} achieves disentanglement by minimizing
\begin{equation}
\label{eq:beta-vae}
    \min_{\theta,\phi} \mathbb{E}_{\*x\sim p_\textrm{data},\*z\sim q_\phi(\*z|\*x)}[\underbrace{-\log p_\theta(\*x|\*z)}_{\textrm{(i) reconstruction}} + \underbrace{\beta D_{\text{KL}}(q_\phi(\*z|\*x)||p(\*z))}_{\textrm{(ii) prior matching}}],
\end{equation}
which enforces $\beta$-weighted prior matching on top of the reconstruction loss, assuming statistical factor independence~\citep{roth2023disentanglement}. Several follow-ups refine latent prior matching through different losses or prior choices~\citep{kim2018factorvae,chen2018betatcvae,kumar2018dipvae,burgess2018annealedvae,rolinek2019vae_pca,moor2021topologicalautoencoders,balabin2024disentanglementlearningtopology}.

\paragraph{Disentanglement through a Geometric Lens.}

\rev{Recent studies have revealed a fundamental connection between geometric structure preservation and disentanglement in learned representations~\citep{gropp2020isometric,chen2020learning,lee2022regularized,nakagawa2023gromovwasserstein,huh2023isometric}. This connection was theoretically established by \citet{horan2021when} proving that unsupervised disentanglement is always feasible when the generative factors are sufficiently non-Gaussian and maintain local \textit{isometry} to the data. Our work builds directly on this insight by developing a learning framework that promotes representations that are as close as possible to being \textit{isometric} to the data. To quantify geometric preservation between spaces of different dimensions, we leverage quadratic OT theory, which originated in \citet{koopmans1957assignment} and was formalized by \citet{memoli2011gromov} as a framework for measuring isometric correspondence between metric spaces. We detail these tools in Section \ref{subsec:quadratic-ot} before showing how they can be used to learn representations in Section \ref{sec:distortion}.}  Additionally, in a concurrent work, \citet{sotiropoulou2024stronglyisomorphicneuraloptimal} also introduced the Gromov-Monge gap. However, they do not apply it for disentangled representational learning.

\subsection{Quadratic Optimal Transport}
\label{subsec:quadratic-ot}
OT~\citep{Peyre2019computational} \rev{theory studies efficient ways to map a probability distribution onto another. \textit{Linear} OT formulations, such as the \citet{Monge1781} problem, require domains $\+X,\+Y$ that can be directly compared through a cost function $c(x,y)$ defined between their elements.%
}
When these distributions lie on incomparable domains, one \rev{must instead rely on \textit{quadratic} formulations of OT (Q-OT), which instead compare geometric structure through \textit{intra-domain} costs, also known as} the Gromov-Monge (GM) and GW problems. \rev{In the context of representation learning, representation and data spaces are \textit{incomparable} by design, which necessitates the use of Q-OT in this work.}
\paragraph{Gromov-\{Monge, Wasserstein\} Formulations.}  Consider two compact $\+X \subset \mathbb{R}^{d_{\+X}}$, $\+Y \subset \mathbb{R}^{d_{\+Y}}$, equipped with \textit{intra-domain} cost $c_\+{X}: \+{X} \times \+{X} \rightarrow \mathbb{R}$ and $c_\+{Y}: \+{Y} \times \+{Y} \rightarrow \mathbb{R}$. We assume that $c_\mathcal{X}$ and $c_\mathcal{Y}$ (or $-c_\mathcal{X}$ and $-c_\mathcal{Y}$) are CPD kernels (Def.~\eqref{def:conditional-kernel}). For $p\in\+P(\+X)$ and $q\in\+P(\+Y)$, two distributions supported on each domain, the GM problem~\citep{mémoli2022comparison} seeks a map $T : \mathcal{X} \to \mathcal{Y}$ that push-forwards $p$ onto $q$, while minimizing the distortion of the costs:
\begin{equation}
\label{eq:gromov-monge-problem}
\tag*{(GMP)}
\inf_{T:T\sharp p=q} \int_{\+X\times\+X} 
\tfrac{1}{2}| 
c_\+X(\*x,\*x') - c_\+Y(T(\*x), T(\*x'))
|^2 \diff p(\*x)\diff p(\*x')\,.
\end{equation}
When it exists, we call a solution $T^\star$ to~\ref{eq:gromov-monge-problem} a \textit{Gromov-Monge map} for costs $c_\+X,c_\+Y$. 
However, this formulation is ill-suited for discrete distributions $p,q$, as the constraint set might be empty in that case. Replacing maps by coupling $\pi \in \Pi(p, q)$, i.e.\ distributions on $\+X \times \+Y$ with marginals $p$ and $q$, we obtain the GW problem~\citep{memoli2011gromov,sturm2020space}
\begin{equation}
\label{eq:gromov-wasserstein-problem}
\tag*{(GWP)}
\mathrm{GW}(p,q) := \min_{\pi \in \Pi(p, q)} 
\int_{(\+X\times\+Y)^2} 
\tfrac{1}{2} |
c_\+X(\*x,\*x') - c_\+Y(\*y, \*y')
|^2
\diff\pi(\*x,\*y)\diff\pi(\*x', \*y')\,.
\end{equation}
A solution $\pi^\star$ to \ref{eq:gromov-wasserstein-problem} always exists, making $\mathrm{GW}(p,q)$ a well-defined quantity. It quantifies the minimal distortion of the geometries induced by $c_\+X$ and $c_\+Y$ achievable when coupling $p$ and $q$. 

\paragraph{Discrete Solvers.}

When both $p$ and $q$ are instantiated as samples, GW Prob.~\ref{eq:gromov-wasserstein-problem} translates to a quadratic assignment problem, whose objective can be regularized using entropy~\citep{cuturi2013sinkhorn,peyre2016gromov}. For empirical measures $ \emp*p = \frac{1}{n} \sum_{i=1}^n \delta_{\*x_i}$, $\emp*q = \frac{1}{n} \sum_{j=1}^n \delta_{\*y_j}$ and $\varepsilon \geq 0$, we set:
\begin{equation}
\label{eq:entropic-gromov-wasserstein}
\tag*{(EGWP)}
\mathrm{GW}_\varepsilon(\emp*p, \emp*q) \eqdef \min_{\*P \in U_n} \sum_{i,j,i',j'=1}^n (\*C_{\+X_{i,i'}}-\*C_{\+Y_{j,j'}})^2 \, \*P_{i,j}\*P_{i',j'}  - \varepsilon H(\*P)\,,
\end{equation}
with $\*C_\+X = [c_\+X(\*x_i,\*x_{i'})]_{i,i'}$, $\*C_\+Y = [c_\+Y(\*y_j,\*y_{j'})]_{j,j'} \in \mathbb{R}^{n \times n}$, $U_n = \{\*P \in \mathbb{R}^{n \times n}_+, \*P \*1_n = \*P^T \*1_n = \tfrac{1}{n} \*1_n \}$ and $H(\*P) = - \sum_{i,j=1}^n \*P_{i,j} \log(\*P_{i,j})$. As $\varepsilon\to 0$, we recover $\mathrm{GW}^{c_\+X,c_\+Y}_0 = \mathrm{GW}^{c_\+X,c_\+Y}$. 
Entropic regularization improves computational performance, as we can solve~\ref{eq:entropic-gromov-wasserstein} using a scheme that iterates the Sinkhorn algorithm \rev{(see Appendix~\ref{app:ablation-epsilon} for full details)}.
This solver has $\mathcal{O}(n^2)$ memory complexity. Its time complexity depends on $c_\mathcal{X},c_\mathcal{Y}$. For general $c_\+X,c_\+Y$, it runs in $\mathcal{O}(n^3)$. However, for the most common practical choices of $c_\mathcal{X} = c_\mathcal{Y} = \langle \cdot, \cdot \rangle$ or $c_\mathcal{X} = c_\mathcal{Y} = \|\cdot - \cdot\|_2^2$, it can also be reduced to $\mathcal{O}(n^2(d_\mathcal{X} + d_\mathcal{Y}))$, as detailed by~\citet[\S 3 \& Alg.\ 2]{scetbon2022linear}.

\subsection{\rev{Distortion in Representation Learning}}
\label{sec:distortion}
\rev{Given an arbitrary map $T : \mathcal{X} \to \mathcal{Y}$, we consider how it can be learned to preserve predefined geometric features. In a VAE, $T$ may represent either the encoder $e_\phi$, which generates latent codes from the data, or the decoder $d_\theta$, which reconstructs the data from these codes. In the case of the encoder, $\mathcal{X}$ corresponds to the data, and $\mathcal{Y}$ is the latent space, while for the decoder, these roles are swapped. Assuming that $d_\theta$ perfectly reconstructs the data from the latents produced by $e_\phi$, i.e., $e_\phi \circ d_\theta = \mathrm{I}d$, the preservation of geometric features by either $e_\phi$ or $d_\theta$ becomes equivalent. Therefore, in the following sections, we refer to $T$ as either the encoder or the decoder without loss of generality. 

We encode geometric features using a cost function for each domain: $c_\mathcal{X} : \mathcal{X} \times \mathcal{X} \to \mathbb{R}$ and $c_\mathcal{Y} : \mathcal{Y} \times \mathcal{Y} \to \mathbb{R}$. Ideally, $T$ should preserve geometry, which means that $T$ preserves costs, that is, $c_\mathcal{X}(\mathbf{x}, \mathbf{x'}) \approx c_\mathcal{Y}(T(\mathbf{x}), T(\mathbf{x'}))$ for $\mathbf{x}, \mathbf{x'} \in \mathcal{X}$.  In practice, two types of costs are often used:}

\begin{itemize}[leftmargin=.5cm,itemsep=.0cm,topsep=0cm]

\item[\textbf{[i]}] \textbf{(Scaled) squared L2 distance}: $c_\+X(\*x, \*x') = \|\*x - \*x'\|_2^2$ and  $c_\+Y(\*y, \*y') = \alpha^2\| \*y - \*y \|_2^2$, with $\alpha >0$. A map $T$ preserving $c_\+X,c_\+Y$ preserves the scaled distances between the points, i.e.\ it is a \textit{scaled isometry}. When $\alpha = 1$, we recover the standard definition of an \textit{isometry}.

\item[\textbf{[ii]}] \textbf{Cosine-Similarity}: $c_\+X(\*x, \*x') = \textrm{cos-sim}(\*x, \*x') \eqdef \langle \tfrac{\*x}{\|\*x\|_2}, \tfrac{\*x'}{\|\*x'\|_2} \rangle$ and $c_\+Y(\*y, \*y') = \textrm{cos-sim}(\*y, \*y')$ similarly. On has $\textrm{cos-sim}(\*x, \*x') = \cos(\theta_{\*x,\*x'})$ where $\theta_{\*x,\*x'}$ is the angle between $\*x$ and $\*x'$. A map $T$ preserving $c_\+X,c_\+Y$ then preserves the angles between the points, i.e.\ it is a \textit{conformal map}. Note that if $T$ is (scaled) isometry (see above), it is a conformal map.
\end{itemize}

We refer to these costs via $\mathbf{L2^2}$ for \textbf{[i]} with $\alpha=1$, $\mathbf{ScL2^2}$ for \textbf{[i]} with $\alpha\neq1$ and $\mathbf{Cos}$ for \textbf{[ii]}.
Introducing a reference distribution \( r \in \mathcal{P}(\mathcal{X}) \), weighting the areas of $\+X$ where we penalize deviations of $c_\+X(\*x, \*x')$ from $c_\+Y(T(\*x), T(\*x'))$, we can quantify this property using the following criterion:

\begin{definition}[Distortion]
The distortion (DST) of a map $T$, for cost functions $c_\+X,c_\+Y$ and reference distribution $r$, is defined as:
\begin{align}
\label{def:distortion}
\tag{\dst}
\mathrm{DST}_r(T) := \int_{\+X\times\+X} \tfrac{1}{2} (c_\+X(\*x,\*x') - c_\+Y(T(\*x), T(\*x')))^2 \diff r(\*x)\diff r(\*x')\,.
\end{align}
\end{definition}

$\mathrm{DST}_r(T)$ quantifies how much $T$ distorts geometric features encoded by $c_\+X,c_\+Y$ on the support of $r$, that is, when $\mathrm{DST}_r(T)=0$, one has $c_\+X(\*x, \*x') = c_\+Y(T(\*x), T(\*x'))$ for $\*x,\*x'\in\supp(r)$. 

\textbf{Distortion as a Loss for Representation Learning.} 
\citet{nakagawa2023gromovwasserstein} suggest promoting geometry preservation by regularizing the encoder $d_\theta$ using the \ref*{def:distortion}, with $\mathbf{ScL2^2}$ as costs, and the latent representation as reference distribution, namely $r = e_\phi \sharp p_\textrm{data}$. While they use it within a WAE, \rev{we translate their objective to VAE setting adopted later in the paper}. This results in:
\begin{equation}
\label{eq:beta-vae}
    \!\!\!\!\!\min_{\theta,\phi}\mathbb{E}_{\*x\sim p_\textrm{data},\*z\sim q_\phi(\*z|\*x)}[\underbrace{-\log p_\theta(\*x|\*z)}_{\textrm{(i) reconstruction}} + \underbrace{\beta D_{\text{KL}}(q_\phi(\*z|\*x)||p(\*z))}_{\textrm{(ii) prior matching}}] + \underbrace{\lambda\, \mathrm{DST}_{r}(d_\theta)}_{\textrm{(iii) geom. preservation}},  \quad\lambda > 0
\end{equation}
Given the choice of the costs, $d_\theta$ is distortion-free (i.e., $\mathrm{DST}_{r}(d_\theta) = 0$), if it is a scaled isometry. 

\textbf{\rev{Challenges Arising from a Mixed Loss.}} Since a scaled isometry that maps the prior onto the data distribution may not exist, there is an inherent trade-off between minimizing terms (ii) and (iii), which are responsible for achieving practical disentanglement. 
As these terms cannot be simultaneously minimized to $0$, the~\ref*{def:distortion} loss will move away from accurately matching the prior, which will negatively impact the quality of the learned latent representations. This naturally raises the question of how to avoid this over-penalization. \rev{Instead} of seeking a distortion-free decoder, we should seek a decoder that transports the prior to the data distribution with \textit{minimal distortion} of the costs $c_\+X,c_\+Y$. In other words, the decoder should be a Gromov-Monge map between the prior and the data distributions for costs $c_\+X,c_\+Y$ (see~\ref{subsec:quadratic-ot}). Conversely, if we choose to regularize the encoder, the same reasoning applies by swapping the roles of the prior and the data distribution. In the next section, we introduce a regularizer to fit Gromov-Monge maps, which we will use as a replacement to the~\ref*{def:distortion}.

\section{Disentanglement with the Gromov-Monge gap}
\label{sec:method}

\rev{Building on these geometric preservation principles, we} introduce in \S\ref{sec:to-the-gmg} the Gromov-Monge Gap (GMG), a regularizer that measures whether a map moves distributions while preserving geometric features as much as possible, i.e., minimizing distortion \rev{while fitting the marginal constraints}.
\S\ref{subsec:gmg-estimation} then shows how the GMG can be efficiently computed from samples to be practically applicable in the VAE framework. This transitions into \S\ref{sec:properties} studying (weak) convexity properties of the GMG, as an operator. Finally, in \S\ref{sec:learning-with-the-gmg}, we describe how to integrate the GMG with disentangled representation learning objectives, effectively combining prior matching with geometric constraints. %

\subsection{\rev{An Efficient Gap Formulation for Distortion}} 
\label{sec:to-the-gmg}

\begin{figure}[t]
    \vspace{-11mm}
   \includegraphics[width=1\linewidth]
   {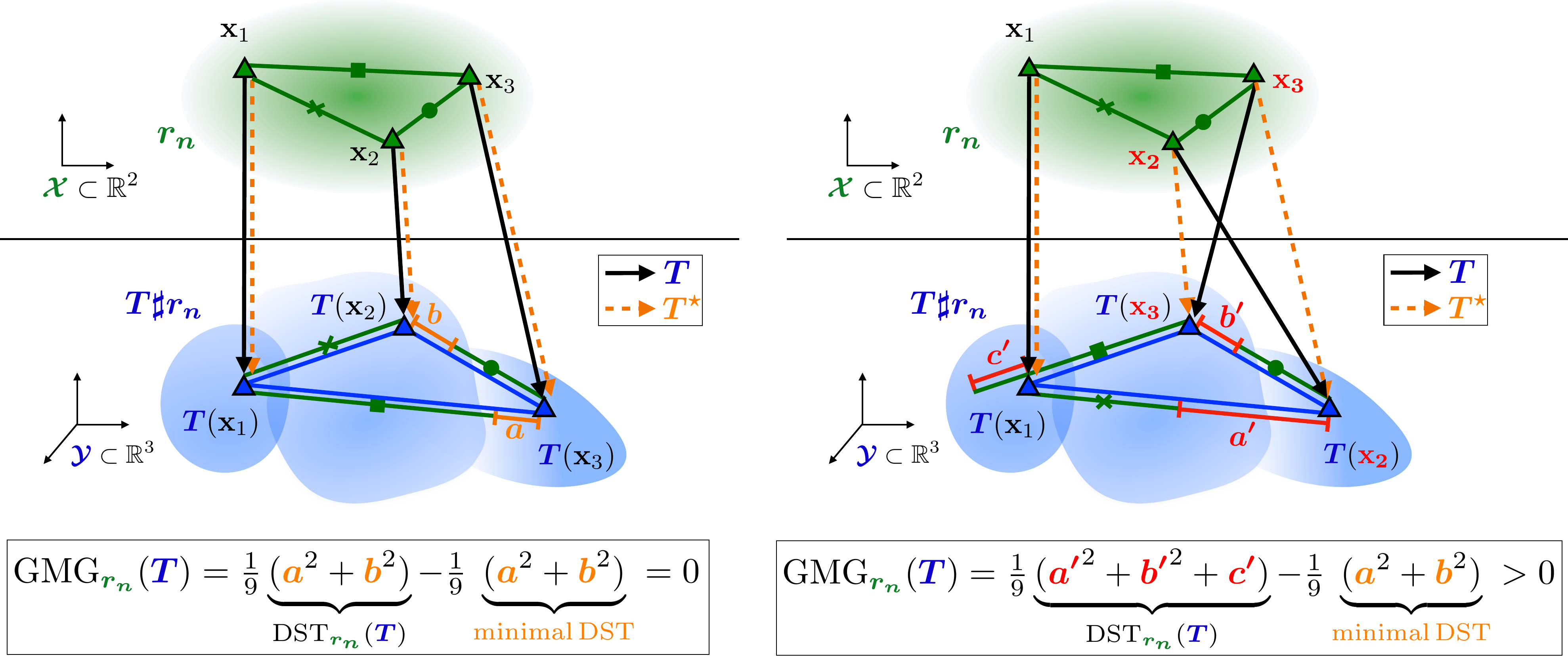}
   \vspace{-3mm}
   \caption{Sketch of $\mathrm{GMG}_{\EGGE{r_n}}(\EBGE{T})$ for two different maps $\EBGE{T}$. We use a discrete reference distribution $\EGGE{r_n}$ on 3 points, and $c_\+X=c_\+Y=\|\cdot-\cdot\|_2$, hence measuring if $\EBGE{T}$ minimally distorts the distances. 
   On the left, $\EBGE{T}$ is the optimal map, $\EORG{T^\star}$, and maps the three points with minimal (yet non-zero) distortion, which is measured as the sum of the squared lengths of the orange segments. This results in $\mathrm{GMG}_{\EGGE{r_n}}(\EBGE{T}) = 0$. On the right, $\EBGE{T}$ swaps two points compared to $\EORG{T^\star}$—specifically, $\ERED{\*x_2}$ and $\ERED{\*x_3}$—causing a higher distortion than the minimal one, and measured as the sum of the squared lengths of the red segments. This results in $\mathrm{GMG}_{\EGGE{r_n}}(\EBGE{T}) > 0$, equal to the gap between these distortions.}
   \vspace{-2mm}
\label{fig:gmg-concept-figure}
\end{figure}

Recently, \citet{uscidda2023monge} introduced the Monge gap, a regularizer that measures whether a map $T$ transports a reference distribution at the minimal displacement cost. Building on this concept, we replace "displacement" with "distortion" to introduce the Gromov-Monge gap, a regularizer that assesses whether a map $T$ transports a reference distribution at the minimal distortion cost.

\begin{definition}[Gromov-Monge gap]
The Gromov-Monge gap (GMG) of a map $T$, for cost functions $c_\+X,c_\+Y$ and reference distribution $r$, is defined as:
\vspace{1mm}
\begin{align}
\label{def:gromov-monge-gap}
\tag{\gmg}
\mathrm{GMG}_r(T) 
\eqdef \mathrm{DST}_r(T) - \mathrm{GW}(r, T\sharp r) 
\end{align}
\end{definition}
We recall from Eq.~\ref{eq:gromov-wasserstein-problem} that $\mathrm{GW}(r, T\sharp r)$ represents the minimal distortion of $c_\+X,c_\+Y$ achievable when transporting $r$ to $T\sharp r$. Thus, the~\ref*{def:gromov-monge-gap} quantifies the difference between the distortion incurred when transporting $r$ to $T\sharp r$ via $T$, and this minimal distortion. Formally, when Prob.~\ref{eq:gromov-monge-problem} and  Prob.~\ref{eq:gromov-wasserstein-problem} between $r$ and $T\sharp r$ are equivalent, the the~\ref*{def:gromov-monge-gap} is the suboptimality gap of $T$ in Prob.~\ref{eq:gromov-monge-problem}. This is the case, for example, when $r$ is a density and $c_\+X=c_\+Y=\langle\cdot,\cdot\rangle$~\citep{dumont2022existence}. Otherwise, the~\ref*{def:gromov-monge-gap} is the suboptimality gap of $\pi = (\mathrm{I}_d,T)\sharp r$ in Prob.~\ref{eq:gromov-wasserstein-problem} between $r$ and $T\sharp r$. In light of this, it is a well-defined quantity and:

\begin{itemize}[leftmargin=.5cm,itemsep=.0cm,topsep=0cm]
\item \textbf{The GMG measures how close $T$ is to be a Gromov-Monge map for costs $c_\+X,c_\+Y$.} Indeed, $\mathrm{GMG}_r(T) \geq 0$ with equality if $T$ is a Gromov-Monge map solution of Prob.~\ref{eq:gromov-monge-problem} between $r$ and $T\sharp r$, i.e., $T$ moves $r$ with minimal (but eventually non zero) distortion; see App.~\ref{sec:gmg-equals-0}.

\item \textbf{When transport without distortion is possible, the GMG coincides with the distortion.} When there exists another map \( U : \mathcal{X} \to \mathcal{Y} \) transporting \( r \) to \( T\sharp r \) with zero distortion, i.e., \( U\sharp r = T\sharp r \) and \( \mathrm{DST}_r(U) = 0 \), then \( \mathrm{GMG}_r(T) = \mathrm{DST}_r(T) \). Indeed, \( \mathrm{GW}(r, T\sharp r) = 0 \) in that case, as the coupling \( \pi = (\mathrm{Id}, U)\sharp r \) sets the GW objective to zero, thereby minimizing it.

\end{itemize}

The last point illustrates how the~\ref*{def:gromov-monge-gap} functions as a \textit{debiased distortion}. It compares the distortion induced by \( T \) to a baseline distortion, defined as the minimal achievable distortion when transforming $r$ into $T\sharp r$. Thus, when transformation without distortion is achievable, this baseline becomes zero, and the \ref*{def:gromov-monge-gap} equals the distortion. Consequently, the \ref*{def:gromov-monge-gap} offers the optimal compromise: it avoids the over-penalization induced by the distortion when fully preserving \( c_{\+X}, c_{\+Y} \) is not feasible, yet it coincides with it when such full preservation is feasible. See Fig.~\ref{fig:gmg-concept-figure} for a simple illustration.

\textbf{The Influence of the Reference Distribution.} 
A crucial property of \(\mathrm{DST}_r\) is that if \(T\) transforms \(r\) without distortion, it will also apply distortion-free to any distribution \(s\) whose support is contained within that of \(r\). Formally, if $\mathrm{DST}_r(T) = 0$ and \(s \in \mathcal{P}(\mathcal{X})\) with $\supp(s) \subseteq \supp(r)$, then \(\mathrm{DST}_s(T) = 0\). This raises a question for the GMG: If \(T\) maps \(r\) with minimal distortion, does it similarly map $s$ with minimal distortion? We answer this question positively with Prop.~\eqref{prop:more-do-less}.

\begin{restatable}
{proposition}{PropMoreDoLess}
\label{prop:more-do-less}
    If $\mathrm{GMG}_r(T) = 0$, $\forall s \in \+P(\+X)$ s.t.\ $\supp(s) \subseteq \supp(r)$, one has $\mathrm{GMG}_s(T) = 0$.
\end{restatable}

\subsection{Estimation and Computation from Samples}\label{subsec:gmg-estimation}

\textbf{Plug-In Estimation.} In practice, we estimate Eq.~\eqref{def:distortion} and Eq.~\eqref{def:gromov-monge-gap} from samples $\*x_1, ..., \*x_n \sim r$. We consider the empirical version $\emp*r \eqdef \tfrac{1}{n} \sum_{i=1}^n \delta_{\*x_i}$ of $r$ and use plug-in estimators, i.e.
\begin{equation}
\begin{split}
\label{eq:distortion-estimator}
\mathrm{DST}_{\emp*r}(T) 
= 
\tfrac{1}{n^2}
\sum_{i,j=1}^n (c_\+X(\*x_i, \*x_j) - c_\+Y(T(\*x_i), T(\*x_j)))^2\,,
\end{split}
\end{equation}
and $\mathrm{GMG}_{\emp*r}(T) = \mathrm{DST}_{\emp*r}(T) - \mathrm{GW}(\emp*r, T \sharp \emp*r)$, where $T \sharp \emp*r = \frac{1}{n} \sum_{i=1}^n \delta_{T(\*x_i)}$.

\begin{wrapfigure}{r}{0.59\textwidth}
\begin{minipage}{0.59\textwidth}
\vspace{-8mm}
\begin{algorithm}[H]
\caption{\textsc{GMG}$(\*x_1,\dots\*x_n, T, \varepsilon)$.}
\label{algo:entropic-gmg}
\begin{algorithmic}[1]
    \State{\textbf{Require:} $\*x_1,\dots,\*x_n \sim r$; map $T$; entropic solver $\texttt{GW}$, entropic regularization scale $\varepsilon_0$ (default $=0.1$), statistic operator on cost matrix \texttt{stat} (default $=\texttt{mean}$).}
    \State{$\*t_1,\dots\*t_n \gets T(\*x_1),\dots,T(\*x_n)$.}
    \State{$\*{C}_\+X\gets [c_\+X(\*x_i,\*x_{i'})]_{1\leq i,i'\leq n}$ \,\, \Comment{usually \teal{$\mathcal{O}(n^2d_\+X)$}}}
    \State{$\*{C}_\+Y\gets [c_\+Y(\*t_j,\*t_{j'})]_{1\leq j,j'\leq n}$ \,\, \Comment{usually \teal{$\mathcal{O}(n^2d_\+Y$)}}}
    \State{$\textsc{DST} \gets \texttt{mean}((\*C_\+X - \*C_\+Y)^2)$ \Comment{\teal{$n^2$}}}
    \State{$\varepsilon \gets \varepsilon_0 \cdot \texttt{stat}(\*C_\+X) \cdot \texttt{stat}(\*C_\+Y)$ \Comment{usually \teal{$\mathcal{O}(n^2)$}}}
    \State{$\textsc{GW} \gets \texttt{GW}(\*C_\+X,\*C_\+Y, \varepsilon)$ \Comment{\teal{$\mathcal{O}(n^3)$}} or \teal{$\mathcal{O}(n^2(d_\+X+d_\+Y))$}}
    \State{{\bfseries return}{ \textsc{DST} - \textsc{GW}}}
\end{algorithmic}
\end{algorithm}
\end{minipage}
\vspace{-3mm}
\end{wrapfigure}

\paragraph{Efficient and Stable Computation.} Computing the \ref*{def:gromov-monge-gap} requires solving a discrete GW problem between $r_n$ and $T\sharp r_n$ to get $\mathrm{GW}(\emp*r, T\sharp\emp*r)$. We compute this term using an entropic regularization $\varepsilon \geq 0$, as in Eq.~\ref{eq:entropic-gromov-wasserstein}:
\begin{align}
\label{eq:entropic-gmg}
& \mathrm{GMG}_{\emp*r,\varepsilon}(T) \eqdef
 \mathrm{DST}_{\emp*r}(T) \\\notag
& \,\, - \mathrm{GW}_{\varepsilon}(\emp*r, T \sharp \emp*r)\,.
\end{align}
Choosing $\varepsilon = 0$, we recover $\mathrm{GMG}_{\emp*r,0}(T) = \mathrm{GMG}_{\emp*r}(T)$. Moreover, the entropic estimator preserves positivity, as for $\varepsilon \geq 0$, $\mathrm{GMG}_{r_n, \varepsilon}(T) \geq 0$ (see~\ref{sec:positivity-entropic-estimator}). We compute $\mathrm{GW}_\varepsilon(\emp*r, T \sharp \emp*r)$ using \citet{peyre2016gromov}'s solver introduced in~\ref{subsec:quadratic-ot}. We use the implementation provided by \texttt{ott-jax}~\citep{cuturi2022optimal}. 
In practice, we select \( \varepsilon \) based on (positive) statistics from the cost matrices \( \*C_\+X \),\( \*C_\+Y \). We define a scale \( \varepsilon_0 \) and set \( \varepsilon = \varepsilon_0 \cdot \texttt{stat}(\*C_\+X) \cdot \texttt{stat}(\*C_\+Y) \). Standard options for the statistic include \( \texttt{stat} \in \{\texttt{mean}, \texttt{max}, \texttt{std}\} \). This procedure is equivalent to running the entropic GW solver on the re-scaled cost matrices \( \*C_\+X/\texttt{stat}(\*C_\+X) \) and \( \*C_\+Y/\texttt{stat}(\*C_\+Y) \) with \( \varepsilon = \varepsilon_0 \); see App.~\ref{sec:rescaling-costs}. It is a common practical trick, initially suggested for stabilizing the Sinkhorn algorithm~\cite{cuturi2013sinkhorn}. 

\paragraph{Computational Complexity.} For usual costs, including inner products, $\ell_p^q$ distances, and standard CPD kernels, the computation of $c_\+X(\*x, \*x')$ (resp.\ $c_\+Y(\*y,\*y')$) can be done in $\mathcal{O}(d_\+X)$ time (resp.\ $\mathcal{O}(d_\+Y)$). Consequently, the~\ref*{def:distortion} can be computed in $\mathcal{O}(n^2(d_\+X+d_\+Y))$ time. Furthermore, as discussed in \S~\ref{subsec:quadratic-ot}, the time complexity of the entropic GW solver is $\mathcal{O}(n^3)$ in general, but can be reduced to $\mathcal{O}(n^2(d_\+X+d_\+Y))$ when $c_\+X=c_\+Y=\langle\cdot,\cdot\rangle$, or $c_\+X=\|\cdot-\cdot\|_2^2, c_\+Y=\alpha\|\cdot-\cdot\|_2^2$. Therefore, since the cosine similarity is equivalent to the inner product, up to pre-normalization of $\*x_i$
and $T(\*x_i)$, this solver runs in $\mathcal{O}(n^2(d_\+X+d_\+Y))$ for the costs of interest $\mathbf{(Sc)L2^2}$ and $\mathbf{Cos}$. The complete algorithm, along with a time complexity analysis of each step, is described in Alg.~\ref{algo:entropic-gmg}. \rev{We stress that, in all cases and for any cost function, the time complexity depends linearly on the dimensions of the source and target spaces, $d_\mathcal{X}$ and $d_\mathcal{Y}$, making the GMG scalable to high-dimensional distributions.}

\subsection{(Weak) Convexity of the Gromov-Monge gap} 
\label{sec:properties}

As laid out, the \ref*{def:gromov-monge-gap} can be used as a regularization loss to push any model $T$ to move distributions with minimal distortion. A natural question arises: is this regularizer convex? In the following, we study the convexity of $T \mapsto \mathrm{GMG}_{r}(T)$ and its finite-sample counterpart $T \mapsto \mathrm{GMG}_{\emp*r}(T)$. We focus on the costs $\mathbf{L2}$ and $\mathbf{Cos}$. For simplicity, we replace $\mathbf{Cos}$ with $\langle \cdot, \cdot \rangle$, as these costs are equivalent, up to normalization of $r$ and $T$. We respectively denote by $\mathrm{GMG}_r^2$ and $\mathrm{GMG}_r^{\langle \cdot, \cdot \rangle}$ the~\ref*{def:gromov-monge-gap} for these costs. 
We start by introducing a weaker notion of convexity, previously defined on $\mathbb{R}^d$~\citep{davis2018subgradient}, which we extend here to $L_2(r) \eqdef \{T \,|\, \|T\|_{L_2(r)}^2 \eqdef \int_\+X \|T(\*x)\|_2^2 \diff r(\*x) < +\infty\}$.

\begin{definition}[Weak convexity.] 
\label{def:weak-convexity}
With $\gamma > 0$, a functional $\mathcal{F} : L_2(r) \to \mathbb{R}$ is $\gamma$-weakly convex if $\mathcal{F}_\gamma : T \mapsto \mathcal{F}(T) + \tfrac{\gamma}{2} \|T\|_{L_2(r)}^2$ is convex.
\end{definition}

A weakly convex functional is convex up to an additive quadratic perturbation. The weak convexity constant $\gamma$ quantifies the magnitude of this perturbation and indicates a degree of non-convexity of $\mathcal{F}$. A lower $\gamma$ suggests that $\mathcal{F}$ is closer to being convex, while a higher $\gamma$ indicates greater non-convexity.

\begin{restatable}
{theorem}{WeakConvexity}
\label{thm:weak-convexity}
Both $\mathrm{GMG}_r^2$ and $\mathrm{GMG}_r^{\langle \cdot, \cdot \rangle}$, as well as their finite sample versions, are  weakly convex.

\begin{itemize}[leftmargin=.5cm,itemsep=.0cm,topsep=0cm]

\item \textbf{Finite sample.} We note $\*X\in\mathbb{R}^{n\times d}$ the matrix that stores the $\*x_i$, i.e.\ the support of $r_n$, as rows. %
Then, (i) $\mathrm{GMG}_{r_n}^2$ and (ii) $\mathrm{GMG}_{r_n}^{\langle \cdot, \cdot \rangle}$ are respectively (i) $\gamma_{2,n}$ and (ii) $\gamma_{\textrm{inner},n}$-weakly convex, where: $\gamma_{\textrm{inner},n} = \lambda_{\max}(\tfrac{1}{n}\*X\*X^\top) - \lambda_{\min}(\tfrac{1}{n}\*X\*X^\top)$ and $\gamma_{2,n} = \gamma_{\textrm{inner},n} + 
\max_{i=1\dots n}\|\*x_i\|_2^2$. 

\item \textbf{Asymptotic.} (i) $\mathrm{GMG}_{r}^2$ and (ii) $\mathrm{GMG}_{r}^{\langle \cdot, \cdot \rangle}$ are respectively (i) $\gamma_{2}$ and (ii) $\gamma_{\textrm{inner}}$-weakly convex, where: $\gamma_{\textrm{inner}} = \lambda_{\max}(\mathbb{E}_{\*x\sim r}[\*x\*x^\top]) $ and $\gamma_{2,n} = \gamma_{\textrm{inner}} + \max_{\*x\in \supp(r)}\|\*x\|_2^2$.

\end{itemize}
\end{restatable}

From a practitioner's perspective, we analyze the insights provided by Thm.~\eqref{thm:weak-convexity} in three parts.
\begin{itemize}[leftmargin=.4cm,itemsep=.0cm,topsep=0cm]

\item First, we have \(\gamma_2 \geq \gamma_\textrm{inner}\). Therefore, \(\mathrm{GMG}_{r}^2\) is less convex than \(\mathrm{GMG}_{r}^{\langle \cdot, \cdot \rangle}\), making it harder to optimize, and the same argument holds for their estimator. In other words, we provably recover that, in practice, preserving the (scaled) distances is harder than simply preserving the angles.

\item Second, as $\gamma_\textrm{inner} = \lambda_{\max}(\mathbb{E}_{\*x\sim r}[\*x\*x^\top]) \geq \lambda_{\max}(\mathrm{Cov}_{\*x\sim r}[\*x])$, this exhibits a tradeoff w.r.t.\ Prop.~\eqref{prop:more-do-less}: by choosing a bigger reference distribution $r$, we trade the convexity of the GMG. For $\gamma_2$, the dependency in $r$ is even worse. In practice, we then choose $r$ with support as small as possible, precisely where we want $T$ to move points with minimal distortion. 

\item Third, and most surprising, the finite sample \ref*{def:gromov-monge-gap} is more convex in high dimension. Indeed, $\gamma_{\textrm{inner},n}$ is the spectral width of $\tfrac{1}{n}\*X\*X^\top$, which contains the (rescaled) inner products between $\*x_i \sim r$. 
When \(n > d\), \(\lambda_{\min} (\mathbf{X}\mathbf{X}^\top) = 0\) as \(\mathrm{rank}(\mathbf{X}\mathbf{X}^\top) = d\). Then, \(\gamma_{\textrm{inner},n}\) increases, which in turn decreases the GMG's convexity. However, when $d >n$, $\lambda_{\min} (\*X\*X^\top) > 0$ if $\*X$ is full rank. Intuitively, \(\mathrm{GMG}_{r_n}^{\langle \cdot, \cdot \rangle}\) is nearly convex when $\*X\*X^\top$ is well conditioned. Assuming that the $\*x_i$ are normalized, this might happen in high dimension, as they will be orthogonal with high probability. This suggests that, contrary to the insights provided by the statistical OT literature~\citep{weed2017sharp,zhang2023gromovwasserstein}, the \ref*{def:gromov-monge-gap} might not benefit a large sample size.

\end{itemize}

\subsection{Learning with the Gromov-Monge gap}
\label{sec:learning-with-the-gmg}

\paragraph{General Learning Procedure.} Given a source \( p \) and a target distribution \( q \), we can use the \ref*{def:gromov-monge-gap} to guide a parameterized map \( T_\theta \) towards approximating a Gromov-Monge map between \( p \) and \( q \). We handle the marginal constraint $T_\theta\sharp p = q$ separately through a fitting loss \( \Delta(T_\theta, p, q) \). Provided any reference $r$ s.t.\ $\supp(p) \subset \supp(r)$, we minimize 
\begin{equation}
\label{eq:general-loss}
    \min_\theta \Delta(T_\theta, p, q) + \lambda\mathrm{GMG}_r(T_\theta)
\end{equation}

$\Delta$ can operate on paired (e.g., in VAE, the reconstruction loss), or unpaired (e.g., in VAE, the KL loss in VAE), samples of $p$ and $q$. 
Note that, in theory and as stated in Prop.~\eqref{prop:more-do-less}, we can select any reference $r$ such that $\supp(p) \subset \supp(r)$. However, based on the insights from Thm.~\eqref{def:weak-convexity}, we typically choose $r$ with minimal support size and, in practice, set $r=p$. This learning procedure is illustrated in Fig.~\ref{fig:gmg}, where we also explore the effect of replacing $\mathrm{GMG}_r(T_\theta)$ by $\mathrm{DST}_r(T_\theta)$ in Eq.~\eqref{eq:general-loss}.

\paragraph{VAE Learning Procedure.} In the VAE setting, we can use the \ref*{def:gromov-monge-gap} promote the (i) encoder $e_\phi$ or the (ii) decoder $d_\theta$ to mimic a Gromov-Monge map. In (i) we use $r_e = p_\textrm{data}$ the data distribution as reference $r$, while in (ii) we use the latent distribution $r_d = e_\phi\sharp p_\textrm{data}$. Introducing weightings $\lambda_e,\lambda_d \geq 0$, determining which mapping we regularize, this remains to minimize
\begin{equation}
    \min_{\theta,\phi} \mathbb{E}_{\*x\sim p_\textrm{data},\*z\sim q_\phi(\*z|\*x)}[\underbrace{-\log p_\theta(\*x|\*z)}_{\textrm{(i) reconstruction}} + \underbrace{\beta D_{\text{KL}}(q_\phi(\*z|\*x)||p(\*z))}_{\textrm{(ii) prior matching}}] + \underbrace{\lambda_e \mathrm{GMG}_{r_e}(e_\phi) + \lambda_d\mathrm{GMG}_{r_d}(d_\theta)}_{\textrm{(iii) geom. preservation}},
\end{equation}
With this loss, prior matching and geometric desiderata can be efficiently combined, as terms (ii) and (iii) can simultaneously be 0. Note that this loss can be easily extended to more advanced prior matching objectives, such as the \(\beta\)-TCVAE loss~\citep{chen2018betatcvae}, and can be combined with other regularizers, such as the HFS~\citep{roth2023disentanglement}. We explore this strategy in experiments \S~\ref{sec:experiments}.

\begin{figure}[t]
    \vspace{-11mm}
   \includegraphics[width=0.97\linewidth]
   {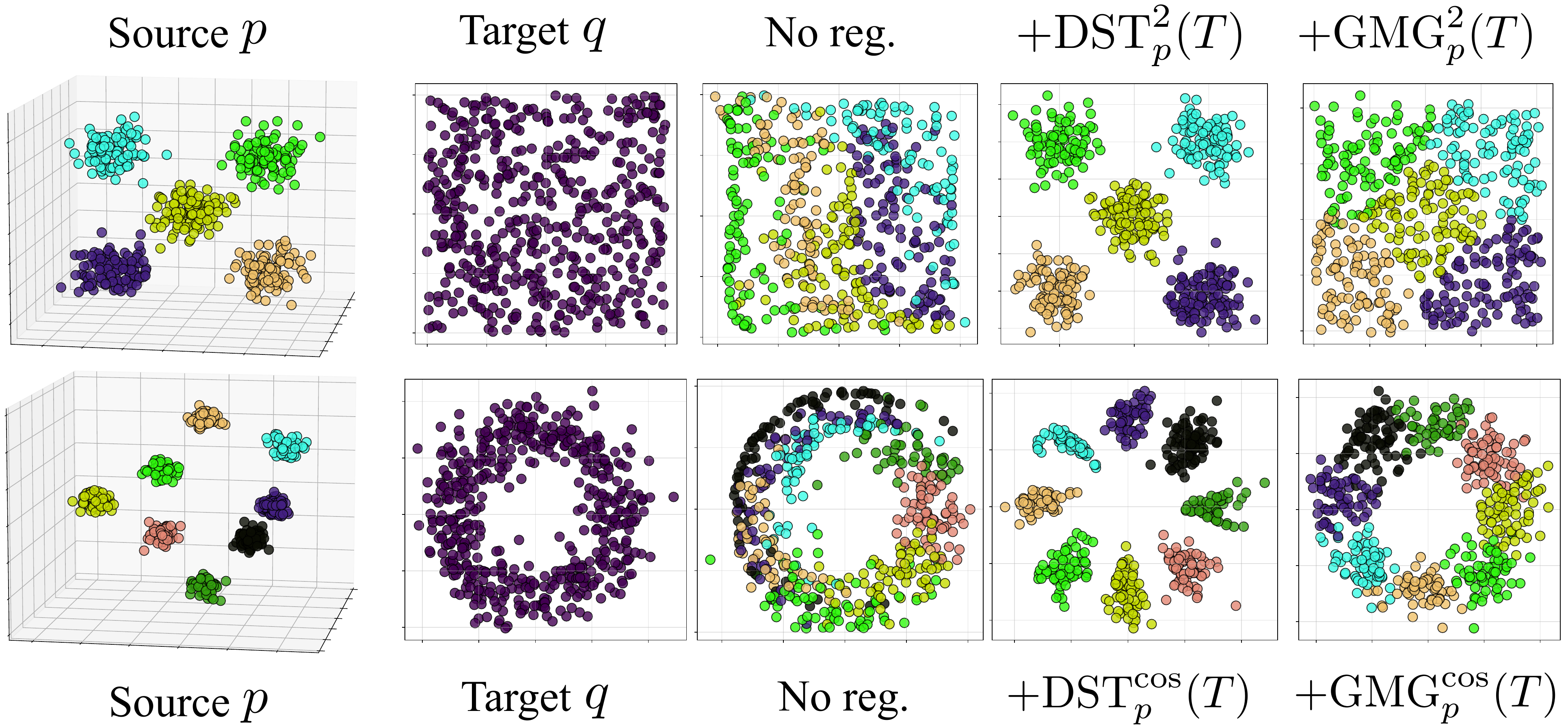}
   \vspace{-1mm}
   \caption{Learning of geometry-preserving maps with the~\ref*{def:distortion} and the~\ref*{def:gromov-monge-gap}. Provided a mixture of Gaussian source distribution $p$ and a uniform or circular target distribution $q$, we minimize a fitting loss together with a geometry-preserving regularization. Each line correspond to a different data setup and cost. As costs on the top line, we use $\mathbf{L2^2}=\|\cdot-\cdot\|_2^2$, to preserve distances between points while on the bottom line, we use $\mathbf{Cos}=\textrm{cos-sim}(\cdot,\cdot)$, to preserve angles.}
   \vspace{-3mm}
\label{fig:gmg}
\end{figure}

\section{Experiments}

\subsection{GMG vs. DST: Illustrative Example}
\label{sec:experiments-synthetic-data}

In Fig.~\ref{fig:gmg}, we start by illustrating the difference between the \ref*{def:distortion} and \ref*{def:gromov-monge-gap}. We train $T_\theta$ to map from a $3$D source $p$ to a $2$D target $q$ (first and second column) by minimizing $\mathcal{L}(\theta) \eqdef \mathrm{S}_{\varepsilon}(T_\theta\sharp p, q) + \mathrm{R}(T_\theta)$, where $\mathrm{S}_{\varepsilon}$ is the Sinkhorn divergence~\citep{feydy2018interpolating}. We compare three settings, [i] no regularization $\mathrm{R}=0$, [ii] $\mathrm{R} = \mathrm{DST}_p$ as regularizer , and [iii] $\mathrm{R} = \mathrm{GMG}_p$. For all cases, we plot the transported distribution $T_\theta\sharp p$ after training. Without regularization (third column), we fit the marginal constraint $T_\theta\sharp p= q$ but do not preserve the geometric features. With the~\ref*{def:distortion} (fourth column), we preserve geometric features but do not fit the marginal constraint, i.e., $T_\theta\sharp p\ne q$. On the other hand, with the~\ref*{def:gromov-monge-gap} (fifth column), we get the best compromise by approximating a Gromov-Monge map: we fit $T_\theta\sharp p= q$, while preserving the geometric features as fully as possible.

\subsection{Leveraging the Gromov-Monge Gap for Disentanglement}
\label{sec:experiments}

\xhdr{Experimental Setup.}
Having demonstrated how the \ref*{def:gromov-monge-gap} enables (i) fitting a marginal constraint while (ii) preserving geometric features as much as possible, we now apply it to disentangled representation learning. 
Our primary goal is to investigate whether the \ref*{def:gromov-monge-gap} results in enhanced disentanglement compared to the \ref*{def:distortion} by efficiently combining (i) prior matching with (ii) geometric constraints on the representation space. Moreover, we aim to determine which $c_\+X ,c_\+Y$ to choose and what part of the pipeline should be regularized, the encoder $e_\phi$, or the decoder $d_\theta$. 
\begin{itemize}[leftmargin=0cm,itemsep=.0cm,topsep=0cm]
\item \textbf{Baselines.} We use the standard $\beta$-VAE and $\beta$-TCVAE as our starting models, with the option to apply the recent HFS regularization ~\citep{roth2023disentanglement} to each, resulting in a total of four base configurations. Note that the latter does not leverage geometric constraints; it is only used to enhance prior matching. We then investigate the effect of various geometry-preserving regularizations on disentanglement on top of these four base configurations. We consider \ref*{def:gromov-monge-gap}, \ref*{def:distortion} and the Jacobian-based (\texttt{Jac}) regularization~\citep{lee2022regularized} discussed in~\ref{sec:disentangled-representational-learning}. Given the inclusion of the \ref*{def:distortion}, we naturally consider~\citet{nakagawa2023gromovwasserstein} as a baseline.

\item \textbf{Metrics.} We evaluate the learned  representations using the \textbf{DCI-D}~\citep{eastwood2018a} as it was found that it is the metric most suitable for measuring the disentanglement~\citep{pmlr-v119-locatello20a,dittadi2021on}. We report mean and standard deviation over 5 seeds.

\item \textbf{Datasets.} We benchmark over four \rev{$64\times64$ image} datasets: Shapes3D~\citep{kim2018factorvae}, DSprites~\citep{higgins2017betavae}, SmallNORB~\citep{lecun2004smallnorb}, and  Cars3D~\citep{reed2015cars3d}.

\item \textbf{Hyperparameters.} To ensure a fair experimental comparison, we follow recent works~\citep{pmlr-v97-locatello19a,pmlr-v119-locatello20a,roth2023disentanglement} by using the same architecture and hyperparameters. We perform a similar small grid search for the weighting terms: $\beta$ for the KL loss and $\gamma$ for HFS. Additionally, we include the weighting $\lambda$ for the geometry-preserving regularizer in the grid search. Note that we search over the same loss weightings $\lambda$ for \ref*{def:distortion}, \ref*{def:gromov-monge-gap}, and \texttt{Jac}. For all experiments with the GMG, we compute it with Alg.~\ref{algo:entropic-gmg}, and \rev{systematically} use $\varepsilon_0 = 0.1$ and $\texttt{stat} = \texttt{mean}$. As a result, $\varepsilon_0$ is \textit{not} included in the grid search. \rev{We conduct an ablation study on $\varepsilon_0$ in App.~\ref{app:ablation-epsilon}}. \rev{We use a batch size of $n=64$. At this scale, the computational cost of compute the GMG loss for a batch is negligible, about 3 milliseconds.} See App.~\ref{app:experimental_details} for full details on hyperparameters.
\end{itemize}

\begin{table}[t]
\centering
\vspace{-5mm}
\caption{Effect of different regularization on disentanglement (DCI-D on Shapes3D). We \textcolor{Periwinkle}{highlight} the best method per regularization type ($\mathbf{L2^2}$, $\mathbf{ScL2^2}$, or $\mathbf{Cos}$), and the \textbf{best}/\underline{second best} per column.}
\vspace{-1mm}
\label{tab:reg-benchmark}
\begin{tabular}{lcccc}
\toprule
& {$\beta$-VAE} & {$\beta$-TCVAE} & $\beta$-VAE + HFS& $\beta$-TCVAE + HFS\\
\midrule
Base & $65.8$ \textcolor{gray}{$\pm 15.6$} & $75.0$ \textcolor{gray}{$\pm 3.4$} & $88.1$ \textcolor{gray}{$\pm 7.4$} & $90.2$ \textcolor{gray}{$\pm 7.5$} \\
\midrule
\multicolumn{5}{c}{$\mathbf{L2^2}: c_\+X=c_\+Y= \|\cdot-\cdot\|_2^2$} \\
\midrule
+ \texttt{Enc}-\ref*{def:distortion} & {$59.6$} \textcolor{gray}{$\pm 6.9$}  & {$75.7$} \textcolor{gray}{$\pm 3.0$} & {$88.7$} \textcolor{gray}{$\pm 7.1$} & {$90.3$} \textcolor{gray}{$\pm 7.9$} \\
+ \texttt{Enc}-\ref*{def:gromov-monge-gap} & {$62.3$} \textcolor{gray}{$\pm 8.4$} & {$75.4$} \textcolor{gray}{$\pm 5.3$} & {$88.4$} \textcolor{gray}{$\pm 7.7$} & {$90.1$} \textcolor{gray}{$\pm 4.3$}\\
+ \texttt{Dec}-\ref*{def:distortion} & $71.5$ \textcolor{gray}{$\pm 3.6$} & $75.8$ \textcolor{gray}{$\pm 6.6$} & $92.1$ \textcolor{gray}{$\pm 9.7$} & $90.9$ \textcolor{gray}{$\pm 7.6$} \\
+ \texttt{Dec}-\ref*{def:gromov-monge-gap} & \textcolor{Periwinkle}{$72.0$} \textcolor{gray}{$\pm 8.5$} & \textcolor{Periwinkle}{$78.9$} \textcolor{gray}{$\pm 5.0$}  & \textcolor{Periwinkle}{$92.5$} \textcolor{gray}{$\pm 4.4$} & \textcolor{Periwinkle}{$91.7$} \textcolor{gray}{$\pm 6.0$}\\
\midrule
\multicolumn{5}{c}{$\mathbf{ScL2^2}:c_\+X=\|\cdot-\cdot\|_2^2, c_\+Y=\alpha\|\cdot-\cdot\|_2^2$, $\alpha>0$ learnable} \\
\midrule
+ \texttt{Jac} & $61.4$ \textcolor{gray}{$\pm 12.8$} & $76.7$ \textcolor{gray}{$\pm 4.5$} & $90.5$ \textcolor{gray}{$\pm 3.8$} & $91.5$ \textcolor{gray}{$\pm 5.6$} \\
+ \texttt{Enc}-\ref*{def:distortion} & {$65.8$} \textcolor{gray}{$\pm 11.9$} &  {$73.0$} \textcolor{gray}{$\pm 7.9$} & {$92.4$} \textcolor{gray}{$\pm 3.7$} & {$89.2$} \textcolor{gray}{$\pm 3.8$}\\
+ \texttt{Enc}-\ref*{def:gromov-monge-gap} & {$65.1$} \textcolor{gray}{$\pm 5.5$}& {$76.1$} \textcolor{gray}{$\pm 7.7$} & {$90.8$} \textcolor{gray}{$\pm 9.2$} & {$92.0$} \textcolor{gray}{$\pm 5.3$}\\
+ \texttt{Dec}-\ref*{def:distortion} & $67.4$ \textcolor{gray}{$\pm 7.1$} & $77.9$ \textcolor{gray}{$\pm 4.5$} & {$93.2$} \textcolor{gray}{$\pm 9.7$} & $94.5$ \textcolor{gray}{$\pm 6.9$}\\
+ \texttt{Dec}-\ref*{def:gromov-monge-gap} & \textcolor{Periwinkle}{70.0} \textcolor{gray}{$\pm 5.9$} & \textcolor{Periwinkle}{81.0} \textcolor{gray}{$\pm 3.2$} & \textcolor{Periwinkle}{\underline{93.3}} \textcolor{gray}{$\pm 8.6$} & \textcolor{Periwinkle}{\underline{96.1}} \textcolor{gray}{$\pm 3.8$}\\
\midrule
\multicolumn{5}{c}{$\mathbf{Cos}: c_\+X=c_\+Y=\textrm{cos-sim}(\cdot, \cdot)$} \\
\midrule
+ \texttt{Enc}-\ref*{def:distortion} & $69.2$ \textcolor{gray}{$\pm 9.1$} & $77.2$ \textcolor{gray}{$\pm 7.5$} & $87.7$ \textcolor{gray}{$\pm 7.7$} & $90.5$ \textcolor{gray}{$\pm 5.9$}\\
+ \texttt{Enc}-\ref*{def:gromov-monge-gap} & {$70.9$} \textcolor{gray}{$\pm 9.5$} & {$79.6$} \textcolor{gray}{$\pm 6.6$} & {$92.5$} \textcolor{gray}{$\pm 5.9$} & {$93.5$} \textcolor{gray}{$\pm 6.9$}\\
+ \texttt{Dec}-\ref*{def:distortion} & \underline{76.8} \textcolor{gray}{$\pm 4.1$} & \underline{81.3} \textcolor{gray}{$\pm 4.7$} & $87.5$ \textcolor{gray}{$\pm 3.3$} & $91.9$ \textcolor{gray}{$\pm 9.4$} \\
+ \texttt{Dec}-\ref*{def:gromov-monge-gap} & \textcolor{Periwinkle}{\textbf{82.1}} \textcolor{gray}{$\pm 4.5$} & \textcolor{Periwinkle}{\textbf{83.7}} \textcolor{gray}{$\pm 8.8$} & \textcolor{Periwinkle}{\textbf{95.7}} \textcolor{gray}{$\pm 5.8$} & \textcolor{Periwinkle}{\textbf{96.9}} \textcolor{gray}{$\pm 4.9$}\\
\bottomrule
\end{tabular}
\vspace{-3mm}
\end{table}

\xhdr{Which costs $c_\+X ,c_\+Y$ should we choose?} The first question that naturally arises when using a geometry-preserving regularizer is: Which geometric features should be preserved? Previous works~\citep{nakagawa2023gromovwasserstein,lee2022regularized,huh2023isometric} focused on preserving scaled distances between points, with the scale being learnable. We follow and extend this approach by also investigating plain distances and angles. This leads to three choices for $c_\+X,c_\+Y$ ($\mathbf{L2^2}, \mathbf{ScL2^2}, \mathbf{Cos}$), as introduced in \S\ref{sec:distortion}, following the hierarchy of geometry-preserving mappings proposed in \citet{lee2022regularized}. We benchmark these on Shapes3D across various settings including \ref*{def:distortion}, \ref*{def:gromov-monge-gap}, and \texttt{Jac}. Table~\ref{tab:reg-benchmark} shows that angle preservation ($\mathbf{Cos}$), previously unconsidered for disentangled representation learning, consistently outperforms (scaled) distance preservation. This result is intuitive, as preserving angles imposes a weaker constraint, allowing for greater latent space expressiveness. In practice, preserving scaled distances seems to overly restrict the expressiveness of the latent space.

\xhdr{Should we regularize the encoder $e_\phi$, or the decoder $d_\theta$?} The next question we aim to answer is whether the decoder or encoder should be regularized. Therefore, we follow the previous setup on Shapes3D and benchmark all geometry-preserving regularizers on $e_\phi$ and $d_\theta$ as reported in Table~\ref{tab:reg-benchmark}. We find that regularizing the decoder is beneficial over regularizing the encoder. We hypothesize this is due to the regularization of $d_\theta$ offering a stronger signal as its gradients impact both $\phi$ and $\theta$, as in this case, the reference $r=e_\phi\sharp p_\textrm{data}$ is the distribution of encoded images. Our findings align with prior works \citep{lee2022regularized,nakagawa2023gromovwasserstein}, which focus on regularizing $d_\theta$ yet do not offer this type of analysis. Additionally, we find that the~\ref*{def:gromov-monge-gap} consistently outperforms the~\ref*{def:distortion} over all costs. Overall, the \ref*{def:gromov-monge-gap} on $d_\theta$ with $\mathbf{Cos}$ achieves best DCI-D results over all baselines. Consequently, moving forward we regularize the decoder for angle preservation.

\begin{table}[t]
\vspace{-7mm}
\caption{Effect of \ref*{def:gromov-monge-gap} and \ref*{def:distortion} with $\mathbf{Cos}$ as costs on disentanglement, as measured by \textbf{DCI-D}, over three datasets. We highlight the \textbf{best}, and \underline{second best} result for each dataset and method.}
\label{table:decoder-benchmark}
\vspace{-1mm}
\centering
\begin{tabular}{lcccc}
\toprule
With $\mathbf{Cos}$ costs & {$\beta$-VAE} & {$\beta$-TCVAE} & $\beta$-VAE + HFS& $\beta$-TCVAE + HFS\\
\midrule
\multicolumn{5}{c}{\textbf{DSprites}~\citep{higgins2017betavae}} \\
\midrule
Base & $26.2$ \textcolor{gray}{$\pm 18.5$} & ${32.3}$ \textcolor{gray}{$\pm 19.3$} & $33.6$ \textcolor{gray}{$\pm 17.9$} & $48.7$ \textcolor{gray}{$\pm 10.2$} \\
+ \texttt{Dec}-\ref*{def:distortion} & $\underline{28.6}$ \textcolor{gray}{$\pm 19.3$} & \underline{$32.4$} \textcolor{gray}{$\pm 8.5$} & \underline{$39.3$} \textcolor{gray}{$\pm 18.1$} & \underline{$49.0$} \textcolor{gray}{$\pm 11.2$} \\
+ \texttt{Dec}-\ref*{def:gromov-monge-gap} & \textbf{39.5} \textcolor{gray}{$\pm 15.2$} & \textbf{42.2} \textcolor{gray}{$\pm 3.6$} & \textbf{46.7} \textcolor{gray}{$\pm 2.0$} & \textbf{50.1} \textcolor{gray}{$\pm 8.5$}  \\
\midrule
\multicolumn{5}{c}{\textbf{SmallNORB}~\citep{lecun2004smallnorb}} \\
\midrule
Base & $26.8$ \textcolor{gray}{$\pm 0.2$} & $29.8$ \textcolor{gray}{$\pm 0.4$} & $26.8$ \textcolor{gray}{$\pm 0.2$} & $29.8$ \textcolor{gray}{$\pm 0.4$} \\
+ \texttt{Dec}-\ref*{def:distortion} & \underline{$28.2$} \textcolor{gray}{$\pm 0.3$} & \textbf{29.9} \textcolor{gray}{$\pm 0.4$} & \underline{$28.2$} \textcolor{gray}{$\pm 0.3$} & \textbf{29.9} \textcolor{gray}{$\pm 0.4$} \\
+ \texttt{Dec}-\ref*{def:gromov-monge-gap} & \textbf{28.3} \textcolor{gray}{$\pm 0.6$} & \textbf{29.9} \textcolor{gray}{$\pm 0.5$} & \textbf{28.3} \textcolor{gray}{$\pm 0.6$} & \textbf{29.9} \textcolor{gray}{$\pm 0.5$} \\
\midrule
\multicolumn{5}{c}{\textbf{Cars3D}~\citep{reed2015cars3d}} \\
\midrule
Base & \underline{$29.6$} \textcolor{gray}{$\pm 5.7$} & $32.3$ \textcolor{gray}{$\pm 4.6$} & \underline{$29.6$} \textcolor{gray}{$\pm 5.7$} & $32.3$ \textcolor{gray}{$\pm 4.6$} \\
+ \texttt{Dec}-\ref*{def:distortion} & $26.8$ \textcolor{gray}{$\pm 3.6$} & \underline{$33.7$} \textcolor{gray}{$\pm 4.2$} & $26.8$ \textcolor{gray}{$\pm 3.6$} & \underline{$33.7$} \textcolor{gray}{$\pm 4.2$} \\
+ \texttt{Dec}-\ref*{def:gromov-monge-gap} & \textbf{30.1} \textcolor{gray}{$\pm 5.6$} & \textbf{36.4} \textcolor{gray}{$\pm 5.7$} & \textbf{30.1} \textcolor{gray}{$\pm 5.6$} & \textbf{36.4} \textcolor{gray}{$\pm 5.7$} \\
\bottomrule
\end{tabular}%
\vspace{-1mm}

\end{table}

\xhdr{GMG Consistently Enhances Disentanglement.} To further validate our findings, we benchmark the \ref*{def:gromov-monge-gap} for decoder regularization with angle preservation ($\mathbf{Cos}$) against its distortion counterpart across three more datasets. We report full results in Table~\ref{table:decoder-benchmark}. Again we observe that the \ref*{def:gromov-monge-gap} outperforms or performs equally well to its distortion equivalent. Note that for SmallNORB and Cars3D, we found no benefits with respect to DCI-D in adding HFS and obtained the best results without it. We emphasize that using the \ref*{def:gromov-monge-gap} with $\mathbf{Cos}$ regularization significantly improves results for all datasets. This establishes the \ref*{def:gromov-monge-gap} as an effective tool for enhanced disentanglement.

\begin{wrapfigure}{o}{0.55\textwidth}
    \centering
    \vspace{-5mm}
    \includegraphics[width=1\linewidth]{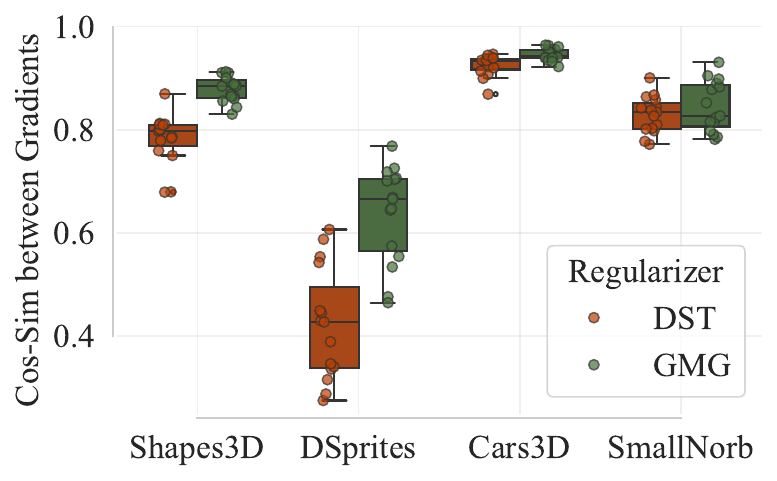}
    \vspace{-6mm}
    \caption{Stability analysis for the \ref*{def:distortion}  and the \ref*{def:gromov-monge-gap} applied to the decoder $d_\theta$ with cost $\mathbf{Cos}$. For each regularizer, we assess stability by measuring the alignment of its gradients across 5 batches. We compute the cosine similarity between all pairs of gradients. Values closer to 1 indicates better alignment, hence better stability.}
    \label{fig:stability-exp}
\vspace{-1mm}
\end{wrapfigure}
\textbf{Stability of the GMG.} Furthermore, we investigate the stability of the \ref*{def:gromov-monge-gap} compared to the \ref*{def:distortion}. Computing the \ref*{def:gromov-monge-gap}  using Algorithm~\ref{algo:entropic-gmg} involves solving an optimization problem through a GW solver. This experiment aims to demonstrate that our proposed method for solving the GW problem leads to a stable loss in the \ref*{def:gromov-monge-gap} . We assess the stability by measuring the alignment of \ref*{def:gromov-monge-gap}  gradients $\nabla_\theta\mathrm{GMG}_{r_n} (T_\theta)$ across 5 randomly sampled batches from each dataset $r_n \sim \mathcal{D}$, for a fixed neural map $T_\theta$. We repeat this procedure for each of the four datasets $\mathcal{D}$ considered and apply the same methodology to the \ref*{def:distortion}. Figure~\ref{fig:stability-exp} presents results. We observe that the \ref*{def:gromov-monge-gap}'s gradients exhibit significantly higher alignment compared to those of the \ref*{def:distortion}, demonstrating greater stability of our proposed regularizer. This suggests that the \ref*{def:gromov-monge-gap}, which accounts for the the minimal distortion, effectively mitigates the inherent variability inherent of the distortion, leading to more stable gradient computations.

\section{Conclusion}

In this work, we introduce an OT perspective on unsupervised disentangled representation learning to incorporate general latent geometrical constraints. We derive the GMG, a provably weakly convex regularizer that measures whether a map $T$ transports a fixed reference distribution with minimal distortion of some predefined geometric features. By formulating disentangled representation learning as a transport problem, we integrate the GMG into standard training objectives, allowing for incorporating and studying various geometric constraints on the learned representation spaces. We show significant performance benefits of our approach on four standard disentanglement benchmarks.

\section{Reproduciblity}
In this work, we introduce the \ref*{def:gromov-monge-gap}, computed as detailed in Algorithm~\ref{algo:entropic-gmg}. To facilitate reproducibility, we provide the implementation code for computing the \ref*{def:gromov-monge-gap} on source and target batches in Appendix~\ref{app:code}. Comprehensive proofs, including all underlying assumptions, are presented in Appendix~\ref{app:proofs}. For our experiments on disentanglement benchmarks, we adhere to standard practices, employing streamlined preprocessing across all datasets. Detailed descriptions of these procedures are available in Appendix~\ref{app:experimental_details}. All experiments described in this paper can be conducted using a single RTX 2080TI GPU, ensuring accessibility and replicability of our results.

\section*{Acknowledgements}
Co-funded by the European Union (ERC, DeepCell - 101054957). Views and opinions expressed are, however, those of the author(s) only and do not necessarily reflect those of the European Union or the European Research Council. Neither the European Union nor the granting authority can be
held responsible for them. Luca Eyring and Karsten Roth thank the European Laboratory for Learning and Intelligent Systems (ELLIS) PhD program for support. Karsten Roth also thanks the International Max Planck Research School for Intelligent Systems (IMPRS-IS) for support.
Zeynep Akata was supported by BMBF FKZ: 01IS18039A, by the ERC (853489 - DEXIM), by EXC number 2064/1 – project number 390727645. Fabian J. Theis consults for Immunai Inc., Singularity Bio B.V., CytoReason Ltd, Cellarity, and has ownership interest in Dermagnostix GmbH and Cellarity

{
\bibliographystyle{unsrtnat}
\bibliography{main.bib}

\begin{thebibliography}{74}
\providecommand{\natexlab}[1]{#1}
\providecommand{\url}[1]{\texttt{#1}}
\expandafter\ifx\csname urlstyle\endcsname\relax
  \providecommand{\doi}[1]{doi: #1}\else
  \providecommand{\doi}{doi: \begingroup \urlstyle{rm}\Url}\fi

\bibitem[Bengio et~al.(2014)Bengio, Courville, and Vincent]{bengio2014representation}
Yoshua Bengio, Aaron Courville, and Pascal Vincent.
\newblock Representation learning: A review and new perspectives, 2014.

\bibitem[Higgins et~al.(2018)Higgins, Amos, Pfau, Racaniere, Matthey, Rezende, and Lerchner]{higgins2018definition}
Irina Higgins, David Amos, David Pfau, Sebastien Racaniere, Loic Matthey, Danilo Rezende, and Alexander Lerchner.
\newblock Towards a definition of disentangled representations, 2018.

\bibitem[Locatello et~al.(2019{\natexlab{a}})Locatello, Bauer, Lucic, Raetsch, Gelly, Sch{\"o}lkopf, and Bachem]{pmlr-v97-locatello19a}
Francesco Locatello, Stefan Bauer, Mario Lucic, Gunnar Raetsch, Sylvain Gelly, Bernhard Sch{\"o}lkopf, and Olivier Bachem.
\newblock Challenging common assumptions in the unsupervised learning of disentangled representations.
\newblock In Kamalika Chaudhuri and Ruslan Salakhutdinov, editors, \emph{Proceedings of the 36th International Conference on Machine Learning}, volume~97 of \emph{Proceedings of Machine Learning Research}, pages 4114--4124. PMLR, 09--15 Jun 2019{\natexlab{a}}.
\newblock URL \url{https://proceedings.mlr.press/v97/locatello19a.html}.

\bibitem[Locatello et~al.(2020)Locatello, Poole, Raetsch, Sch{\"o}lkopf, Bachem, and Tschannen]{pmlr-v119-locatello20a}
Francesco Locatello, Ben Poole, Gunnar Raetsch, Bernhard Sch{\"o}lkopf, Olivier Bachem, and Michael Tschannen.
\newblock Weakly-supervised disentanglement without compromises.
\newblock In Hal~Daumé III and Aarti Singh, editors, \emph{Proceedings of the 37th International Conference on Machine Learning}, volume 119 of \emph{Proceedings of Machine Learning Research}, pages 6348--6359. PMLR, 13--18 Jul 2020.
\newblock URL \url{https://proceedings.mlr.press/v119/locatello20a.html}.

\bibitem[Roth et~al.(2023)Roth, Ibrahim, Akata, Vincent, and Bouchacourt]{roth2023disentanglement}
Karsten Roth, Mark Ibrahim, Zeynep Akata, Pascal Vincent, and Diane Bouchacourt.
\newblock Disentanglement of correlated factors via hausdorff factorized support.
\newblock In \emph{The Eleventh International Conference on Learning Representations}, 2023.
\newblock URL \url{https://openreview.net/forum?id=OKcJhpQiGiX}.

\bibitem[Hsu et~al.(2023)Hsu, Dorrell, Whittington, Wu, and Finn]{hsu2023disentanglement}
Kyle Hsu, Will Dorrell, James C.~R. Whittington, Jiajun Wu, and Chelsea Finn.
\newblock Disentanglement via latent quantization.
\newblock In \emph{Thirty-seventh Conference on Neural Information Processing Systems}, 2023.
\newblock URL \url{https://openreview.net/forum?id=LLETO26Ga2}.

\bibitem[Barin-Pacela et~al.(2024)Barin-Pacela, Ahuja, Lacoste-Julien, and Vincent]{barinpacela2024identifiability}
Vitória Barin-Pacela, Kartik Ahuja, Simon Lacoste-Julien, and Pascal Vincent.
\newblock On the identifiability of quantized factors, 2024.

\bibitem[Locatello et~al.(2019{\natexlab{b}})Locatello, Abbati, Rainforth, Bauer, Sch\"{o}lkopf, and Bachem]{locatello2019fair}
F.~Locatello, G.~Abbati, T.~Rainforth, S.~Bauer, B.~Sch\"{o}lkopf, and O.~Bachem.
\newblock On the fairness of disentangled representations.
\newblock In \emph{Advances in Neural Information Processing Systems 32 (NeurIPS 2019)}, pages 14584--14597. Curran Associates, Inc., December 2019{\natexlab{b}}.
\newblock URL \url{https://papers.nips.cc/paper/9603-on-the-fairness-of-disentangled-representations}.

\bibitem[Tr{\"a}uble et~al.(2021)Tr{\"a}uble, Creager, Kilbertus, Locatello, Dittadi, Goyal, Sch{\"o}lkopf, and Bauer]{trauble2021corr}
Frederik Tr{\"a}uble, Elliot Creager, Niki Kilbertus, Francesco Locatello, Andrea Dittadi, Anirudh Goyal, Bernhard Sch{\"o}lkopf, and Stefan Bauer.
\newblock On disentangled representations learned from correlated data.
\newblock In Marina Meila and Tong Zhang, editors, \emph{Proceedings of the 38th International Conference on Machine Learning}, volume 139 of \emph{Proceedings of Machine Learning Research}, pages 10401--10412. PMLR, 18--24 Jul 2021.
\newblock URL \url{https://proceedings.mlr.press/v139/trauble21a.html}.

\bibitem[Rolinek et~al.(2019)Rolinek, Zietlow, and Martius]{rolinek2019vae_pca}
Michal Rolinek, Dominik Zietlow, and Georg Martius.
\newblock Variational autoencoders pursue pca directions (by accident).
\newblock In \emph{Proceedings of the IEEE/CVF Conference on Computer Vision and Pattern Recognition (CVPR)}, June 2019.

\bibitem[Zietlow et~al.(2021)Zietlow, Rolinek, and Martius]{zietlow2021demystifying}
Dominik Zietlow, Michal Rolinek, and Georg Martius.
\newblock Demystifying inductive biases for (beta-)vae based architectures.
\newblock In Marina Meila and Tong Zhang, editors, \emph{Proceedings of the 38th International Conference on Machine Learning}, volume 139 of \emph{Proceedings of Machine Learning Research}, pages 12945--12954. PMLR, 18--24 Jul 2021.
\newblock URL \url{https://proceedings.mlr.press/v139/zietlow21a.html}.

\bibitem[Kingma and Welling(2014)]{Kingma2014}
Diederik~P. Kingma and Max Welling.
\newblock {Auto-Encoding Variational Bayes}.
\newblock In \emph{2nd International Conference on Learning Representations, {ICLR} 2014, Banff, AB, Canada, April 14-16, 2014, Conference Track Proceedings}, 2014.

\bibitem[Higgins et~al.(2017)Higgins, Matthey, Pal, Burgess, Glorot, Botvinick, Mohamed, and Lerchner]{higgins2017betavae}
Irina Higgins, Loic Matthey, Arka Pal, Christopher Burgess, Xavier Glorot, Matthew Botvinick, Shakir Mohamed, and Alexander Lerchner.
\newblock beta-{VAE}: Learning basic visual concepts with a constrained variational framework.
\newblock In \emph{International Conference on Learning Representations}, 2017.
\newblock URL \url{https://openreview.net/forum?id=Sy2fzU9gl}.

\bibitem[Kim and Mnih(2018)]{kim2018factorvae}
Hyunjik Kim and Andriy Mnih.
\newblock Disentangling by factorising.
\newblock In Jennifer Dy and Andreas Krause, editors, \emph{Proceedings of the 35th International Conference on Machine Learning}, volume~80 of \emph{Proceedings of Machine Learning Research}, pages 2649--2658. PMLR, 10--15 Jul 2018.
\newblock URL \url{https://proceedings.mlr.press/v80/kim18b.html}.

\bibitem[Kumar et~al.(2018)Kumar, Sattigeri, and Balakrishnan]{kumar2018dipvae}
Abhishek Kumar, Prasanna Sattigeri, and Avinash Balakrishnan.
\newblock {VARIATIONAL} {INFERENCE} {OF} {DISENTANGLED} {LATENT} {CONCEPTS} {FROM} {UNLABELED} {OBSERVATIONS}.
\newblock In \emph{International Conference on Learning Representations}, 2018.
\newblock URL \url{https://openreview.net/forum?id=H1kG7GZAW}.

\bibitem[Burgess et~al.(2018)Burgess, Higgins, Pal, Matthey, Watters, Desjardins, and Lerchner]{burgess2018annealedvae}
Christopher~P. Burgess, Irina Higgins, Arka Pal, Loïc Matthey, Nick Watters, Guillaume Desjardins, and Alexander Lerchner.
\newblock Understanding disentangling in beta-vae.
\newblock \emph{CoRR}, abs/1804.03599, 2018.
\newblock URL \url{http://arxiv.org/abs/1804.03599}.

\bibitem[Chen et~al.(2018)Chen, Li, Grosse, and Duvenaud]{chen2018betatcvae}
Ricky T.~Q. Chen, Xuechen Li, Roger~B Grosse, and David~K Duvenaud.
\newblock Isolating sources of disentanglement in variational autoencoders.
\newblock In S.~Bengio, H.~Wallach, H.~Larochelle, K.~Grauman, N.~Cesa-Bianchi, and R.~Garnett, editors, \emph{Advances in Neural Information Processing Systems}, volume~31. Curran Associates, Inc., 2018.
\newblock URL \url{https://proceedings.neurips.cc/paper/2018/file/1ee3dfcd8a0645a25a35977997223d22-Paper.pdf}.

\bibitem[Gropp et~al.(2020)Gropp, Atzmon, and Lipman]{gropp2020isometric}
Amos Gropp, Matan Atzmon, and Yaron Lipman.
\newblock Isometric autoencoders, 2020.

\bibitem[Chen et~al.(2020{\natexlab{a}})Chen, Klushyn, Ferroni, Bayer, and van~der Smagt]{chen2020learning}
Nutan Chen, Alexej Klushyn, Francesco Ferroni, Justin Bayer, and Patrick van~der Smagt.
\newblock Learning flat latent manifolds with vaes, 2020{\natexlab{a}}.

\bibitem[Lee et~al.(2022)Lee, Yoon, Son, and Park]{lee2022regularized}
Yonghyeon Lee, Sangwoong Yoon, MinJun Son, and Frank~C. Park.
\newblock Regularized autoencoders for isometric representation learning.
\newblock In \emph{International Conference on Learning Representations}, 2022.
\newblock URL \url{https://openreview.net/forum?id=mQxt8l7JL04}.

\bibitem[Horan et~al.(2021)Horan, Richardson, and Weiss]{horan2021when}
Daniella Horan, Eitan Richardson, and Yair Weiss.
\newblock When is unsupervised disentanglement possible?
\newblock In A.~Beygelzimer, Y.~Dauphin, P.~Liang, and J.~Wortman Vaughan, editors, \emph{Advances in Neural Information Processing Systems}, 2021.
\newblock URL \url{https://openreview.net/forum?id=XqEF9riB93S}.

\bibitem[Nakagawa et~al.(2023)Nakagawa, Togo, Ogawa, and Haseyama]{nakagawa2023gromovwasserstein}
Nao Nakagawa, Ren Togo, Takahiro Ogawa, and Miki Haseyama.
\newblock Gromov-wasserstein autoencoders, 2023.

\bibitem[Huh et~al.(2023)Huh, changwook jeong, Choe, Kim, and Kim]{huh2023isometric}
In~Huh, changwook jeong, Jae~Myung Choe, Young-Gu Kim, and Dae~Sin Kim.
\newblock Isometric quotient variational auto-encoders for structure-preserving representation learning.
\newblock In \emph{Thirty-seventh Conference on Neural Information Processing Systems}, 2023.
\newblock URL \url{https://openreview.net/forum?id=EdgPb3ngR4}.

\bibitem[Hahm et~al.(2024)Hahm, Lee, Kim, and Lee]{hahm2024isometricrepresentationlearningdisentangled}
Jaehoon Hahm, Junho Lee, Sunghyun Kim, and Joonseok Lee.
\newblock Isometric representation learning for disentangled latent space of diffusion models, 2024.
\newblock URL \url{https://arxiv.org/abs/2407.11451}.

\bibitem[Santambrogio(2015)]{santambrogio2015optimal}
Filippo Santambrogio.
\newblock {Optimal Transport for Applied Mathematicians}.
\newblock \emph{Birkh{\"a}user, NY}, 55\penalty0 (58-63):\penalty0 94, 2015.

\bibitem[Peyr{\'e} and Cuturi(2019)]{Peyre2019computational}
Gabriel Peyr{\'e} and Marco Cuturi.
\newblock {Computational Optimal Transport}.
\newblock \emph{Foundations and Trends in Machine Learning}, 11\penalty0 (5-6), 2019.
\newblock ISSN 1935-8245.

\bibitem[Sturm(2023)]{sturm2020space}
Karl-Theodor Sturm.
\newblock \emph{The space of spaces: curvature bounds and gradient flows on the space of metric measure spaces}, volume 290.
\newblock American Mathematical Society, 2023.

\bibitem[M{\'e}moli(2011)]{memoli2011gromov}
Facundo M{\'e}moli.
\newblock Gromov--wasserstein distances and the metric approach to object matching.
\newblock \emph{Foundations of computational mathematics}, 11:\penalty0 417--487, 2011.

\bibitem[Uscidda and Cuturi(2023)]{uscidda2023monge}
Théo Uscidda and Marco Cuturi.
\newblock The monge gap: A regularizer to learn all transport maps, 2023.

\bibitem[Moor et~al.(2021)Moor, Horn, Rieck, and Borgwardt]{moor2021topologicalautoencoders}
Michael Moor, Max Horn, Bastian Rieck, and Karsten Borgwardt.
\newblock Topological autoencoders, 2021.
\newblock URL \url{https://arxiv.org/abs/1906.00722}.

\bibitem[Balabin et~al.(2024)Balabin, Voronkova, Trofimov, Burnaev, and Barannikov]{balabin2024disentanglementlearningtopology}
Nikita Balabin, Daria Voronkova, Ilya Trofimov, Evgeny Burnaev, and Serguei Barannikov.
\newblock Disentanglement learning via topology, 2024.
\newblock URL \url{https://arxiv.org/abs/2308.12696}.

\bibitem[Koopmans and Beckmann(1957)]{koopmans1957assignment}
Tjalling~C Koopmans and Martin Beckmann.
\newblock Assignment problems and the location of economic activities.
\newblock \emph{Econometrica: journal of the Econometric Society}, pages 53--76, 1957.

\bibitem[Sotiropoulou and Alvarez-Melis(2024)]{sotiropoulou2024stronglyisomorphicneuraloptimal}
Athina Sotiropoulou and David Alvarez-Melis.
\newblock Strongly isomorphic neural optimal transport across incomparable spaces, 2024.
\newblock URL \url{https://arxiv.org/abs/2407.14957}.

\bibitem[Monge(1781)]{Monge1781}
Gaspard Monge.
\newblock M{\'e}moire sur la th{\'e}orie des d{\'e}blais et des remblais.
\newblock \emph{Histoire de l'Acad{\'e}mie Royale des Sciences}, pages 666--704, 1781.

\bibitem[Mémoli and Needham(2022)]{mémoli2022comparison}
Facundo Mémoli and Tom Needham.
\newblock Comparison results for gromov-wasserstein and gromov-monge distances, 2022.

\bibitem[Cuturi(2013)]{cuturi2013sinkhorn}
Marco Cuturi.
\newblock {Sinkhorn Distances: Lightspeed Computation of Optimal Transport}.
\newblock In \emph{Advances in Neural Information Processing Systems (NeurIPS)}, volume~26, 2013.

\bibitem[Peyr{\'e} et~al.(2016)Peyr{\'e}, Cuturi, and Solomon]{peyre2016gromov}
Gabriel Peyr{\'e}, Marco Cuturi, and Justin Solomon.
\newblock Gromov-wasserstein averaging of kernel and distance matrices.
\newblock In \emph{International Conference on Machine Learning}, pages 2664--2672, 2016.

\bibitem[Scetbon et~al.(2022)Scetbon, Peyr{\'e}, and Cuturi]{scetbon2022linear}
Meyer Scetbon, Gabriel Peyr{\'e}, and Marco Cuturi.
\newblock Linear-time gromov wasserstein distances using low rank couplings and costs.
\newblock In \emph{International Conference on Machine Learning}, pages 19347--19365. PMLR, 2022.

\bibitem[Dumont et~al.(2022)Dumont, Lacombe, and Vialard]{dumont2022existence}
Th{\'e}o Dumont, Th{\'e}o Lacombe, and Fran{\c{c}}ois-Xavier Vialard.
\newblock On the existence of monge maps for the gromov-wasserstein problem.
\newblock 2022.

\bibitem[Cuturi et~al.(2022)Cuturi, Meng-Papaxanthos, Tian, Bunne, Davis, and Teboul]{cuturi2022optimal}
Marco Cuturi, Laetitia Meng-Papaxanthos, Yingtao Tian, Charlotte Bunne, Geoff Davis, and Olivier Teboul.
\newblock {Optimal Transport Tools (OTT): A JAX Toolbox for all things Wasserstein}.
\newblock \emph{arXiv Preprint arXiv:2201.12324}, 2022.

\bibitem[Davis et~al.(2018)Davis, Drusvyatskiy, MacPhee, and Paquette]{davis2018subgradient}
Damek Davis, Dmitriy Drusvyatskiy, Kellie~J. MacPhee, and Courtney Paquette.
\newblock Subgradient methods for sharp weakly convex functions, 2018.

\bibitem[Weed and Bach(2017)]{weed2017sharp}
Jonathan Weed and Francis Bach.
\newblock Sharp asymptotic and finite-sample rates of convergence of empirical measures in wasserstein distance, 2017.
\newblock URL \url{https://arxiv.org/abs/1707.00087}.

\bibitem[Zhang et~al.(2023)Zhang, Goldfeld, Mroueh, and Sriperumbudur]{zhang2023gromovwasserstein}
Zhengxin Zhang, Ziv Goldfeld, Youssef Mroueh, and Bharath~K. Sriperumbudur.
\newblock Gromov-wasserstein distances: Entropic regularization, duality, and sample complexity, 2023.

\bibitem[Feydy et~al.(2019)Feydy, S{\'e}journ{\'e}, Vialard, Amari, Trouv{\'e}, and Peyr{\'e}]{feydy2018interpolating}
Jean Feydy, Thibault S{\'e}journ{\'e}, Fran{\c{c}}ois-Xavier Vialard, Shun-Ichi Amari, Alain Trouv{\'e}, and Gabriel Peyr{\'e}.
\newblock {Interpolating between Optimal Transport and MMD using Sinkhorn Divergences}.
\newblock In \emph{International Conference on Artificial Intelligence and Statistics (AISTATS)}, volume~22, 2019.

\bibitem[Eastwood and Williams(2018)]{eastwood2018a}
Cian Eastwood and Christopher K.~I. Williams.
\newblock A framework for the quantitative evaluation of disentangled representations.
\newblock In \emph{International Conference on Learning Representations}, 2018.
\newblock URL \url{https://openreview.net/forum?id=By-7dz-AZ}.

\bibitem[Dittadi et~al.(2021)Dittadi, Tr{\"a}uble, Locatello, Wuthrich, Agrawal, Winther, Bauer, and Sch{\"o}lkopf]{dittadi2021on}
Andrea Dittadi, Frederik Tr{\"a}uble, Francesco Locatello, Manuel Wuthrich, Vaibhav Agrawal, Ole Winther, Stefan Bauer, and Bernhard Sch{\"o}lkopf.
\newblock On the transfer of disentangled representations in realistic settings.
\newblock In \emph{International Conference on Learning Representations}, 2021.
\newblock URL \url{https://openreview.net/forum?id=8VXvj1QNRl1}.

\bibitem[LeCun et~al.(2004)LeCun, Huang, and Bottou]{lecun2004smallnorb}
Yann LeCun, {Fu Jie} Huang, and L{\'e}on Bottou.
\newblock Learning methods for generic object recognition with invariance to pose and lighting.
\newblock \emph{Proceedings of the IEEE Computer Society Conference on Computer Vision and Pattern Recognition}, 2:\penalty0 II97--II104, 2004.
\newblock ISSN 1063-6919.
\newblock Proceedings of the 2004 IEEE Computer Society Conference on Computer Vision and Pattern Recognition, CVPR 2004 ; Conference date: 27-06-2004 Through 02-07-2004.

\bibitem[Reed et~al.(2015)Reed, Zhang, Zhang, and Lee]{reed2015cars3d}
Scott~E Reed, Yi~Zhang, Yuting Zhang, and Honglak Lee.
\newblock Deep visual analogy-making.
\newblock In C.~Cortes, N.~Lawrence, D.~Lee, M.~Sugiyama, and R.~Garnett, editors, \emph{Advances in Neural Information Processing Systems}, volume~28. Curran Associates, Inc., 2015.
\newblock URL \url{https://proceedings.neurips.cc/paper/2015/file/e07413354875be01a996dc560274708e-Paper.pdf}.

\bibitem[Kantorovich(1942)]{kantorovich1942transfer}
L~Kantorovich.
\newblock On the transfer of masses (in russian).
\newblock In \emph{Doklady Akademii Nauk}, volume~37, page 227, 1942.

\bibitem[Séjourné et~al.(2023)Séjourné, Vialard, and Peyré]{sejourne2023the-unbalanced-gromov}
Thibault Séjourné, François-Xavier Vialard, and Gabriel Peyré.
\newblock The unbalanced gromov wasserstein distance: Conic formulation and relaxation, 2023.

\bibitem[Birkhoff(1946)]{birkhoff}
Garrett Birkhoff.
\newblock Tres observaciones sobre el algebra lineal.
\newblock \emph{Universidad Nacional de Tucum{\'a}n Revista Series A}, 5:\penalty0 147--151, 1946.

\bibitem[Bertsimas and Tsitsiklis(1997)]{bertsimas-LPbook}
D.~Bertsimas and J.N. Tsitsiklis.
\newblock \emph{Introduction to linear optimization}.
\newblock Athena Scientific, 1997.

\bibitem[Le~Gall()]{legall2006integration}
Jean-François Le~Gall.
\newblock \emph{Intégration, Probabilités et Processus Aléatoires}.

\bibitem[Rioux et~al.(2023)Rioux, Goldfeld, and Kato]{rioux2023entropic}
Gabriel Rioux, Ziv Goldfeld, and Kengo Kato.
\newblock Entropic gromov-wasserstein distances: Stability, algorithms, and distributional limits, 2023.

\bibitem[Manole and Niles-Weed(2024)]{Manole_2024}
Tudor Manole and Jonathan Niles-Weed.
\newblock Sharp convergence rates for empirical optimal transport with smooth costs.
\newblock \emph{The Annals of Applied Probability}, 34\penalty0 (1B), February 2024.
\newblock ISSN 1050-5164.
\newblock \doi{10.1214/23-aap1986}.
\newblock URL \url{http://dx.doi.org/10.1214/23-AAP1986}.

\bibitem[Boyd and Vandenberghe(2004)]{boyd2004convex}
Stephen Boyd and Lieven Vandenberghe.
\newblock \emph{Convex Optimization}.
\newblock {Cambridge University Press}, March 2004.
\newblock ISBN 0521833787.
\newblock URL \url{http://www.amazon.com/exec/obidos/redirect?tag=citeulike-20\&path=ASIN/0521833787}.

\bibitem[Petersen and Pedersen(2008)]{Petersen2008}
K.~B. Petersen and M.~S. Pedersen.
\newblock The matrix cookbook, October 2008.
\newblock URL \url{http://www2.imm.dtu.dk/pubdb/p.php?3274}.
\newblock Version 20081110.

\bibitem[Babuschkin et~al.(2020)Babuschkin, Baumli, Bell, Bhupatiraju, Bruce, Buchlovsky, Budden, Cai, Clark, Danihelka, Fantacci, Godwin, Jones, Hemsley, Hennigan, Hessel, Hou, Kapturowski, Keck, Kemaev, King, Kunesch, Martens, Merzic, Mikulik, Norman, Quan, Papamakarios, Ring, Ruiz, Sanchez, Schneider, Sezener, Spencer, Srinivasan, Wang, Stokowiec, and Viola]{deepmind2020jax}
Igor Babuschkin, Kate Baumli, Alison Bell, Surya Bhupatiraju, Jake Bruce, Peter Buchlovsky, David Budden, Trevor Cai, Aidan Clark, Ivo Danihelka, Claudio Fantacci, Jonathan Godwin, Chris Jones, Ross Hemsley, Tom Hennigan, Matteo Hessel, Shaobo Hou, Steven Kapturowski, Thomas Keck, Iurii Kemaev, Michael King, Markus Kunesch, Lena Martens, Hamza Merzic, Vladimir Mikulik, Tamara Norman, John Quan, George Papamakarios, Roman Ring, Francisco Ruiz, Alvaro Sanchez, Rosalia Schneider, Eren Sezener, Stephen Spencer, Srivatsan Srinivasan, Luyu Wang, Wojciech Stokowiec, and Fabio Viola.
\newblock The {D}eep{M}ind {JAX} {E}cosystem, 2020.
\newblock URL \url{http://github.com/deepmind}.

\bibitem[Kingma and Ba(2014)]{kingma2014adam}
Diederik~P Kingma and Jimmy Ba.
\newblock {Adam: A Method for Stochastic Optimization}.
\newblock In \emph{International Conference on Learning Representations (ICLR)}, 2014.

\bibitem[Altschuler et~al.(2018)Altschuler, Weed, and Rigollet]{altschuler2018nearlineartimeapproximationalgorithms}
Jason Altschuler, Jonathan Weed, and Philippe Rigollet.
\newblock Near-linear time approximation algorithms for optimal transport via sinkhorn iteration, 2018.
\newblock URL \url{https://arxiv.org/abs/1705.09634}.

\bibitem[Piran et~al.(2024)Piran, Klein, Thornton, and Cuturi]{piran2024contrastingmultiplerepresentationsmultimarginal}
Zoe Piran, Michal Klein, James Thornton, and Marco Cuturi.
\newblock Contrasting multiple representations with the multi-marginal matching gap, 2024.
\newblock URL \url{https://arxiv.org/abs/2405.19532}.

\bibitem[Pooladian and Niles-Weed(2021)]{pooladian2021entropic}
Aram-Alexandre Pooladian and Jonathan Niles-Weed.
\newblock Entropic estimation of optimal transport maps.
\newblock \emph{arXiv preprint arXiv:2109.12004}, 2021.

\bibitem[Kassraie et~al.(2024)Kassraie, Pooladian, Klein, Thornton, Niles-Weed, and Cuturi]{kassraie2024progressiveentropicoptimaltransport}
Parnian Kassraie, Aram-Alexandre Pooladian, Michal Klein, James Thornton, Jonathan Niles-Weed, and Marco Cuturi.
\newblock Progressive entropic optimal transport solvers, 2024.
\newblock URL \url{https://arxiv.org/abs/2406.05061}.

\bibitem[Klein et~al.(2024)Klein, Uscidda, Theis, and Cuturi]{klein2024entropic}
Dominik Klein, Théo Uscidda, Fabian Theis, and Marco Cuturi.
\newblock Entropic (gromov) wasserstein flow matching with genot, 2024.

\bibitem[Burgess and Kim(2018)]{3dshapes18}
Chris Burgess and Hyunjik Kim.
\newblock 3d shapes dataset.
\newblock https://github.com/deepmind/3dshapes-dataset/, 2018.

\bibitem[Burns et~al.(2021)Burns, Sarna, Krishnan, and Maschinot]{burns2021unsupervised}
Andrea Burns, Aaron Sarna, Dilip Krishnan, and Aaron Maschinot.
\newblock Unsupervised disentanglement without autoencoding: Pitfalls and future directions.
\newblock \emph{arXiv preprint arXiv:2108.06613}, 2021.

\bibitem[von Kügelgen et~al.(2021)von Kügelgen, Sharma, Gresele, Brendel, Schölkopf, Besserve, and Locatello]{vonkugelgen2021self}
Julius von Kügelgen, Yash Sharma, Luigi Gresele, Wieland Brendel, Bernhard Schölkopf, Michel Besserve, and Francesco Locatello.
\newblock Self-supervised learning with data augmentations provably isolates content from style.
\newblock In \emph{Advances in Neural Information Processing Systems}, 2021.

\bibitem[Eastwood et~al.(2023)Eastwood, von K{\"u}gelgen, Ericsson, Bouchacourt, Vincent, Ibrahim, and Sch{\"o}lkopf]{eastwood2023selfsupervised}
Cian Eastwood, Julius von K{\"u}gelgen, Linus Ericsson, Diane Bouchacourt, Pascal Vincent, Mark Ibrahim, and Bernhard Sch{\"o}lkopf.
\newblock Self-supervised disentanglement by leveraging structure in data augmentations.
\newblock In \emph{Causal Representation Learning Workshop at NeurIPS 2023}, 2023.
\newblock URL \url{https://openreview.net/forum?id=JoISqbH8vl}.

\bibitem[Matthes et~al.(2023)Matthes, Han, and Shen]{matthes2023towards}
Stefan Matthes, Zhiwei Han, and Hao Shen.
\newblock Towards a unified framework of contrastive learning for disentangled representations.
\newblock In \emph{Thirty-seventh Conference on Neural Information Processing Systems}, 2023.
\newblock URL \url{https://openreview.net/forum?id=QrB38MAAEP}.

\bibitem[Aitchison and Ganev(2024)]{aitchison2024infonce}
Laurence Aitchison and Stoil~Krasimirov Ganev.
\newblock Info{NCE} is variational inference in a recognition parameterised model.
\newblock \emph{Transactions on Machine Learning Research}, 2024.
\newblock ISSN 2835-8856.
\newblock URL \url{https://openreview.net/forum?id=chbRsWwjax}.

\bibitem[Chen et~al.(2020{\natexlab{b}})Chen, Kornblith, Norouzi, and Hinton]{chen2020simclr}
Ting Chen, Simon Kornblith, Mohammad Norouzi, and Geoffrey Hinton.
\newblock A simple framework for contrastive learning of visual representations.
\newblock In Hal~Daumé III and Aarti Singh, editors, \emph{Proceedings of the 37th International Conference on Machine Learning}, volume 119 of \emph{Proceedings of Machine Learning Research}, pages 1597--1607. PMLR, 13--18 Jul 2020{\natexlab{b}}.
\newblock URL \url{https://proceedings.mlr.press/v119/chen20j.html}.

\bibitem[Zbontar et~al.(2021)Zbontar, Jing, Misra, LeCun, and Deny]{zbontar2021barlow}
Jure Zbontar, Li~Jing, Ishan Misra, Yann LeCun, and Stephane Deny.
\newblock Barlow twins: Self-supervised learning via redundancy reduction.
\newblock In Marina Meila and Tong Zhang, editors, \emph{Proceedings of the 38th International Conference on Machine Learning}, volume 139 of \emph{Proceedings of Machine Learning Research}, pages 12310--12320. PMLR, 18--24 Jul 2021.
\newblock URL \url{https://proceedings.mlr.press/v139/zbontar21a.html}.

\bibitem[Bardes et~al.(2022)Bardes, Ponce, and LeCun]{bardes2022vicreg}
Adrien Bardes, Jean Ponce, and Yann LeCun.
\newblock {VICR}eg: Variance-invariance-covariance regularization for self-supervised learning.
\newblock In \emph{International Conference on Learning Representations}, 2022.
\newblock URL \url{https://openreview.net/forum?id=xm6YD62D1Ub}.

\bibitem[Garrido et~al.(2023)Garrido, Chen, Bardes, Najman, and LeCun]{garrido2023on}
Quentin Garrido, Yubei Chen, Adrien Bardes, Laurent Najman, and Yann LeCun.
\newblock On the duality between contrastive and non-contrastive self-supervised learning.
\newblock In \emph{The Eleventh International Conference on Learning Representations}, 2023.
\newblock URL \url{https://openreview.net/forum?id=kDEL91Dufpa}.

\end{thebibliography}
}

\appendix
\newpage
\section*{Appendix}
The Appendix is organized as follows:

\begin{itemize}
\item Section \ref{app:background} provides additional background information to supplement the main text.
\item Section \ref{app:proofs} presents all theoretical proofs, including detailed assumptions.
\item Section \ref{app:experimental_details} outlines comprehensive experimental details, ensuring reproducibility.
\item Section \ref{app:additional-results} offers supplementary empirical results that further support our findings.
\item Section \ref{app:code} includes the implementation code for computing the \ref*{def:gromov-monge-gap}.
\end{itemize}

\section{Additional Background}
\label{app:background}
\subsection{Reminders on Monge and Kantorovich OT}

In this section, we recall the Monge and Kantorovich formulations of OT, which we will use to prove various results. These are the classical formulations of OT. Although we introduce them here after discussing the Gromov-Monge and Gromov-Wasserstein formulations, it should be noted that they are generally introduced beforehand. Indeed, the Gromov-Monge and Gromov-Wasserstein formulations were historically developed to derive OT formulations for comparing measures supported on incomparable spaces.

\paragraph{Monge Formulation.} Instead of intra-domain cost functions, we consider here an \textit{inter-domain} continuous cost function \(c : \mathcal{X} \times \mathcal{Y} \to \mathbb{R}\). This assumes that we have a meaningful way to compare elements \(\*x, \*y\) from the source and target domains. The \citet{Monge1781} problem~\ref{eq:monge-problem} between \( p \in \mathcal{P}(\mathcal{X})\) and \( p \in \mathcal{P}(\mathcal{Y})\) consists of finding a map \(T: \mathcal{X} \to \mathcal{Y}\) that push-forwards \( p\) onto \( p\), while minimizing the average displacement cost quantified by $c$
\begin{equation}
\label{eq:monge-problem}
\tag*{(MP)}
\inf_{T:T\sharp p= p} \int_\+X c(\*x, T(\*x)) \diff p(\*x)\,.
\end{equation}
We call any solution $T^\star$ to this problem a Monge map between $ p$ and $ q$ for cost $c$. 
Similarly to the Gromov-Monge Problem~\ref{eq:gromov-monge-problem}, solving the Monge Problem~\ref{eq:monge-problem} is difficult, as the constraint set is not convex and might be empty, especially when $ p, q$ are discrete. 

\paragraph{Kantorovich Formulation.} Instead of transport maps, the \citeauthor{kantorovich1942transfer} problem~\ref{eq:kantorovich-problem} seeks a couplings $\pi \in \Pi( p,  q)$:
\begin{equation}
\label{eq:kantorovich-problem}
\tag*{(KP)}
\mathrm{W}( p, q) := \min_{\pi \in \Pi( p,  q)} \int_{\+X\times\+Y} c(\*x, \*y) \diff\pi(\*x, \*y)\,.
\end{equation}
An optimal coupling $\pi^\star$ solution of~\ref{eq:kantorovich-problem}, always exists. Studying the equivalence between~\ref{eq:monge-problem} and \ref{eq:kantorovich-problem} is easier than in the Gromov-Monge and Gromov-Wasserstein cases. Indeed,
when~\ref{eq:monge-problem} is feasible, the Monge and Kantorovich formulations coincide and $\pi^\star = (\Id, T^\star)\sharp p$. 

\subsection{Conditionally Positive Kernels}

In this section, we recall the definition of a conditionally positive kernel, which is involved in multiple proofs relying on the linearization of the Gromov-Wasserstein problem as a Kantorovich problem. 

\begin{definition}
\label{def:conditional-kernel}
A kernel $k : \mathbb{R}^d \times \mathbb{R}^d \rightarrow \mathbb{R}$ is CPD, i.e., conditionally positive, if it is symmetric and for any $\*x_1, ..., \*x_n \in \mathbb{R}^d$ and $\*a \in \mathbb{R}^n$ s.t.\ $\*a^\top\*1_n = 0$, one has 
$$
\sum_{i,j=1}^n \*a_i \*a_j \, k(\*x_i, \*x_j) \geq 0 
$$
\end{definition}

CPD include all positive kernels, such as the inner-product $k(\*x, \*x') = \langle \*x , \*x'\rangle$, or the cosine similarity $k(\*x, \*x') = \textrm{cos-sim}(\*x, \*x') = \langle \tfrac{\*x}{\|\*x\|_2} , \tfrac{\*x'}{\|\*x'\|_2} \rangle$, but also the (scaled) negative squared Euclidean distance $k(\*x, \*x') = - \alpha \| \*x -\*x'\|_2^2$, $\alpha > 0$. Therefore, each of the costs of interest is either a conditionally positive kernel - for the inner product and the cosine distance - or its opposite is - for the squared Euclidean distance. Additionally, CPD kernels also include more exotic cost functions, such as:

\begin{itemize}[leftmargin=.5cm,itemsep=.0cm,topsep=0cm]
\item The RBF kernel $k(\*x, \*x') = \exp(-\|\*x - \*y\|_2^2/\gamma), \gamma >0$.
\item The power kernel $k(\*x, \*x') = - \|\*x - \*x'\|_2^p, 0 < p < 2$. 
\item The thin plate spline kernel $k(\*x, \*x') = \|\*x - \*x'\|_2^2\log(\|\*x-\*x')$
\item The inverse multi-quadratic kernel $k(\*x, \*x') = 1 / \sqrt{\|\*x - \*x'\|_2^2 + c^2}, c\in \mathbb{R}$. 
\end{itemize}

As a result, the family of CPD kernels includes a large variety of cost functions, which can be used to define various~\ref{def:gromov-monge-gap}s.

\section{Proofs}
\label{app:proofs}

\subsection{The GMG characterizes Gromov-Monge Optimality}
\label{sec:gmg-equals-0}

We show here that if $\mathrm{GMG}_r(T) = 0$, then $T$ is a Gromov-Monge map between $r$ and $T\sharp r$ for costs $c_\+X,c_\+Y$. As the set of deterministic couplings $\{\pi_F \eqdef (\mathrm{I}_d, F)\sharp r | F:\+X\to\+Y\,, F\sharp r = T \sharp r\} \subset \Pi(r, T\sharp r\}$, we immediately get that
\begin{equation}
\label{eq:gmg-equals-0-eq-1}
    \inf_{F\sharp r = T\sharp r}\int_{\+X\times\+X} (c_\+X(\*x,\*x')-c_\+Y(F(\*x),F(\*x')))^2 \diff r(\*x)\diff r(\*x') \geq \mathrm{GW}(r, T\sharp r)
\end{equation}

On the other hand, if $\mathrm{GMG}_r(T) = 0$, one has 
\begin{equation}
\label{eq:gmg-equals-0-eq-2}
    \mathrm{GW}(r, T\sharp r) = \int_{\+X\times\+X} (c_\+X(\*x,\*x')-c_\+Y(T(\*x),T(\*x')))^2 \diff r(\*x)\diff r(\*x')
\end{equation}
Therefore, combining Eq.~\eqref{eq:gmg-equals-0-eq-1} and Eq.~\eqref{eq:gmg-equals-0-eq-2}, we get 

\begin{equation}
\begin{split}
    & \inf_{F\sharp r = T\sharp r}\int_{\+X\times\+X} (c_\+X(\*x,\*x')-c_\+Y(F(\*x),F(\*x')))^2 \diff r(\*x)\diff r(\*x') \\
    = & \int_{\+X\times\+X} (c_\+X(\*x,\*x')-c_\+Y(T(\*x),T(\*x')))^2 \diff r(\*x)\diff r(\*x')
\end{split}
\end{equation}
Finally, as $T$ naturally satisfies the marginal constraint, we conclude that $T$ is a Gromov-Monge map between $r$ and $T\sharp r$ for costs $c_\+X,c_\+Y$. 

\subsection{On rescaling the costs matrices in the entropic GW solver}
\label{sec:rescaling-costs}

We remind, from Eq.~\ref{eq:entropic-gromov-wasserstein}, that
\begin{equation}
\mathrm{GW}_\varepsilon(\emp*p, \emp*q) = \min_{\*P \in U_n} \sum_{i,j,i',j'=1}^n (\*C_{\+X_{i,i'}}-\*C_{\+Y_{j,j'}})^2 \, \*P_{i,j}\*P_{i',j'}  - \varepsilon H(\*P)\,.
\end{equation}
By developping each terms, and using the fact that $\*P\in U_n$, we get

\begin{equation}
\begin{split}
\mathrm{GW}_\varepsilon(\emp*p, \emp*q) 
& = 
\min_{\*P \in U_n}
\tfrac{1}{n^2} \langle \*C_\+X^{\odot 2} \*1_n, \*1_n \rangle + \tfrac{1}{n^2} \langle \*C_\+Y^{\odot 2} \*1_n, \*1_n \rangle - 2 \langle \*C_\+X\*P\*C_\+Y , \*P \rangle
- \varepsilon H(\*P) \\
& = 
\tfrac{1}{n^2} \langle \*C_\+X^{\odot 2} \*1_n, \*1_n \rangle + \tfrac{1}{n^2} \langle \*C_\+Y^{\odot 2} \*1_n, \*1_n \rangle + 
\min_{\*P \in U_n}
- 2 \langle \*C_\+X\*P\*C_\+Y , \*P \rangle
- \varepsilon H(\*P)\,
\end{split}
\end{equation}
where $ \*C_\+X^{\odot 2} = \*C_\+X \odot \*C_\+X$, with $\odot$ the Hadamard (i.e., elementwise) product, and similarly for $\*C_\+Y^{\odot 2}$.
As we can see that the two terms on the left do not depend on $\*P$, they do not impact the minimization, an OT coupling $\*P^\star$ solving the problem satisfies:
\begin{equation}
\begin{split}
\*P^\star \in \argmin_{\*P \in U_n} \, 
- 2 \langle \*C_\+X\*P\*C_\+Y , \*P \rangle
- \varepsilon H(\*P)\,.
\end{split}
\end{equation}
As a result, if we now replace $\*C_\+X$ and $\*C_\+Y$ by $\*C_\+X/\texttt{stat}(\*C_\+X)$ and $\*C_\+Y/\texttt{stat}(\*C_\+Y)$, respectively, the new OT coupling $\tilde{\*P}^\star$ solving the problem satisfies
\begin{equation}
\begin{split}
& \tilde{\*P}^\star \in \argmin_{\*P \in U_n} \,
- 2 \langle \frac{\*C_\+X}{\texttt{stat}(\*C_\+X)}\*P\frac{\*C_\+Y}{\texttt{stat}(\*C_\+Y)} , \*P \rangle
- \varepsilon H(\*P)\, \\
\Leftrightarrow \quad
& \tilde{\*P}^\star \in \argmin_{\*P \in U_n} \,
- \frac{2}{{\texttt{stat}(\*C_\+X)}\cdot\texttt{stat}(\*C_\+Y)} \langle \*C_\+X\*P\*C_\+Y , \*P \rangle
- \varepsilon H(\*P)\, \\
\Leftrightarrow \quad
& \tilde{\*P}^\star \in \argmin_{\*P \in U_n} \,
- 2 \langle \*C_\+X\*P\*C_\+Y , \*P \rangle
- \texttt{stat}(\*C_\+X)\cdot\texttt{stat}(\*C_\+Y) \cdot \varepsilon H(\*P)\,,
\end{split}
\end{equation}
where in the last line, we use the fact that $\texttt{stat}(\*C_\+X),\texttt{stat}(\*C_\+Y)>0$. This yields the desired equivalence on scaling the cost matrix and adapting the entropic regularization strength.

\subsection{Positivity of the Entropic GMG estimator}
\label{sec:positivity-entropic-estimator}

Recall that 
\begin{align*}
\begin{split}
\mathrm{GMG}_{\emp*r, \varepsilon}(T) 
& = \mathrm{DST}_{r_n}(T) - \mathrm{GW}_{\varepsilon}(\emp*r, T \sharp \emp*r)\, \\
& = \mathrm{DST}_{r_n}(T) - \min_{\*P \in U_n} \sum_{i,j,i',j'=1}^n (c_\+X(\*x_i,\*x_j) - c_\+Y(\*y_i,\*y_j))^2 \*P_{ij}\*P_{i'j'}  - \varepsilon H(\*P)\,,
\end{split}
\end{align*}
 For any coupling $\*P \in U_n$, since $-\varepsilon H(\*P)  = -\varepsilon \sum_{i,j=1}^n \*P_{ij} \log(\*P_{ij}) < 0$, one has:
$$\sum_{i,j,i',j'=1}^n (c_\+X(\*x_i,\*x_j) - c_\+Y(\*y_i,\*y_j))^2 \*P_{ij}\*P_{i'j'} -\varepsilon H(\*P) < \sum_{i,j,i',j'=1}^n (c_\+X(\*x_i,\*x_j) - c_\+Y(\*y_i,\*y_j))^2 \*P_{ij}\*P_{i'j'}$$
As a result, applying minimization on both sides yields that $\mathrm{GW}_{\varepsilon}(\emp*r, T \sharp \emp*r) < \mathrm{GW}_{0}(\emp*r, T \sharp \emp*r) = \mathrm{GW}(\emp*r, T \sharp \emp*r)$, and therefore: 
$$
\mathrm{GMG}_{r_n, \varepsilon}(T) > \mathrm{GMG}_{r_n, 0}(T) = \mathrm{GMG}_{r_n}(T)\geq 0. 
$$

\subsection{Proofs of Prop.~\ref{prop:more-do-less}}

\PropMoreDoLess*

\begin{proof}

Let $T,  r,  s$ as described and suppose that $\mathcal{GM}_r^c(T)=0$. Then, $\pi^r := (\mathrm{Id}, T)\sharp r$ is an optimal Gromov-Wasserstein coupling, solution of Problem~\ref{eq:gromov-wasserstein-problem} between $r$ and $T \sharp r$ for costs $c_\+X$ and $c_\+Y$. Therefore, from \citep[Theorem.\ 3]{sejourne2023the-unbalanced-gromov}, $\pi^r$ is an optimal Kantorvich coupling, solution of Problem \ref{eq:kantorovich-problem} between $r$ and $T\sharp r$ for the linearized cost:

\begin{equation}
    \tilde{c} : (\*x, \*y) \in \+X \times \+Y \mapsto \int_{\+X\times\+Y} \tfrac{1}{2} |c_\+X(\*x,\*x') - c_\+Y(\*y,\*y')|^2 \diff\pi^r(\*x',\*y')
\end{equation}

Additionally, $\+X \times \+Y$ is a compact set as a product of compact sets, so since $(\*x, \*y) \mapsto |c_\+X(\*x, \*x') - c_\+Y(\*y,\*y')|^2$ is continuous as $c_\+X$ and $c_\+Y$ are continuous, it is bounded on $\+X \times \+Y$. Afterward, since $\pi^r$ has finite mass, by Lebesgue's dominated convergence Theorem, it follows that $\tilde{c}$ is continuous, and hence uniformly continuous, again since $\+X \times \+Y$ is compact. 

Afterwards, by virtue of \citep[Theorem 1.38]{santambrogio2015optimal}, $\supp\left(\pi^r \right)$ is a $\tilde{c}$-cyclically monotone (CM) set (see \citep[Definition.\ 1.36]{santambrogio2015optimal}). From the definition of cyclical monotonicity, this property translates to subsets. Then, by defining $\pi^s = (\Id, T)\sharp s$, as $\supp( p) \subset \supp(r)$, one has $\supp(\pi^s) = \supp((\Id, T)\sharp s) \subset \supp((\Id, T)\sharp r) = \supp(\pi^r)$, so $\supp(\pi^s)$ is $\tilde{c}$-CM. Finally, since $\+X$ and $\+Y$ are compact, and $\tilde{c}$ is uniformly continuous, the $\tilde{c}$-cyclical monotonicity of its support implies that the coupling $\pi^ p$ is a Kantorovich optimal coupling between its marginals for cost $\tilde{c}$, thanks to \citep[Theorem 1.49]{santambrogio2015optimal}. By re-applying \citep[Theorem.\ 3]{sejourne2023the-unbalanced-gromov}, we get that $\pi^s$ solves the Gromov-Wasserstein problem between its marginals for costs $c_\+X$ and $c_\+Y$. In other words, $\pi^s = (\Id, T)\sharp s$ is Gromov-Wasserstein optimal coupling between $ s$ and $T\sharp s$ so $T$ is a Gromov-Monge map between $ s$ and $T\sharp s$ and $\mathrm{GMG}_s(T)=0$.
\end{proof}

\subsection{Proofs of Thm.~\ref{thm:weak-convexity}}

\WeakConvexity*

Before proving Thm.~\ref{thm:weak-convexity}, we first demonstrate some \textbf{technical results} that will be useful later.

\paragraph{Reformulation of the empirical GMG using permutations.} We start by showing that $\mathrm{GMG}_{\emp*r}(T)$ is always the sub-optimality gap of $T$ in Prob.~\ref{eq:gromov-monge-problem} between $r_n$ and $T\sharp r_n$. This occurs because Prob.~\ref{eq:gromov-monge-problem} and Prob.~\ref{eq:gromov-wasserstein-problem} coincide when applied between empirical measures on the same number of points. In other words, we can reformulate Prob.~\ref{eq:gromov-wasserstein-problem} between $r_n$ and $T\sharp r_n$ using permutation matrices, instead of (plain) couplings.

\begin{restatable}{proposition}{Permutation}
\label{prop:discrete-permutation}
The empirical GMG reads
\begin{align}
\label{eq:gromov-gap-estimator}
\begin{split} 
\mathrm{GMG}_{\emp*r}(T) =
 \mathrm{DST}_{\emp*r}(T) - 
 \min_{\sigma \in \mathcal{S}_n} \tfrac{1}{n^2} \sum_{i,j=1}^n \left(c_\+X(\*x_i, \*x_j) - c_\+Y(T(\*x_{\sigma(i)}), T(\*x_{\sigma(j)}))\right)^2
\end{split}
\end{align}
\end{restatable}

\begin{proof}

We first show a more general results, stating that when $c_\+X, c_\+Y$ are conditionally positive kernels (see~\ref{def:conditional-kernel}), the discrete GW couplings between uniform, empirical distributions supported on the same number of points, ae permutation matrices. 

\begin{proposition}[Equivalence between Gromov-Monge and Gromov-Wasserstein problems in the discrete case.]
\label{prop:equivalence-gm-gw}
    Let $\emp* p = \tfrac{1}{n} \sum_{i=1}^n \delta_{\*x_i}$ and $\emp* q = \tfrac{1}{n} \sum_{i=1}^n \delta_{\*y_i}$ two uniform, empirical measures, supported on the same number of points.  We denote by $P_n = \{ \*P \in \mathbb{R}^{n\times n}, \exists \sigma \in \mathcal{S}_n, \*P_{ij} := \delta_{j, \sigma(i)}\}$ the set set of permutation matrices.  
    Assume that $c_\+X$ and $c_\+Y$ (or $- c_\+X$ and $- c_\+Y$) are conditionally positive kernels (see~\ref{def:conditional-kernel}). 
    Then, the GM and GW formulations coincide, in the sense that we can restrict the GW problem to permutations, namely
\begin{equation}
\begin{split}
    \mathrm{GW}(\emp* p, \emp* p) 
    & = \min_{\*P\in U_n} \sum_{i,j,i',j'=1}^n (c_\+X(\*x_i, \*x_{i'}) - c_\+Y(\*y_j, \*y_{j'}))^2\*P_{ij}\*P_{i'j'} \\
    & = \tfrac{1}{n^2} \min_{\*P\in P_n} \sum_{i,j,i',j'=1}^n (c_\+X(\*x_i, \*x_{i'}) - c_\+Y(\*y_j, \*y_{j'}))^2\*P_{ij}\*P_{i'j'} \\
    & = \tfrac{1}{n^2} \min_{\sigma \in \mathcal{S}_n}
    \sum_{i,j=1}^n (c_\+X(\*x_i, \*x_{j}) - c_\+Y(\*y_{\sigma(i)}, \*y_{\sigma(j)}))^2
\end{split}
\end{equation}
\end{proposition}

\begin{proof} 

Let $\*P^\star \in U_n$ solution of the Gromov-Wasserstein between $\emp* p$ and $\emp* p$, i.e.
$$
\*P^\star \in \argmin_{\*P\in U_n} \sum_{i,j,i',j'=1}^n (c_\+X(\*x_i, \*x_{i'}) - c_\+Y(\*y_j, \*y_{j'}))^2\*P_{ij}\*P_{i'j'}
$$
that always exists by continuity of the GW objective function on the compact $U_n$. We show that $\*P^\star$ can be chosen as a (rescaled) permutation matrix without loss of generality. 

As we assume that $c_\+X$ and $c_\+Y$ (or $- c_\+X$ and $- c_\+Y$) are conditionally positive kernels, from \citep[Theorem.\ 3]{sejourne2023the-unbalanced-gromov}, $\*P^\star$ also solves:
\begin{equation}
\label{eq:linearized-problem}
\*P^\star \in \argmin_{\*Q\in U_n} \sum_{i,j,i',j'=1}^n (c_\+X(\*x_i, \*x_{i'}) - c_\+Y(\*y_j, \*y_{j'}))^2\*P^\star_{ij}\*Q_{i'j'}
\end{equation}

We then define the linearized cost matrix $\tilde{C} \in \mathbb{R}^{n\times n}$, s.t.
\begin{equation*}
\tilde{\*C}_{ij} = \sum_{i',j'=1}^n (c_\+X(\*x_i, \*x_{i'}) - c_\+Y(\*y_j, \*y_{j'}))^2\*P^\star_{ij}
\end{equation*}

which allows us to reformulate Eq.~\eqref{eq:linearized-problem} as
\begin{equation}
\label{eq:linearized-problem-2}
\*P^\star \in \argmin_{\*Q\in U_n} \langle \tilde{\*C}, \*Q \rangle 
\end{equation}

\citeauthor{birkhoff}'s theorem states that the extremal points of $U_n$
are the permutation matrices $P_n$. Moreover, a seminal theorem of linear programming \citep[Theorem 2.7]{bertsimas-LPbook} states that the minimum of a linear objective on a bounded polytope, if finite, is reached at an extremal point of the
polyhedron. Therefore, as $\*P^\star$ solves Eq.~\eqref{eq:linearized-problem-2}, it is an extremal point of $U_n$, so it can always be chosen as a permutation matrix. Therefore, the equivalence between GW and GM follows. 

\end{proof}

To conclude the proof of Prop.~\ref{prop:discrete-permutation}, we simply remark that $r_n = \tfrac{1}{n}\sum_{i=1}^n\delta_{\*x_i}$ and $T\sharp r_n = \tfrac{1}{n}\sum_{i=1}^n\delta_{T(\*x_i)}$ are uniform, empirical distribution, and supported on the same number of points.

\end{proof}

\paragraph{Consistency of the empirical GMG.} We continue by proving a consistency result for the empirical GMG, which we will later use to deduce the asymptotic weak convexity constant from the finite-sample case.

\begin{restatable}{proposition}{Consistency}
\label{prop:consistency}
For both $c_\+X=c_\+Y=\|\cdot-\cdot\|_2^2$ and $c_\+X=c_\+Y=\langle\cdot,\cdot\rangle$, one has
 $\mathrm{GMG}_{\emp*r}(T) \rightarrow \mathrm{GMG}_{r}(T)$ almost surely.
\end{restatable}

\begin{proof}

We first note that the empirical estimator of the distortion is consistent, as both costs are continuous, and $\+X$ is compact. We then need to study, in both cases, the convergence of $\mathrm{GW}(r_n, T\sharp r_n)$ to $\mathrm{GW}(r_n, T\sharp r)$.

To that end, we first remark that as, almost surely, $r_n \to r$ in distribution, one also has that, almost surely, $T\sharp r_n \to T\sharp r$ in distribution. Indeed, since $\+Y$ is compact, $T$ is bounded so for any bounded and continuous $f : \+Y \rightarrow \mathbb{R}$ and $X \sim r$, $f \circ T(X)$ is well defined and bounded so integrable. Afterwards, one can simply adapt the proof of the almost sure weak convergence of empirical measure based on the strong law of large numbers to show that, almost surely, $T\sharp \emp*r \rightarrow T\sharp r$ in distribution. See for instance \citep[Theorem 10.4.1]{legall2006integration}. 

We start with the squared Euclidean distance. As, almost surely, both $r_n \to r$ and  $T\sharp \emp*r \to T\sharp r$ in distribution, the results follows from \citep[Thm 5.1, (e)]{memoli2011gromov}.

We continue with the inner product. As noticed by~\citet[Lemma 2]{rioux2023entropic}–in the first version of the paper – the GW for inner product costs can be reformulated as:
\begin{equation}
\begin{split}
    \mathrm{GW}^{\langle\cdot,\cdot\rangle}(p, q) = &\int_{\+X\times\+X} \langle\*x,\*x'\rangle\diff p(\*x)\diff p(\*x') + \int_{\+Y\times\+Y} \langle\*y,\*y'\rangle\diff q(\*y)\diff q(\*y') \\
    + & \min_{\*M\in \mathcal{M}}\min_{\pi \in \Pi(p, q)} \int_{\+X\times\+Y} - 4 \langle\*M\*x, \*y\rangle \diff\pi(\*x,\*y) + 4 \|\*M\|_2^2\,,
\end{split}
\end{equation}
where we define $\mathcal{M} = [-M/2, M/2]^{d_\+X\times d_\+Y}$ with $M = \sqrt{\int_\+X\|\*x\|_2^2\diff p(\*x) \int_\+Y\|\*y\|_2^2\diff q(\*y)}$. 
In particular, they show this result for the entropic GW problem with $\varepsilon > 0$, but their proof is also valid for $\varepsilon=0$. The above terms only involving the marginal, i.e., not involved in the minimization, are naturally stable under convergence in distribution, as $\+X$ and $\+Y$ are compact, so as $\+X\times\+X$ and $\+Y\times\+Y$.
As a result, we only need to study the stability of this quantity under the convergence in distribution of the following functional:
\begin{equation}
    \mathcal{F}(p, q) = \min_{\*M\in \mathcal{M}}\min_{\pi \in \Pi(p, q)} \int_{\+X\times\+Y} - 4 \langle\*M\*x, \*y\rangle \diff\pi(\*x,\*y) + 4 \|\*M\|_2^2\,,
\end{equation}
We first remark that:

\begin{equation}
\begin{split}
    & |\mathcal{F}(p, q) - \mathcal{F}(p_n, q_n)| \\
    \leq & \sup_{M\in\mathcal{M}}
    | \min_{\pi \in \Pi(p, q)} \int_{\+X\times\+Y} - 4 \langle\*M\*x, \*y\rangle \diff\pi(\*x,\*y)
    -
    \min_{\pi \in \Pi(p, q)} \int_{\+X\times\+Y} - 4 \langle\*M\*x, \*y\rangle \diff\pi(\*x,\*y)
    | \\
    \leq &
    \sup_{M\in\mathcal{M}} | \min_{\pi \in \Pi(p, q)} \int_{\+X\times\+Y} 2\|\*M\*x-\*y\|_2^2 \diff\pi(\*x,\*y)
    -
    \min_{\pi \in \Pi(p_n, q_n)} \int_{\+X\times\+Y}  2\|\*M\*x-\*y\|_2^2 \diff\pi(\*x,\*y)
    2| \\
    + & 
    2\cdot\,\sup_{M\in\mathcal{M}} |\int_\+X\|\*M\*x\|_2^2 \diff p(\*x) - \int_\+X\|\*M\*x\|_2^2 \diff p_n(\*x) | \\
    + &
    2\cdot\,|\int_\+Y\|\*y\|_2^2 \diff q(\*y) - \int_\+Y\|\*y\|_2^2 \diff q_n(\*y) |
\end{split}
\end{equation}

Then, we show the convergence of each term separately. 

\begin{itemize}[leftmargin=.5cm,itemsep=.0cm,topsep=0cm]

\item For the first term, we remark that (up to a constant factor) it can be reformulated:
$$
\sup_{M\in\mathcal{M}} | \mathrm{W}_2^2(\*M\sharp p, q) - \mathrm{W}_2^2(\*M\sharp p_n, q_n) |
$$
where we remind that that $\mathrm{W}_2^2$ is the (squared) Wasserstein distance, solution of Eq.~\ref{eq:kantorovich-problem} induced by $c(\*x, \*y) = \|\*x - \*y\|_2^2$. By virtue of~\citep[Theorem 2]{Manole_2024}, there exists a constant $C > 0$, s.t.\ we can uniformly bound 
$$
\sup_{M\in\mathcal{M}} | \mathrm{W}_2^2(\*M\sharp p, q) - \mathrm{W}_2^2(\*M\sharp p_n, q_n) | \leq C n^{-1/d}
$$
and the convergence follows. 

\item For the second one, this follows from from the convergence in distribution of $p_n$ to $p$ along with the Ascoli-Arzela theorem, since both $\mathcal{M}$ and $\+X$ are compact sets, so the $\{f_\*M \, | \, f_\*M : \*x \mapsto \|\*M\*x\|_2^2\}$ are uniformly bounded and equi-continuous.

\item For the third one, this follows from the convergence in distribution of $q_n$ to $q$.

\end{itemize}

As a result, we finally get $\mathrm{GW}^{\langle\cdot,\cdot\rangle}(p_n, q_n) \to \mathrm{GW}^{\langle\cdot,\cdot\rangle}(p, q)$.

\end{proof}

\paragraph{Weak convexity.} Finally, we demonstrate some useful results on weakly convex functions on $\mathbb{R}^d$. 

\begin{definition}
    A function $f : \mathbb{R}^d \to \mathbb{R}$ is $\gamma$-weakly convex if $f + \gamma\|\cdot\|_2^2$ is convex.
\end{definition}

From the definition, we see that if $f$ is $\gamma$-weakly convex, than $f$ is also $\gamma'$ weakly convex for any $\gamma' \geq \gamma$. This naturally extends to weakly convex functionals $\mathcal{F}$ on $L_2(r)$.

\begin{lemma}
\label{prop:quadratic-form-weakly-convex}
Let $\*A \in S_d(\mathbb{R})$ a symmetric matrix and define the quadratic form $f_\*A : \*x \in \mathbb{R}^d \mapsto \*x^\top \*A \*x$. Then, $f_\*A$ is $\max(0, -\lambda_{\min} (\*A))$-weakly convex.
\end{lemma}
\begin{proof}
We use the fact that a twice continuously differentiable function is convex i.f.f.\ its hessian is positive semi-definite~\citep[\S (3.1.4)]{boyd2004convex}. Therefore, $f_\*A$ is convex i.f.f.\ $\nabla^2 f_\*A = \*A \geq 0$. If $\lambda_{\min}(\*A) \geq 0$, then $\*A \geq 0$ so $f_\*A$ is convex, i.e.\ $0$-weakly convex. Otherwise, $f_\*A - \tfrac{1}{2}\lambda_{\min}(\*A)\|\cdot\|_2^2$ has hessian $A - \lambda_{\min}(\*A) \geq 0$, so it is convex, which yields that $f_\*A$ is $-\lambda_{\min}(\*A)$-weakly convex.
\end{proof}

\begin{lemma}
\label{prop:maximum-weakly-convex}
Let $(f_i)_{i\in I}$ a family of $\gamma$-weakly convex functions, with potentially infinite $I$. Then, $f : \*x \in \mathbb{R}^d \mapsto \sup_{i\in I} f_i(\*x)$ is $\gamma$-weakly convex.
\end{lemma}
\begin{proof}
As the $f_i$ are $\gamma$-weakly convex, $f_i + \tfrac{1}{2}\gamma$ is convex, so $\*x \mapsto \sup_{i\in I}f_i(\*x) + \tfrac{1}{2}\gamma\|x\|_2^2 = (\sup_{i\in I}f_i(\*x)) + \tfrac{1}{2}\gamma\|x\|_2^2$ is convex~\citep[Eq.\ (3.7)]{boyd2004convex}. Therefore, the $\gamma$-weak convexity of $f$ follows
\end{proof}

Let’s now proceed to prove the main Thm.~\eqref{thm:weak-convexity}.

\begin{proof}[Proof of Thm.~\eqref{thm:weak-convexity}.]

\textbf{Finite sample}. We first study the weak convexity of $\mathcal{GM}^{\langle \cdot, \cdot \rangle}_{\emp*r}$, i.e. the Gromov-Monge gap for the inner product. For a map $T \in L_2(r)$, it reads

\begin{equation*}
\begin{split}
    \mathrm{GMG}^{\langle \cdot, \cdot \rangle}_{\emp*r}(T)  
    & = \tfrac{1}{n^2}\sum_{i,j=1}^n \tfrac{1}{2}| \langle \*x_i, \*x_j \rangle - \langle T(\*x_i), T(\*x_j) \rangle |^2 \\
    & - \min_{\*P \in U_n}\sum_{i,j,i',j'=1}^n \tfrac{1}{2}| \langle \*x_i, \*x_{i'} \rangle - \langle T(\*x_j), T(\*x_{j'}) \rangle |^2 \*P_{ij}\*P_{i'j'}
\end{split}
\end{equation*}

As $\emp*r$ and $T\sharp\emp*r$ are uniform empirical supported on the same number of points, using Prop.~\ref{prop:equivalence-gm-gw}, we can reformulate the RHS with permutation matrices, which yields

\begin{equation*}
\begin{split}
    \mathrm{GMG}^{\langle \cdot, \cdot \rangle}_{\emp*r}(T)  
    & = \tfrac{1}{n^2}\sum_{i,j=1}^n \tfrac{1}{2}| \langle \*x_i, \*x_j \rangle - \langle T(\*x_i), T(\*x_j) \rangle |^2 \\
    & - \tfrac{1}{n^2}\min_{\*P \in P_n}\sum_{i,j,i',j'=1}^n \tfrac{1}{2}| \langle \*x_i, \*x_{i'} \rangle - \langle T(\*x_j), T(\*x_{j'}) \rangle |^2 \*P_{ij}\*P_{i'j'}
\end{split}
\end{equation*}

From this expression, $\mathrm{GMG}^{\langle \cdot, \cdot \rangle}_{\emp*r}$ can be reformulated as a matrix input function. Indeed, it only depends on the map $T$ via its values on the support of $\emp*r$, namely $\*x_1, ..., \*x_n$. Therefore, we write $\*t_i \eqdef T(\*x_i)$, and define $\*X, \*T\in \mathbb{R}^{n\times d}$ which contain observations $\*x_i$ and $\*t_i$ respectively, stored as rows. Then, studying $\mathrm{GMG}^{\langle \cdot, \cdot \rangle}_{\emp*r}$ remains to study 

\begin{equation*}
    f(\*T) \eqdef \tfrac{1}{n^2}\sum_{i,j=1}^n \tfrac{1}{2} | \langle \*x_i, \*x_j \rangle - \langle \*t_i, \*t_j \rangle |^2 
    - \tfrac{1}{n^2}\min_{\*P \in P_n}\sum_{i,j,i',j'=1}^n \tfrac{1}{2} | \langle \*x_i, \*x_{i'} \rangle - \langle \*t_j, \*t_{j'} \rangle |^2 \*P_{ij}\*P_{i'j'}
\end{equation*}

By developing each term and exploiting that for any $\*P \in P_n$, $\*P\mathrm{1}_n = \*P^\top\mathrm{1}_n = \tfrac{1}{n}\mathrm{1}_n$, we derive 

\begin{equation*}
\begin{split}
    f(\*T) 
    & = \tfrac{1}{n^2}\sum_{i,j=1}^n  - \langle \*x_i, \*x_j \rangle \cdot \langle \*t_i, \*t_j \rangle 
     - \min_{\*P \in P_n} \tfrac{1}{n^2} \sum_{i,j,i',j'=1}^n - \langle \*x_i, \*x_{i'} \rangle \cdot \langle \*t_j, \*t_{j'} \rangle \*P_{ij}\*P_{i'j'} \\
     & = \max_{\*P\in P_n} 
     \tfrac{1}{n^2} \sum_{i,j,i',j'=1}^n \langle \*x_i, \*x_{i'} \rangle \cdot \langle \*t_j, \*t_{j'} \rangle \*P_{ij}\*P_{i'j'}
     - 
     \tfrac{1}{n^2} \sum_{i,j=1}^n \langle \*x_i, \*x_j \rangle \cdot \langle \*t_i, \*t_j \rangle 
     \\
     & = \max_{\*P\in P_n} 
     \langle \tfrac{1}{n^2}  \*P^\top \*X\*X^\top \*P, \*T\*T^\top\rangle 
     - \langle \tfrac{1}{n^2} \*X\*X^\top, \*T\*T^\top \rangle \\
     & = \max_{\*P\in P_n} 
     \langle 
     \tfrac{1}{n^2}(\*P^\top\*X\*X^\top \*P - \*X\*X^\top),  \*T\*T^\top
     \rangle \\
     & = \max_{\*P\in P_n} 
     \langle 
     \tfrac{1}{n^2}(\*P^\top\*X\*X^\top \*P - \*X\*X^\top) \*T,  \*T
     \rangle \\
     & = \max_{\*P\in P_n} 
     \langle 
     \*A_{\*X, \*P} \*T,
     \*T
     \rangle
\end{split}
\end{equation*}
where we define $\*A_{\*X, \*P} \eqdef \tfrac{1}{n^2}(\*P^\top\*X\*X^\top \*P - \*X\*X^\top) \in \mathbb{R}^{n \times n}$. To study the convexity of this matrix input function, we vectorize it. From \citep[Eq.\ (520)]{Petersen2008}, we note that, for any $\*M \in \mathbb{R}^{n\times n}$

\begin{equation*}
\begin{split}
    \langle \*M \*T, \*T \rangle =  \*{vec}(\*T)^\top \*{vec}(\*M\*T) = \*{vec}(\*T)^\top (\*M \otimes I_n) \*{vec}(\*T)
\end{split}
\end{equation*}

where $\*{vec}$ is the vectorization operator, raveling a matrix along its rows, and $\otimes$ is the Kronecker product. Applying this identity, we reformulate:

\begin{equation}
    f(\*T) = \max_{\*P\in U_n} 
     \*{vec}(\*T)^\top (\*A_{\*X, \*P} \otimes I_n) \*{vec}(\*T)
\end{equation}

To study the convexity of $r$, we study the convexity of each $r_{\*A_{\*X, \*P}}(\*T) \eqdef \*{vec}(\*T)^\top (\*A_{\*X, \*P} \otimes I_n) \*{vec}(\*T)$, which are quadratic forms induced by the $\*A_{\*X, \*P} \otimes I_n$. This remains to study the (semi-) positive definiteness of the matrices $\*A_{\*X, \*P} \otimes I_n$. As each $\*A_{\*X, \*P} \in \mathbb{R}^{n\times n}$ is symmetric and square, $\*A_{\*X, \*P} \otimes I_n$ is also symmetric and from \citep[Eq.\ (519)]{Petersen2008} its eigenvalues are the outer products of the eigenvalues of $\*A_{\*X, \*P}$ and $I_n$, namely 
\begin{equation}
\begin{split}
    \*{eig}(\*A_{\*X, \*P} \otimes I_n) 
    & = \{\lambda_i(\*A_{\*X, \*P}) \cdot \lambda_j(I_n)\}_{1\leq i,j \leq n} \\
    & = \{
    \underbrace{\lambda_1(\*A_{\*X, \*P}), \dots, \lambda_1(\*A_{\*X, \*P})}_{\textrm{n times}}, \dots, 
    \underbrace{\lambda_n(\*A_{\*X, \*P}), \dots, \lambda_n(\*A_{\*X, \*P})}_{\textrm{n times}}
    \}
\end{split}
\end{equation}

It follows that the minimal eigenvalue of $\*A_{\*X, \*P} \otimes I_n$ is $\lambda_{\min}(\*A_{\*X, \*P} \otimes I_n) = \lambda_{\min}(\*A_{\*X, \*P})$. Utilizing the expression of $\*A_{\*X, \*P}$ 
\begin{equation}
\begin{split}
\lambda_{\min}(\*A_{\*X, \*P}) 
& = \tfrac{1}{n^2}\lambda_{\min}(\*P^\top\*X\*X^\top \*P - \*X\*X^\top) \\
& \geq \tfrac{1}{n^2} (\lambda_{\min}(\*P^\top\*X\*X^\top \*P) + \lambda_{\min}(- \*X\*X^\top)) \\
& = \tfrac{1}{n^2} (\lambda_{\min}(\*P^\top\*X\*X^\top \*P) - \lambda_{\max}(\*X\*X^\top))
\end{split}
\end{equation}

Reminding that $\*P \in U_n$, one has $\*P^\top = \*P^{-1}$, so $\*P^\top\*X\*X^\top$ and $\*X\*X^\top$ are similar, and they have the same eigenvalues. In particular $\lambda_{\min}(\*P^\top\*X\*X^\top \*P) = \lambda_{\min}(\*X\*X^\top)$. Combining these results, it follows that

\begin{equation}
\label{prop:bound-eigenvalues}
    \lambda_{\min}(\*A_{\*X, \*P}\otimes I_n) = \lambda_{\min}(\*A_{\*X, \*P}) 
    \geq \tfrac{1}{n^2} (\lambda_{\min}(\*X\*X^\top) - \lambda_{\max}(\*X\*X^\top)) 
\end{equation}

We then remind that each $r_{\*A_{\*X,\*P}}$ is the quadratic form defined by $\*A_{\*X, \*P}\otimes I_n$, so by applying Prop.~\ref{prop:quadratic-form-weakly-convex}, it is  $\*A_{\*X, \*P}\otimes I_n$-weakly convex, and hence $\tfrac{1}{n^2} (\lambda_{\max}(\*X\*X^\top) - \lambda_{\min}(\*X\*X^\top))$-weakly convex. Therefore, applying Prop.~\eqref{prop:maximum-weakly-convex}, $r$ is $\tfrac{1}{n^2}(\lambda_{\max}(\*X\*X^\top) - \lambda_{\min}(\*X\*X^\top))$-weakly convex, in $\mathbb{R}^d$. Reminding that  $\gamma_\mathrm{inner} = \tfrac{1}{n}(\lambda_{\max}(\*X\*X^\top) - \lambda_{\min}(\*X\*X^\top))$, $r$ is $\tfrac{1}{n}\gamma_\mathrm{inner}$ weakly convex. 
This implies that $\*T \mapsto f(\*T) + \tfrac{1}{n}\gamma_\mathrm{inner} \|\*T\|_2^2$ is convex. By reminding that $\*T$ stores the $T(\*x_i)$ as rows, $\tfrac{1}{n}\|\*T\|_2^2 = \|T\|_{L_2(r_n)}$. Consequently, $\mathrm{GMG}_{r_n}^{\langle\cdot,\cdot\rangle}$ is $\gamma_\mathrm{inner}$ in $L_2(r_n)$.

We then study the convexity of $\mathrm{GMG}_{r_n}^{2}$. We follow exactly the same approach. One has:

\begin{equation*}
\begin{split}
    \mathrm{GMG}^{2}_{\emp*r}(T)  
    & = \tfrac{1}{n^2}\sum_{i,j=1}^n \tfrac{1}{2}| \| \*x_i - \*x_j \|_2^2 - \| T(\*x_i) - T(\*x_j) \|_2^2 |^2 \\
    & - \tfrac{1}{n^2}\min_{\*P \in P_n}\sum_{i,j,i',j'=1}^n \tfrac{1}{2}| \| \*x_i - \*x_j \|_2^2 - \| T(\*x_i) - T(\*x_j) \|_2^2 |^2 |^2 \*P_{ij}\*P_{i'j'}
\end{split}
\end{equation*}

Similarly, studying the convexity of $\mathrm{GMG}^{2}_{\emp*r}(T)$ remains to study the convexity of the matrix input function:

\begin{equation*}
\begin{split}
    g(\*T) 
    & \eqdef \tfrac{1}{n^2}\sum_{i,j=1}^n \tfrac{1}{2}| \| \*x_i - \*x_j \|_2^2 - \| \*t_i - \*t_j \|_2^2 |^2 \\
    & - \tfrac{1}{n^2}\min_{\*P \in P_n}\sum_{i,j,i',j'=1}^n \tfrac{1}{2}| \| \*x_i - \*x_j \|_2^2 - \| \*t_i - \*t_j \|_2^2 |^2 \*P_{ij}\*P_{i'j'}
\end{split}
\end{equation*}

As before, by developing each term, one has:
\begin{equation*}
\begin{split}
    g(\*T) 
    = \max_{\*P\in P_n} 
    &
     \tfrac{1}{n^2} \sum_{i,j,i',j'=1}^n \langle \*x_i, \*x_{i'} \rangle \cdot \langle \*t_j, \*t_{j'} \rangle \*P_{ij}\*P_{i'j'}
     + \tfrac{1}{2n} \sum_{i,j=1}^n \*P_{ij} \|\*x_i\|_2^2 \|\*t_i\|_2^2 
     \\ - 
     & \left(\tfrac{1}{n^2} \sum_{i,j=1}^n \langle \*x_i, \*x_j \rangle \cdot \langle \*t_i, \*t_j 
     \rangle 
     + \tfrac{1}{2n}\sum_{i,j=1}^n \|\*x_i\|_2^2 \|\*t_i\|_2^2 
     \right)
\end{split}
\end{equation*}

The quadratic terms in $\*P$ can be factorized as before using $\*A_{\*X,\*P}$. For the new terms w.r.t.\ the inner product case, we introduce $\*D_{\*X} \eqdef \mathrm{diag}(\|\*x_1\|_2^2, \dots, \|\*x_n\|_2^2)$, and remark that we can rewrite:

\begin{equation*}
\begin{split}
    \tfrac{1}{2n} \sum_{i,j=1}^n \*P_{ij} \|\*x_i\|_2^2 \|\*t_i\|_2^2 -
     \tfrac{1}{2n}\sum_{i,j=1}^n \|\*x_i\|_2^2 \|\*t_i\|_2^2 
     = \*{vec}(T)^\top \left( \tfrac{1}{2n}(\*P^\top - I_n) \otimes \*D_{\*X} \right) \*{vec}(T)
\end{split}
\end{equation*}

As we can always symetrize the matrix when considering its associated quadratic form, we have: 

\begin{equation*}
\begin{split}
    \tfrac{1}{2n} \sum_{i,j=1}^n \*P_{ij} \|\*x_i\|_2^2 \|\*t_i\|_2^2 -
     \tfrac{1}{2n}\sum_{i,j=1}^n \|\*x_i\|_2^2 \|\*t_i\|_2^2 
     = \*{vec}(T)^\top \left( \tfrac{1}{2}(\tfrac{1}{2n} (\*P^\top + \*P) - I_n) \otimes \*D_{\*X} \right) \*{vec}(T)
\end{split}
\end{equation*}

As a result, we denote $\*B_{\*X, \*P} = \tfrac{1}{n}(\tfrac{1}{2}(\*P^\top + \*P) - I_n) \otimes \*D_{\*X}$ and finally get:
\begin{equation*}
\begin{split}
g(\*T) 
 & = \max_{\*P\in P_n} 
 \*{vec}(T)^\top \left(\*A_{\*X,\*P} \otimes I_n + \*B_{\*X, \*P}\right)\*{vec}(T)
\end{split}
\end{equation*}

As we did for $f$, studying the weak convexity of $f$ remains to lower bound the minimal eigenvalue of $\*A_{\*X,\*P} \otimes I_n + \*B_{\*X,\*P}$. First, one remark that:
$$
\lambda_{\min}(\*A_{\*X,\*P} \otimes I_n + \*B_{\*X,\*P}) \geq \lambda_{\min}(\*A_{\*X,\*P} \otimes I_n) + \lambda_{\min}(\*B_{\*X,\*P})
$$

As we we have already lower bounded $\lambda_{\min}(\*A_{\*X,\*P} \otimes I_n)\geq \tfrac{1}{n^2}(\lambda_{\min}(\*X\*X^\top) - \lambda_{\max}(\*X\*X^\top))$, we focus on the RHS. Similarly, one has:
\begin{equation}
\begin{split}
    \lambda_{\min}(\*B_{\*X,\*P}) 
    & = \lambda_{\min} \left( \tfrac{1}{2n}(\tfrac{1}{2} (\*P^\top + \*P) - I_n) \otimes \*D_{\*X} \right) \\
    & \geq 
    \lambda_{\min} \left( \tfrac{1}{4n}(\*P^\top + \*P) \otimes \*D_{\*X} \right) 
    + 
    \lambda_{\min} \left( - \tfrac{1}{2n} I_n \otimes \*D_{\*X} \right) \\
    & \geq 
    \lambda_{\min} \left( \tfrac{1}{4n}(\*P^\top + \*P) \otimes \*D_{\*X} \right) 
    - 
    \lambda_{\max} \left( \tfrac{1}{2n} I_n \otimes \*D_{\*X} \right)
\end{split}
\end{equation}

For both terms, we apply again~\citep[Eq.\ (519)]{Petersen2008}. For the LHS, one has:
\begin{equation}
    \*{eig}\left( \tfrac{1}{4n}(\*P^\top + \*P) \otimes \*D_{\*X} \right) = \{ \lambda_i(\tfrac{1}{4n}(\*P^\top + \*P)) \lambda_j(\*D_{\*X})\}_{1\leq i,j\leq n}
\end{equation}

We remark that $\tfrac{1}{2}(\*P^\top + \*P)$ is a symetric bi-stochastic matrix, so $\lambda_{\min}(\tfrac{1}{2}(\*P^\top + \*P)) \geq -1$. Therefore, $\lambda_{\min}(\tfrac{1}{4n}(\*P^\top + \*P))\geq - \tfrac{1}{2n}$. As a result, since the eigenvalues of $\*D_{\*X}$ are the $\|\*x_i\|_2^2$, this yields:
$$
\lambda_{\min} \left( \tfrac{1}{4n}(\*P^\top + \*P) \otimes \*D_{\*X} \right) \geq -\tfrac{1}{2n}\max_{i=1,\dots, n}\|\*x_i\|_2^2
$$

Similarly, we have:
$$
- \lambda_{\max} \left( \tfrac{1}{2n}I_n \otimes \*D_{\*X} \right) \geq -\tfrac{1}{2n}\max_{i=1,\dots, n}\|\*x_i\|_2^2
$$
from which we deduce that:
$$
\lambda_{\min}(\*B_{\*X,\*P}) \geq -\tfrac{1}{n}\max_{i=1,\dots, n}\|\*x_i\|_2^2
$$

We can then lower bound:
\begin{equation}
\begin{split}
    \lambda_{\min}(\*A_{\*X,\*P} \otimes I_n + \*B_{\*X,\*P}) 
    & \geq \tfrac{1}{n^2}(\lambda_{\min}(\*X\*X^\top) - \lambda_{\max}(\*X\*X^\top)) -\tfrac{1}{n}\max_{i=1,\dots, n}\|\*x_i\|_2^2 \\
    & = - \tfrac{1}{n}\gamma_{2, n}
\end{split}
\end{equation}

which yields the $\tfrac{1}{n}\gamma_{2, n}$-weak convexity of $g$, and finally the $\gamma_{2, n}$-weak convexity of $\mathcal{GM}^{2}_{\emp*r}$.

\textbf{Asymptotic}. For any $T$, we note that, almost surely, $\|T\|^2_{L_2(r_n)} \to \|T\|^2_{L_2(r)}$. As a result, since convexity is preserved under pointwise convergence and by virtue of Prop.~\eqref{prop:consistency}, we study the (almost sure) convergence of $\gamma_{\textrm{inner}, n}$ and $\gamma_{2, n}$. 

We start by $\gamma_{\textrm{inner}, n}$. We first remark that $\lambda_{\max}(\tfrac{1}{n}\*X\*X^\top) = \lambda_{\max}(\tfrac{1}{n}\*X^\top\*X)$. Moreover, as $\*A \in S_d^+(\mathbb{R}) \mapsto \lambda_{\max}(\*A)$ is continuous and $\tfrac{1}{n}\*X^\top\*X \to \mathbb{E}_{\*x\sim r}[\*x\*x^\top]$ almost surely, one has $\lambda_{\max}(\tfrac{1}{n}\*X\*X^\top) \to \lambda_{\max}(\mathbb{E}_{\*x\sim r}[\*x\*x^\top])$ almost surely. Moreover, for any $n > d$, $\lambda_{\min}(\tfrac{1}{n}\*X\*X^\top) = 0$. As a result, $\gamma_{\textrm{inner}, n} \to \lambda_{\max}(\mathbb{E}_{\*x\sim r}[\*x\*x^\top])$ almost surely, which provides the desired asymptotic result. 

We continue with $\gamma_{2, n}$. We first remark that $\max_{i=1,\dots,n}\|\*x_i\|_2^2 \leq \sup_{\*x\in\supp(r)}\|\*x\|_2^2$. As a result, by defining $\tilde{\gamma}_{2, n} = \gamma_{\textrm{inner}, n} + \max_{\*x\in\supp(r)}\|\*x\|_2^2$, $\mathrm{GMG}^{2}_{\emp*r}$ is also $\tilde{\gamma}_{2, n}$-weakly convex. Moreover, $\max_{\*x\in\supp(r)}\|\*x\|_2^2$ does not depends on $n$, $\tilde{\gamma}_{2, n} \to \lambda_{\max}(\mathbb{E}_{\*x\sim r}[\*x\*x^\top]) + \max_{\*x\in\supp(r)}\|\*x\|_2^2$ almost surely, 
which also provides the desired asymptotic result.

\end{proof}

\section{Experimental Details}
\label{app:experimental_details}
All our experiments build on \texttt{python} and the \texttt{jax}-framework~\citep{deepmind2020jax}, alongside \texttt{ott-jax} for optimal transport utilities.

\subsection{Details on Disentanglement Benchmark}
To effectively conduct comprehensive and representative research on disentangled representation learning, we convert the public PyTorch framework proposed in ~\cite{roth2023disentanglement} to an equivalent \texttt{jax} variant. We verify our implementation through replications of baseline and HFS results in \citet{roth2023disentanglement}, mainting relative performance orderings and close absolute disentanglement scores (as measured using DCI-D, whose implementation directly follows from \cite{pmlr-v97-locatello19a} and leverages gradient boosted tree implementations from \texttt{scikit-learn}).

For exact and fair comparison, we utilize standard hyperparamater choices from \citet{roth2023disentanglement} (which leverages hyerparameters directly from \cite{pmlr-v97-locatello19a}, \cite{pmlr-v119-locatello20a} and \url{https://github.com/google-research/disentanglement_lib}).
Consequently, the base VAE architecture utilized across all experiment is the same as the one utilized in \cite{roth2023disentanglement} and \cite{pmlr-v119-locatello20a}: With image input sizes of $64 \times 64 \times N_c$ (with $N_c$ the number of input image channels, usually 3). The latent dimensionality, if not otherwise specified, is set to 10. The exact VAE model architecture is as follows:
\begin{itemize}
\item \textbf{Encoder}: [conv($32, 4\times 4$, stride 2) + ReLU] $\times$ 2, [conv($64, 4\times 4$, stride 2) + ReLU] $\times$ 2, MLP(256), MLP(2 $\times$ 10)
\item \textbf{Decoder}: MLP(256), [upconv($64, 4\times 4$, stride 2) + ReLU] $\times$ 2, [upconv($32, 4\times 4$, stride 2) + ReLU], [upconv($n_c, 4\times 4$, stride 2) + ReLU] 
\end{itemize}

Similar, we retain all training hyperparameters from \citep{roth2023disentanglement} and \citep{pmlr-v119-locatello20a}: Using an Adam optimizer (\citep{kingma2014adam}, $\beta_1 = 0.9, \beta_2=0.999, \epsilon=10^{-8}$) and a learning rate of $10^{-4}$. Following \citet{pmlr-v119-locatello20a,roth2023disentanglement} we utilize a batch-size of $64$, for which we also ablate all baseline methods. The total number of training steps is set to $300000$.

\begin{table}[t]
\centering
\caption{Hyperparameter grid searches for different baseline and proposed methods.}
\label{tab:hyperparameters}
\begin{tabular}{c|c|c}
\toprule
\textbf{Method} & \textbf{Parameter} & \textbf{Values} \\
\midrule
$\beta$-VAE & $\beta$ & [2, 4, 6, 8, 10, 16] \\ 
$\beta$-TCVAE & $\beta$ & [2, 4, 6, 8, 10, 16]  \\
+ HFS & $\gamma$ & [1, 10] \\
+ \ref*{def:distortion} & $\lambda$ & [0.1, 1, 5, 10, 20] \\
+ \ref*{def:gromov-monge-gap} & $\lambda$ & [0.1, 1, 5, 10, 20] \\
+ \texttt{Jac} & $\lambda$ & [0.1, 1, 5, 10, 20] \\
\bottomrule
\end{tabular}
\end{table}

As commonly done for this setting \citep{pmlr-v97-locatello19a,pmlr-v119-locatello20a,roth2023disentanglement}, we also perform a small grid search over all the hyperparameters. We report the full details in Tab.~\ref{tab:hyperparameters}. 

For $\lambda_{e}$ and $\lambda_{d}$, we set $\lambda_{e}=0$ for the Decoder setting and $\lambda_{d}=0$ for the Encoder setting while altering the weighting for the other $\lambda$. All experiments run on a single RTX 2080TI GPU.

\subsection{Stability Analysis}
For the gradient stability analysis experiment, we repeat the following experiment for each of the four image datasets $\mathcal{D}$ that we consider. We first considered a fixed neural map $T_\theta$, which we choose to be a randomly initialized neural network, consisting of encoder and decoder, before any training. For $k=1,\dots,5$, we sample a batch $\*x_1^{k},\dots,\*x_n^{k} \sim \mathcal{D}$ and let $r_n^{k}=\tfrac{1}{n}\sum_{i=1}^n\delta_{\*x_i^{k}}$. We report the pariwise cosine similarity between the gradients of the~\ref*{def:distortion} and the~\ref*{def:gromov-monge-gap}. Formally, we compute $\textrm{cos-sim}(\nabla_\theta \mathrm{DST}_{r_{n}^{k}}(T_\theta), \nabla_\theta \mathrm{DST}_{r_{n}^{l}}(T_\theta))$, and $\textrm{cos-sim}(\nabla_\theta \mathrm{GMG}_{r_{n}^k}(T_\theta), \nabla_\theta \mathrm{GMG}_{r_{n}^l}(T_\theta))$, for $k,l=1,\dots,5$.

\section{Additional Empirical Results}
\label{app:additional-results}
In this section, we report additional empirical results revolving around the regularization of encoder. First in ~\ref{app:encoder-stability}, we conduct the stability analysis when regularizing the encoder. Then, we report further results of encoder and decoder regularization on DSprites. Lastly, we take a first exploratory step towards decoder-free disentangled representation learning in~\ref{app:decoder-free}.

\subsection{Encoder Stability Analysis}
\label{app:encoder-stability}
We repeat the following experiment for each of the four image datasets $\mathcal{D}$ that we consider. We first considered a fixed neural map $T_\theta$. For $i=1,\dots,5$, we sample a batch $\*x_1^{i},\dots,\*x_n^{i} \sim \mathcal{D}$ and let $r_n^{i}=\tfrac{1}{n}\sum_{i=1}^n\delta_{\*x_n^{i}}$. We report the pariwise cosine similarity between the gradients of the~\ref*{def:distortion}. Namely, for $i,j=1,\dots,5$, we compute $\textrm{cos-sim}(\nabla_\theta \mathrm{DST}_{r_{n}^{i}}(T_\theta), \nabla_\theta \mathrm{DST}_{r_{n}^{j}}(T_\theta))$, and the~\ref*{def:gromov-monge-gap}, $\textrm{cos-sim}(\nabla_\theta \mathrm{GMG}_{r_{n}^i}(T_\theta), \nabla_\theta \mathrm{GMG}_{r_{n}^j}(T_\theta))$.

\begin{figure}[H]
    \centering
    \includegraphics[width=0.6\linewidth]{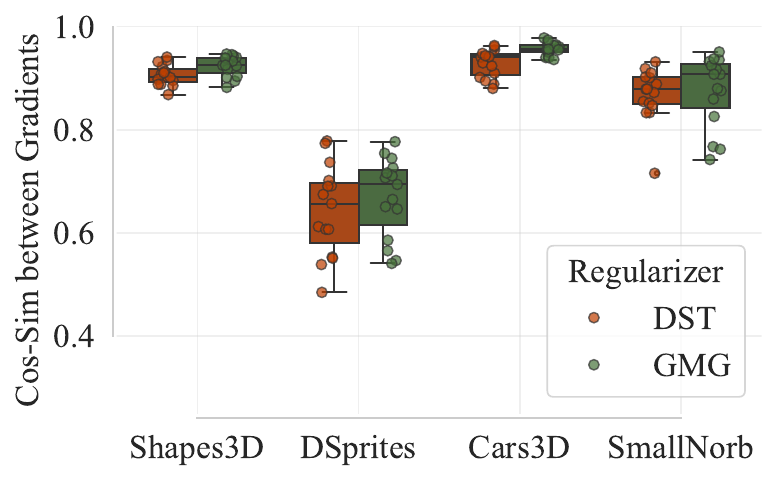}
    \caption{Gradient stability analysis on the \ref*{def:distortion} and \ref*{def:gromov-monge-gap} as \texttt{Encoder} regularizations. The cosine similarity is computed between all pairs of gradients $\nabla_\theta$ obtained through 5 randomly sampled batches and a fixed network $T_\theta$ for each dataset.}
    \label{fig:stability-exp}
\end{figure}

\subsection{Encoder Analysis on DSprites}
\begin{table}[H]
\caption{Disentanglement of regularizing the Encoder and the Encoder and Decoder as measured by \textbf{DCI-D} on DSprites. We highlight \textbf{best}, \underline{second best}, and \textit{third best} results for each method and dataset.}
\label{table:encoder-benchmark}
\begin{center}
\begin{tabular}{lcccc}
\toprule
\textbf{DCI-D} & $\beta$-VAE & $\beta$-TCVAE & $\beta$-VAE + HFS & $\beta$-TCVAE + HFS\\
\midrule
\multicolumn{5}{l}{\textbf{DSprites}~\citep{higgins2017betavae}} \\
\midrule
Base & $27.6$ \textcolor{gray}{$\pm 13.4$} & $36.0$ \textcolor{gray}{$\pm 5.3$} & \textit{38.7} \textcolor{gray}{$\pm 15.7$} & $48.1$ \textcolor{gray}{$\pm 10.8$}\\
+ Enc-\ref*{def:distortion} & \textit{32.8} \textcolor{gray}{$\pm 15.0$} & $36.5$ \textcolor{gray}{$\pm 5.9$} & $33.9$ \textcolor{gray}{$\pm 15.9$} & \textit{$48.9$} \textcolor{gray}{$\pm 11.1$}\\
+ Enc-\ref*{def:gromov-monge-gap} & $27.5$ \textcolor{gray}{$\pm 14.3$} & \textit{37.4} \textcolor{gray}{$\pm 5.8$} & $31.0$ \textcolor{gray}{$\pm 14.3$} & $45.9$ \textcolor{gray}{$\pm 10.9$}\\
+ Dec-\ref*{def:distortion} & $28.6$ \textcolor{gray}{$\pm 19.3$} & {$32.4$} \textcolor{gray}{$\pm 8.5$} & \underline{$39.3$} \textcolor{gray}{$\pm 18.1$} & \underline{$49.0$} \textcolor{gray}{$\pm 11.2$} \\
+ Dec-\ref*{def:gromov-monge-gap} & \textbf{39.5} \textcolor{gray}{$\pm 15.2$} & \textbf{42.2} \textcolor{gray}{$\pm 3.6$} & \textbf{46.7} \textcolor{gray}{$\pm 2.0$} & \textbf{50.1} \textcolor{gray}{$\pm 8.5$}  \\
\bottomrule
\end{tabular}
\end{center}
\end{table}

\newpage
\subsection{\rev{Ablation on the Entropic Regularization Strength}}
\label{app:ablation-epsilon}

\paragraph{\rev{Effect of $\varepsilon$ on the entropic GW solver.}} 
\rev{In this section, we remind some insights about the entropic GW solver~\citep{peyre2016gromov} introduced in \S~\ref{subsec:quadratic-ot} and used to compute the~\ref*{def:gromov-monge-gap} in Alg.~\ref{algo:entropic-gmg}. We provide its algorithmic details in Alg.~\ref{algo:entropic-gw}. This solver naturally provides a trade-off between (i) the approximation of the true GW distance $\mathrm{GW}(p, q)$ and (i) the optimization speed (i.e., convergence rate). This trade-off is controlled by the entropic regularization strength \(\varepsilon\).}

\begin{itemize}[leftmargin=.5cm,itemsep=.0cm,topsep=0cm]
\item[(i)] \rev{\textbf{Optimization.} The solver employs a mirror descent scheme that iteratively linearizes the entropic GW problem~\ref{eq:entropic-gromov-wasserstein} and applies the Sinkhorn algorithm. Each mirror descent step corresponds to a projection with respect to the KL divergence, which can be efficiently performed using the Sinkhorn algorithm. The solver is initialized with \(\mathbf{P}_0 = \tfrac{1}{n^2}\mathbf{1}\mathbf{1}^\top\) and iterates as \(\mathbf{P}_{t+1} \gets \textsc{Sinkhorn}\left(-\mathbf{C}_{\mathcal{X}} \mathbf{P}_t \mathbf{C}_{\mathcal{Y}}, \varepsilon\right)\). Since a larger \(\varepsilon\) in Sinkhorn leads to faster convergence~\citep{cuturi2013sinkhorn, altschuler2018nearlineartimeapproximationalgorithms}, this subsequently accelerates each mirror descent step. In other words, it reduces the number of iterations within each Sinkhorn call, i.e., the number of \textit{inner iterations} of the solver. Furthermore, \citet{rioux2023entropic} recently demonstrated that increasing \(\varepsilon\) enhances the convexity of the entropic GW problem, thereby improving the convergence rate of the mirror descent scheme. This provides theoretical justification for using a larger \(\varepsilon\) to reduce the number of mirror descent steps. In other words, it reduces the number of calls to Sinkhorn, i.e., the number of \textit{outer iterations} of the solver.

We have empirically validated this behavior on DSprites with BetaVAE + GMG. We plot the mean amount of inner and outer iterations for the first 3 epochs over 5 seeds for six different values of $\varepsilon_0$, which we provide in Figure~\ref{fig:eps-ablation}. We can observe the expected scaling of increased \textit{inner} and \textit{outer} iterations with decreased $\varepsilon_0$.}
\item[(ii)]\rev{\textbf{Approximation.} \citet{zhang2023gromovwasserstein} show that when using inner products and squared Euclidean distances as costs, and for \(\varepsilon \in (0, 1]\), the approximation error scales as:
\[
|\mathrm{GW}_\varepsilon(p, q) - \mathrm{GW}(p, q)| \lesssim \varepsilon \log(1/\varepsilon),
\]
where the constants in \(\lesssim\) depend on \(d_{\mathcal{X}}\) and \(d_{\mathcal{Y}}\), that is, the dimensions of the support of \(p\) and \(q\), as well as their fourth-order moments.}
\end{itemize}
\rev{Given this trade-off, the practical goal is to select an \(\varepsilon\) value that is sufficiently large to ensure fast convergence while avoiding any degradation in performance. Across all our experiments, we found that \(\varepsilon = 0.1\) struck the right balance. We validate this observation in the next paragraph.}

\begin{minipage}{1\textwidth}
\begin{algorithm}[H]
\caption{Entropic Gromov-Wasserstein solver~\citep{peyre2016gromov}, \citep[Alg.\ 2]{scetbon2022linear}.}
\label{algo:entropic-gw}
\begin{algorithmic}[1]
    \State{\textbf{Require:} samples $\*x_1,\dots,\*x_n \sim p$; $\*y_1,\dots,\*y_n \sim q$; cost functions $c_\+X, c_\+Y$; entropic regularization scale $\varepsilon_0$ (default $=0.1$), statistic operator on cost matrix \texttt{stat} (default $=\texttt{mean}$).}
    \State{$\*{C}_\+X\gets [c_\+X(\*x_i,\*x_{i'})]_{1\leq i,i'\leq n}$ \,\, \Comment{usually \teal{$\mathcal{O}(n^2d_\+X)$}}}
    \State{$\*{C}_\+Y\gets [c_\+Y(\*y_j,\*y_{j'})]_{1\leq j,j'\leq n}$ \,\, \Comment{usually \teal{$\mathcal{O}(n^2d_\+Y$)}}}
    \State{$\varepsilon \gets \varepsilon_0 \cdot \texttt{stat}(\*C_\+X) \cdot \texttt{stat}(\*C_\+Y)$ \Comment{usually \teal{$\mathcal{O}(n^2)$}}}
    \State{$\*P_t \gets \tfrac{1}{n^2}\*1_n\*1_n$ \Comment{\teal{$n^2$}}}
    \While{converged}
        \State $\*C_{t+1} \gets -\mathbf{C}_{\mathcal{X}} \mathbf{P}_t \mathbf{C}_{\mathcal{Y}}$ \Comment{\teal{$n^3$ or $n^2(d_\+X+d_\+Y)$}}
        \State \(\mathbf{P}_{t+1} \gets \textsc{Sinkhorn}(\*C_{t+1}, \varepsilon)\) \Comment{\teal{$\mathcal{O}(n^2)$}}
    \EndWhile
    \State Compute $\mathrm{GW}_\varepsilon(p_n, q_n)$ from $\*P_t$ using Eq.~\ref{eq:entropic-gromov-wasserstein} \Comment{\teal{$n^3$ or $n^2(d_\+X+d_\+Y)$}}
    \State{{\bfseries return}{ $\mathrm{GW}_\varepsilon(p_n, q_n)$}}
\end{algorithmic}
\end{algorithm}
\end{minipage}

\paragraph{\rev{Effect of $\varepsilon$ in disentanglement.}} \rev{We investigate the effect of the entropic regularization strength $\varepsilon$ used to compute \ref*{def:gromov-monge-gap} on the disentanglement performances. The results are presented in Fig.~\ref{fig:eps-ablation} and show that performance is robust to the choice of entropic regularization scale $\varepsilon_0$. We observed this both with respect to a setting, where we see major improvements from the GMG (DSprites with BetaVAE) as well as one where we only observe minor improvements (DSprites with BetaTCVAE). This validates our choice to use a single reasonable value for all our experiments, namely $\varepsilon_0 = 0.1$. This robustness with respect to the entropic strength was also observed in a recent work~\citep{piran2024contrastingmultiplerepresentationsmultimarginal} proposing a similar gap regularization, based on the entropic multi-marginal OT problem. See~\citep[Fig.\ 3]{piran2024contrastingmultiplerepresentationsmultimarginal} for experiments highlighting this robustness. Our intuition is as follows: when using Sinkhorn (in the case of M3G) or GW/quadratic OT (in our work) to compute a \emph{training loss}, $\varepsilon$ acts as a \emph{sharpness} parameter, emphasizing certain pairs of points more strongly. While the \emph{loss} value changes with $\varepsilon$, the optimization of network variables on top of this loss appears to be largely unaffected by the sharpness introduced by $\varepsilon$. However, when the output of Sinkhorn or GW is used directly for predictions or learning flows, for example, in Monge maps~\citep{pooladian2021entropic,kassraie2024progressiveentropicoptimaltransport} or Gromov-Monge maps~\citep{klein2024entropic}, the entropic regularization strength $\varepsilon$ has a much stronger influence and requires careful tuning.}

\begin{figure}[t]
    \centering
    \begin{subfigure}
        \centering
        \includegraphics[width=0.45\linewidth]{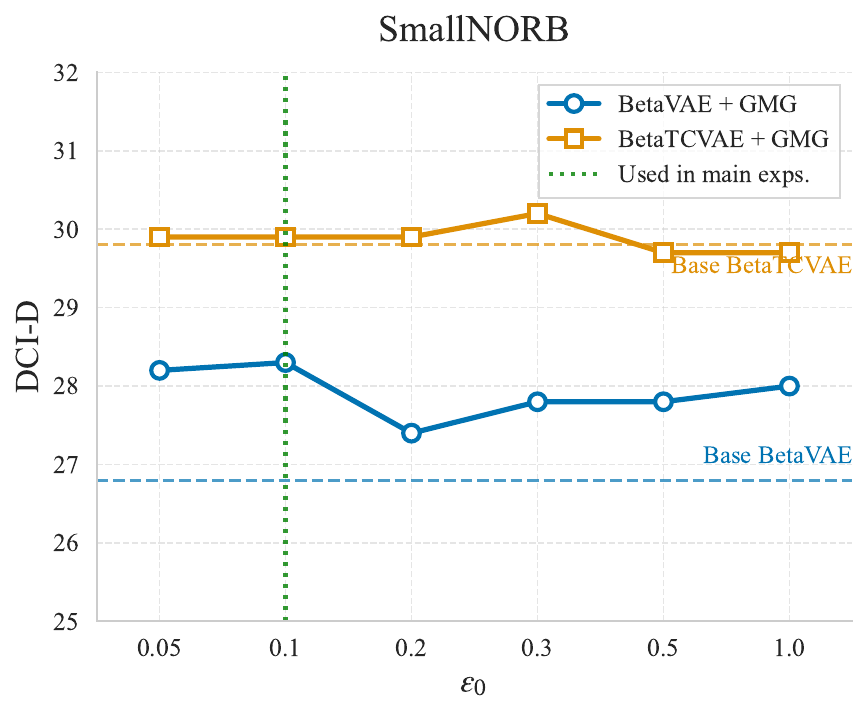}
    \end{subfigure}
    \hfill
    \begin{subfigure}
        \centering
        \includegraphics[width=0.52\linewidth]{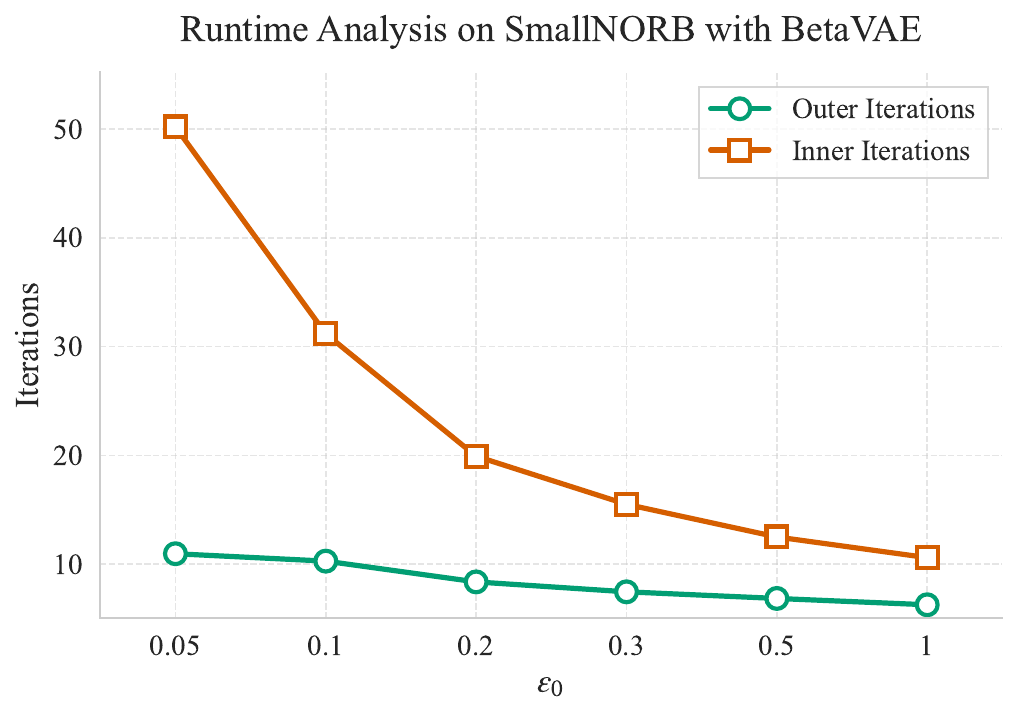}
    \end{subfigure}
    \caption{\rev{Analysis of the effect of the entropic regularization scale $\varepsilon_0$ on the disentanglement performances when learning with the GMG. We use the SmallNORB dataset and consider the $\beta$-VAE (in blue) and $\beta$-TCVAE (in orange) settings. The GMG is applied to the decoder $d_\theta$ with cost functions $c_{\mathcal{X}} = c_{\mathcal{Y}} = \mathrm{cos}\text{-}\mathrm{sim}(\cdot)$, corresponding to the setting that provides better disentanglement. We compute the GMG using Algorithm~\ref{algo:entropic-gmg} with \texttt{stat=mean}. We investigate the effect of the entropic regularization scale $\varepsilon_0$ by testing five other values of $\varepsilon_0$ besides the one used in all other experiments in the paper (namely, $\varepsilon_0 = 0.1$). The values tested are $\varepsilon_0 \in \{0.05, 0.2, 0.3, 0.5, 1\}$. We also include the baseline result without using the GMG (dashed line) as a comparison. Additionally, we provide a runtime analysis with respect to both the inner and outer iteration of the GW-solver.}}
    \label{fig:eps-ablation}
\end{figure}

\rev{
\subsection{Scaling to Higher Image Resolutions}
As detailed in \S~\ref{subsec:gmg-estimation}, the computation of the \ref*{def:gromov-monge-gap} scales linearly with the data dimension, enabling our method to handle high-dimensional settings effectively. To demonstrate this scalability, we benchmark our approach on the Shapes3D dataset upscaled to $128 \times 128$ resolution. To accommodate the increased image resolution, we extend the Decoder by adding one additional layer. Using our best-performing configuration from the $64 \times 64$ experiments—\textbf{Cos} costs and the \ref*{def:gromov-monge-gap} applied to the decoder $d_\theta$—we conduct a focused hyperparameter search as described in Table~\ref{tab:hyperparameters-128}. The results, summarized in Table~\ref{table:128-benchmark}, compare the four base models Beta(TC)VAE (+ HFS), both with and without the \ref*{def:gromov-monge-gap}. Our findings indicate that the proposed setup scales to higher resolutions while preserving the performance improvements achieved by the \ref*{def:gromov-monge-gap}. Finally, we visually validate these results by plotting the latent traversals for the best-performing configuration, as shown in Figure~\ref{fig:latent-traversal-128}.

\begin{table}[t]
\centering
\caption{Hyperparameter grid search for Shapes3D $128 \times 128$.}
\label{tab:hyperparameters-128}
\begin{tabular}{c|c|c}
\toprule
\textbf{Method} & \textbf{Parameter} & \textbf{Values} \\
\midrule
$\beta$-(TC)VAE & $\beta$ & [10, 16] \\ 
+ \ref*{def:gromov-monge-gap} & $\lambda$ & [0.1, 1, 10] \\
\bottomrule
\toprule
$\beta$-(TC)VAE & $\beta$ & [2, 4] \\ 
+ \ref*{def:gromov-monge-gap} & $\lambda$ & [0.1, 1, 10] \\
+ HFS & $\gamma$ & [1, 10] \\
\bottomrule
\end{tabular}
\end{table}

\begin{table}[t]
\caption{\rev{Impact of the \ref*{def:gromov-monge-gap}, applied with $\mathbf{Cos}$ as the cost function on the decoder $d_\theta$, on disentanglement performance using upscaled $128 \times 128$ Shapes3D images. Performance is evaluated using \textbf{DCI-D}, with the \textbf{best} result highlighted for each method.}}
\label{table:128-benchmark}
\centering
\begin{tabular}{lcccc}
\toprule
With $\mathbf{Cos}$ costs & {$\beta$-VAE} & {$\beta$-TCVAE} & $\beta$-VAE + HFS& $\beta$-TCVAE + HFS\\
\midrule
\multicolumn{5}{c}{\textbf{Shapes3D ($128 \times 128$)}~\citep{3dshapes18}} \\
\midrule
Base & $54.3$ \textcolor{gray}{$\pm 19.3$} & ${74.6}$ \textcolor{gray}{$\pm 17.3$} & $87.2$ \textcolor{gray}{$\pm 2.7$} & $88.6$ \textcolor{gray}{$\pm 11.4$} \\
+ \texttt{Dec}-\ref*{def:gromov-monge-gap} & \textbf{63.9} \textcolor{gray}{$\pm 8.2$} & \textbf{82.0} \textcolor{gray}{$\pm 12.8$} & \textbf{90.4} \textcolor{gray}{$\pm 3.7$} & \textbf{92.2} \textcolor{gray}{$\pm 8.2$}  \\
\bottomrule
\end{tabular}%
\vspace{-4mm}

\end{table}

\subsection{Latent Traversal Visualization}
\label{app:traversal}
\begin{figure}[H]
    \centering
    \includegraphics[width=0.8\linewidth]{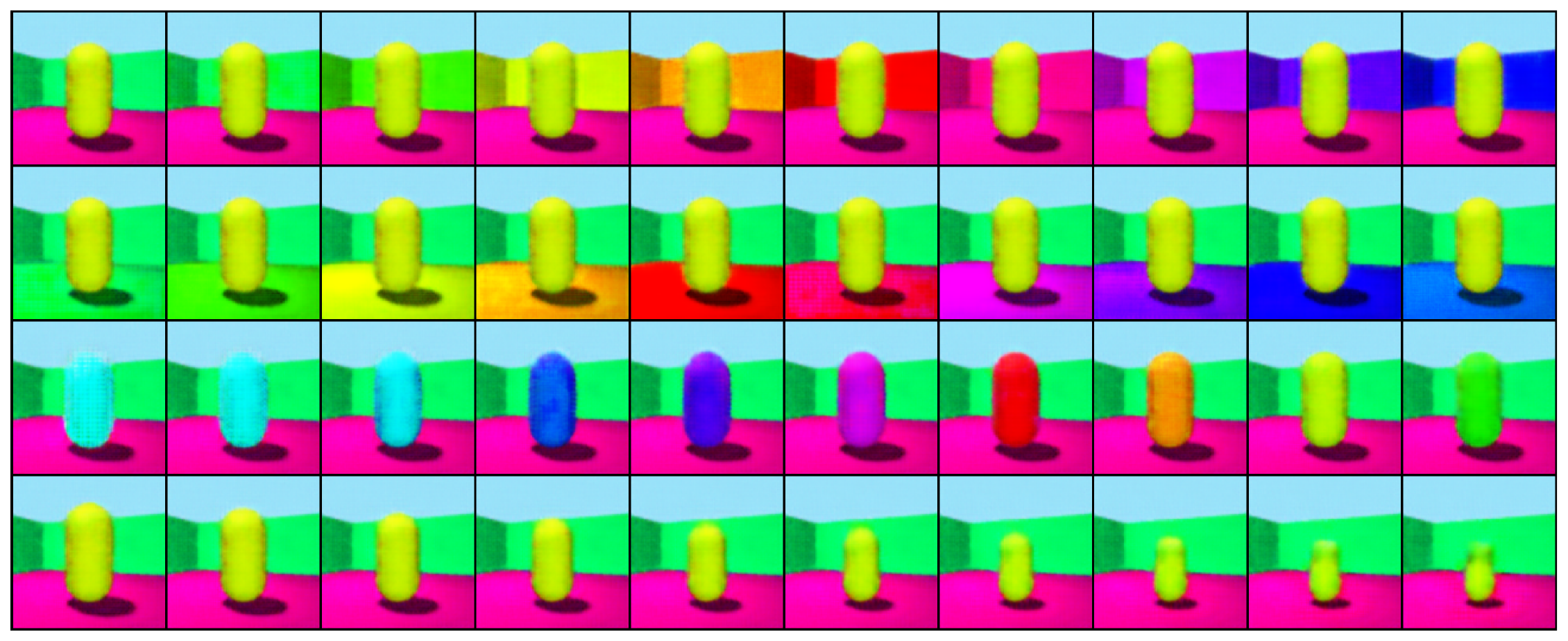}
    \caption{\rev{Latent traversal visualization for Shapes3D $128 \times 128$ with our best performing setup, BetaTCVAE + HFS + \ref*{def:gromov-monge-gap}. We select the best performing result out of 5 seeds achieving a DCI-D of $99.4$. We plot four different latent dimensions while traversing them from $-1.0$ to $1.0$. As visualized the model has clearly learned to separate wall hue, object hue, scale, and floor hue into different latent dimensions.}
    \label{fig:latent-traversal-128}}
\end{figure}}

\begin{figure}[H]
    \centering
    \includegraphics[width=0.4\linewidth]{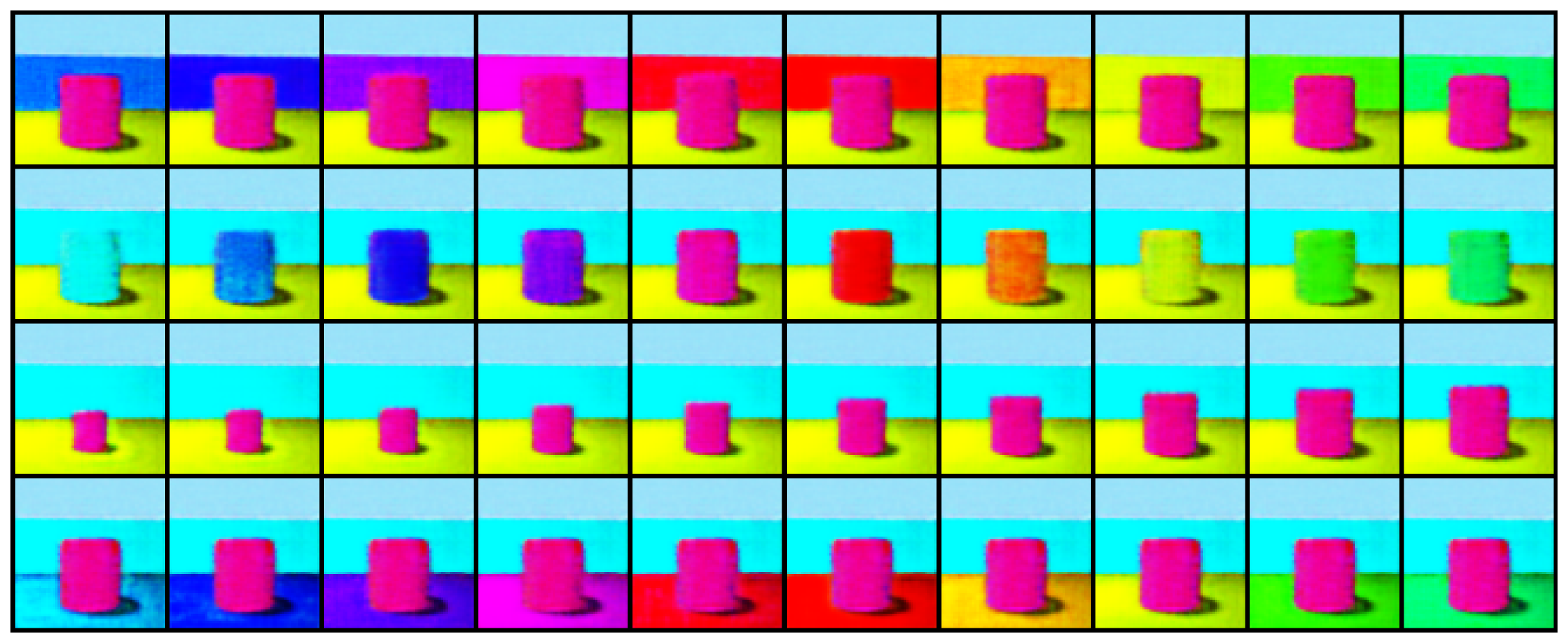}
    \caption{\rev{Latent traversal visualization for Shapes3D $64 \times 64$ with our best performing setup, BetaTCVAE + HFS + \ref*{def:gromov-monge-gap}. We select the best performing result out of 5 seeds achieving a DCI-D of $100.0$. We plot four different latent dimensions while traversing them from $-1.0$ to $1.0$. As visualized the model has clearly learned to separate wall hue, object hue, scale, and floor hue into different latent dimensions.}
    \label{fig:latent-traversal}}
\end{figure}

\newpage
\subsection{Towards Decoder-free Disentanglement}
\label{app:decoder-free}

\begin{wraptable}{o}{0.42\textwidth}
\vspace{-5mm}
\caption{Disentanglement (DCI-D) without a decoder trained with various regularizations on Shapes3D~\citep{3dshapes18}.}
\label{table:decoder-free-benchmark}
\centering
\begin{adjustbox}{width=0.42\textwidth}  
\begin{tabular}{cccc}
    \toprule
    \textbf{Decoder-free}  & $\beta$-VAE & $\beta$-TCVAE\\
    \midrule
    Base &  $0.0$ \textcolor{gray}{$\pm 0.0$}  & $0.0$ \textcolor{gray}{$\pm 0.0$}\\
    \midrule
    \multicolumn{3}{l}{$\mathbf{L2^2}: \|\cdot-\cdot\|_2^2$} \\
    \midrule
    + \ref*{def:distortion} &  \underline{$38.2$} \textcolor{gray}{$\pm 0.8$}  & $42.7$ \textcolor{gray}{$\pm 1.6$}\\
    + \ref*{def:gromov-monge-gap} &  $13.9$ \textcolor{gray}{$\pm 0.4$}  & $20.5$ \textcolor{gray}{$\pm 0.5$}\\
    \midrule
    \multicolumn{3}{l}{$\mathbf{ScL2^2}: \alpha\|\cdot-\cdot\|_2^2$, $\alpha>0$ learnable} \\
    \midrule
    + \ref*{def:distortion} &  \textbf{45.6} \textcolor{gray}{$\pm 1.2$}  & \textbf{53.5} \textcolor{gray}{$\pm 1.0$}\\
    + \ref*{def:gromov-monge-gap} &  $15.2$ \textcolor{gray}{$\pm 0.3$}  & $25.2$ \textcolor{gray}{$\pm 0.6$}\\
    \midrule
    \multicolumn{3}{l}{$\mathbf{Cos}: \textrm{cos-sim}(\cdot, \cdot)$} \\
    \midrule
    + \ref*{def:distortion} &  $37.0$ \textcolor{gray}{$\pm 0.4$}  & \underline{$46.1$} \textcolor{gray}{$\pm 1.5$}\\
    + \ref*{def:gromov-monge-gap} &  $37.0$ \textcolor{gray}{$\pm 0.9$}  & $38.8$ \textcolor{gray}{$\pm 1.1$}\\
    \bottomrule
\end{tabular}
\end{adjustbox}
\vspace{-5mm}
\end{wraptable}
 
Recently, works such as \citep{burns2021unsupervised,vonkugelgen2021self,eastwood2023selfsupervised,matthes2023towards,aitchison2024infonce} have shown the possibility of disentanglement through self-supervised, contrastive learning objectives in an effort to align with the scalability of encoder-only representation learning~\citep{chen2020simclr,zbontar2021barlow,bardes2022vicreg,garrido2023on}. However, these encoder-only approaches still require weak supervision or access to multiple views of an image to learn meaningful representations of the data samples.

As the goal of geometry preservation connects the data manifold and the latent domain through a minimal distortion objective and is applicable to both the encoder and decoder of a VAE (\S\ref{sec:method}, Table~\ref{table:encoder-benchmark}), we posit that its application may provide sufficient training signal to learn meaningful representations and encourage disentanglement, eliminating the need for a reconstruction loss and decoder. Table~\ref{table:decoder-free-benchmark} shows preliminary results on unsupervised decoder-free disentangled representation learning on the Shapes3D benchmark, where the decoder and associated reconstruction objective have been removed.

Standard approaches such as $\beta$-VAE or $\beta$-TCVAE collapse and do not achieve measurable disentanglement (DCI-D of $0.0$). However, the inclusion of either DST or GMG significantly raises achievable disentanglement and, combined with the $\beta$-TCVAE matching objective, can achieve DCI-D scores of up to $53.5$ without needing any decoder or reconstruction loss. While these are preliminary insights, we believe they offer promise for more scalable approaches to unsupervised disentangled representation learning and potential bridges to popular and scalable self-supervised representation learning approaches. Note, that here the distortion loss significantly outperforms the GMG. This is expected due to the nature of the GMG, as the distortion loss offers a more restrictive and, thus, stronger signal for learning representations, which is necessary in the absence of a reconstruction objective. This highlights that while in most scenarios (\S~\ref{sec:disentangled-representational-learning}, Figure~\ref{fig:gmg}), the GMG is preferable over the distortion loss, there also exist settings where a more restrictive optimization signal is desirable.

\newpage
\section{Python Code for the Computation of the Gromov-Monge Gap}
\label{app:code}
\begin{python}
import jax
import jax.numpy as jnp

from ott.geometry import costs, geometry
from ott.solvers.quadratic import gromov_wasserstein
from ott.problems.quadratic import quadratic_problem

def gromov_monge_gap_from_samples(
    source: jax.Array,
    target: jax.Array,
    cost_fn: costs.CostFn = costs.Cosine(),
    epsilon: float = 0.1,
    stat_fn: Callable, # usually computes the mean of the cost matrix
    **kwargs,
) -> float:
    """Gromov Monge gap regularizer on samples."""

    # define source and target geometries
    cost_matrix_x = cost_fn.all_pairs(x=source, y=source)
    scale_cost_x = stat_fn(scale_cost_x)
    cost_matrix_x = cost_matrix_x / jax.lax.stop_gradient(scale_cost_x)
    geom_xx = geometry.Geometry(cost_matrix=cost_matrix_x)

    cost_matrix_y = cost_fn.all_pairs(x=target, y=target)
    scale_cost_y = stat_fn(cost_matrix_y)
    cost_matrix_y = cost_matrix_y / jax.lax.stop_gradient(scale_cost_y)
    geom_yy = geometry.Geometry(cost_matrix=cost_matrix_y)

    # define and solve entropic GW problem
    prob = quadratic_problem.QuadraticProblem(geom_xx, geom_yy)

    solver = gromov_wasserstein.GromovWasserstein(
        epsilon=epsilon, **kwargs
    )
    out = solver(prob)

    # compute the distortion induced by the map
    distortion_cost = jnp.nanmean(
        (geom_xx.cost_matrix - geom_yy.cost_matrix)**2
    )

    # compute optimal (entropic) gromov-monge displacement
    reg_gw_cost = out.reg_gw_cost
    ent_reg_gw_cost = reg_gw_cost - 2 * epsilon * jnp.log(len(source))

    # compute gromov-monge gap
    loss = distortion_cost - ent_reg_gw_cost

    return loss * jax.lax.stop_gradient(scale_cost_x * scale_cost_y)
\end{python}

\end{document}